\documentclass[twoside,11pt]{article}
\usepackage{jair, theapa, rawfonts}

\jairheading{72}{2021}{285-328}{02/2021}{10/2021}

\ShortHeadings{On Quantifying Literals in Boolean Logic}
{Darwiche \& Marquis}

\firstpageno{285}

\usepackage{url}
\usepackage{tikz}
\usepackage{graphicx}
\usepackage{xcolor}
\usepackage{eurosym}
\usepackage{amsthm}
\usepackage{multirow}
\usepackage{amssymb}
\usepackage{amsfonts}

\usepackage{subcaption}

\usetikzlibrary{decorations.markings}
\usetikzlibrary{arrows}
\usetikzlibrary{shapes}
\usetikzlibrary{positioning}
\usetikzlibrary{topaths,calc}
\tikzstyle{vertex} = [fill,shape=circle,node distance=80pt]
\tikzstyle{edge} = [fill,opacity=.5,fill opacity=.5,line cap=round, line join=round, line width=35pt]
\tikzstyle{elabel} =  [fill,shape=circle,node distance=30pt]

\def\RED{black}
\def\BLUE{black}

\newcommand\shrink[1]{}

\def\defas{{\quad\stackrel{\it def}{=}\quad}}
\newcommand{\hl}[1]{\emph{#1}}
\def\then{{\Rightarrow}}
\def\bthen{{\Leftarrow}}
\def\Rule(#1,#2){{{#1} \rightarrow {#2}}}
\def\force(#1,#2){{#1[#2]}}
\def\w{{\omega}}
\def\n(#1){{\bar{#1}}}
\def\l{{\ell}}
\def\nl{\bar{\ell}}
\def\cd{{|}}
\def\nvarphi{{\overline{\varphi}}}
\def\X{{\mathbb{X}}}

\def\RS(#1){{R(#1)}}
\def\MS(#1){{M(#1)}}
\def\BMS(#1){{BM(#1)}}
\def\ps{{\Sigma}}

\newcommand\erase[2]{{#1_{\uparrow#2}}}

\newtheorem{theorem}{Theorem}
\newtheorem{corollary}{Corollary}
\newtheorem{lemma}{Lemma}
\newtheorem{definition}{Definition}
\newtheorem{proposition}{Proposition}

\title{On Quantifying Literals in Boolean Logic and Its \\ Applications to Explainable AI}

\author{Adnan Darwiche \and Pierre Marquis}

\begin{document} 

\title{On Quantifying Literals in Boolean Logic and Its \\ Applications to Explainable AI}

\author{\name Adnan Darwiche \email darwiche@cs.ucla.edu \\
       \addr Computer Science Department, UCLA,\\
       Los Angeles, CA  90095 USA
       \AND
       \name Pierre Marquis \email marquis@cril.fr \\
       \addr CRIL, Universit\'e d’Artois \& CNRS,\\
        Institut Universitaire de France\\
	F-62307, Lens Cedex, France
}

\maketitle
 
\begin{abstract}
Quantified Boolean logic results from adding operators to Boolean logic for 
existentially and universally quantifying variables. This extends the reach of Boolean
logic by enabling a variety of applications that have been explored over the decades.
The existential quantification of literals (variable states) and its applications have also been 
studied in the literature. In this paper, we complement 
this by studying universal literal quantification and its applications, particularly to explainable AI. 
We also provide a novel semantics for quantification,
discuss the interplay between variable/literal and existential/universal quantification. We further
identify some classes of Boolean formulas and circuits on which quantification can be done efficiently.
Literal quantification is more fine-grained than variable quantification as the latter can be 
defined in terms of the former. This leads to a refinement of quantified Boolean logic with
literal quantification as its primitive.
\end{abstract}

\section{Introduction}

To quantify a variable from a Boolean formula is to eliminate that variable from the formula. 
This quantification process can be performed either existentially or universally, each leading to a different semantics
and a different set of applications. The use of quantification in Boolean logic dates back at least to
George Boole's work~\cite{Boole1854}, where he used universal quantification to perform some
forms of logical deduction. Existential quantification has a particularly intuitive interpretation as it
can be viewed as a process of removing from a formula all and only the information which pertains to the quantified
variable. Due to this semantics, which is referred to as forgetting~\cite{LinReiter94}, existential
quantification has received much attention in AI and database theory, particularly in the management 
of inconsistent information. This use dates back at least to~\cite{DBLP:conf/eds/Weber86} who employed 
existential quantification to combine pieces of information that may contradict each other. 
Variable quantification plays a prominent role in complexity theory too since deciding the validity of quantified 
Boolean formulas (QBFs) is the canonical {\sf PSPACE}-complete problem; see, for example,~\cite{Papadimitriou94}. 
The validity of QBFs has also been used to characterize the polynomial hierarchy~\cite{Stockmeyer77}.
In contrast to AI where existential quantification has had a more dominant role than universal quantification, 
the role of these two forms of quantification has been more symmetrical in other areas of computer science
particularly in complexity theory.

The interpretation of existential variable quantification as a process of forgetting information
prompted a refinement in which literals (states of variables) are also existentially quantified~\shortcite{LangLM03}.
With this advance, one can now remove information that pertains to a literal from a formula. 
Existential literal quantification is more fine-grained than existential variable quantification as
the latter can be formulated in terms of the former. This has led to further applications which we review later. 

The dominance of existential quantification within AI is worth noting. We attribute this to the lack of
an intuitive enough interpretation of universal quantification, which is a gap that we aim to address in this work.
We started our study by an observation that some recent work on explainable AI~\cite{DarwicheH20} 
can be formulated in terms of universal quantification, particularly universal literal quantification which has 
not been discussed in the literature before this work.
This prompted us to formalize this notion and to elaborately investigate its connections to explainable AI.
Our investigation led us to interpret universal quantification as a {\em selection} process (in contrast to a forgetting process),
which gave rise to many implications. It also led us to some results on the efficient computation
of universal literal quantification (e.g., being tractable on CNFs) which has further implications on
the efficient computation of explainable AI queries.

In explainable AI, one is typically interested in reasoning about the behavior of classifiers which make 
decisions on instances. For example, one may wish to understand why a classifier made a particular 
decision. One may also wish to determine whether a decision is biased (i.e., would be different if we were to only change 
some protected features of the instance). Such classifiers can be represented using Boolean 
formulas, even in some cases where they have a numeric form that is learned from data.\footnote{One can capture the 
input-output behavior of machine learning classifiers with discrete features using Boolean formulas, either through an 
encoding or a compilation process. This has been shown
for Bayesian network classifiers, random forests and some types of neural networks. For some example
works along these lines, see~\shortcite{ChanD03,NarodytskaKRSW18,ShihCD19,IgnatievNM19a,IgnatievNM19b,shi2020tractable,kr/AudemardKM20,ChoiShihGoyankaDarwiche20}.}
%Pierre: I removed the next two references
%KatzBDJK17,Leofante18,
In particular, one can use a Boolean formula to represent positive instances and use its negation to represent negative instances.
A major insight underlying our results is that many explainable AI queries correspond
to a process of selecting instances with particular properties, where the Boolean formula characterizing
such instances constitutes the answer to the explainable AI query. Moreover, the universal quantification of
literals (and variables) can be used to select such instances. We use this to formulate and generalize some of the
recently introduced notions in explainable AI. We also give examples of additional queries that have not been
treated earlier to emphasize the open ended applications that could be enabled by the new results.

Our formulation of universal literal quantification implies an alternate treatment of existential literal quantification 
compared to the existing literature (due to the duality between them). It also implies a new treatment of
variable quantification as this is subsumed by literal quantification. Our treatment is based on the
novel notion of boundary models. These are models of a Boolean formula that can become models of its negation
if we flip a single variable. In the context of classifiers, these correspond to instances
whose label may change if we flip a single feature. We also provide a wealth of results on boundary models,
which furthers our understanding of Boolean logic; for example, that the boundary models of a Boolean
formula characterize all its models and can be exponentially fewer in count. 

We start in Section~\ref{sec:prelim} with some preliminaries on Boolean logic, where we also introduce
the notion of boundary models. We then review variable quantification and some of its applications in Section~\ref{sec:vq}.
We study boundary models further in Section~\ref{sec:brules} where we develop a number of additional results.
Some of these results are used in our upcoming treatment while others are of a more general interest to the study of Boolean logic. 
We then formalize and study universal literal quantification in 
Section~\ref{sec:lq}, where we also review existential literal quantification and some of its applications. 
We then turn our attention in Section~\ref{sec:tractable} to the computation of literal quantification 
on various types of formulas and circuits, including CNFs, DNFs, Decision-DNNFs~\cite{jair/HuangD07} 
and SDDs~\cite{ijcai/Darwiche11}, which include OBDDs~\cite{tc/Bryant86} as a special case. 
Section~\ref{sec:XAI} constitutes a significant portion of the paper and is dedicated to the interpretation 
of universal quantification as a selection process and its applications to explainable AI. We finally close with some concluding
remarks in Section~\ref{sec:conclusion}.
Proofs of all results can be found in Appendixes~\ref{app:lemmas} and~\ref{app:proofs}.

\section{Boolean Logic Preliminaries}
\label{sec:prelim}

\begin{table}[t]
\begin{center}
\small
\color{\RED}
\begin{tabular}{l|l}
{\bf Notation} & {\bf Description}\\
\hline
$\ps$ & A finite set of Boolean variables \\
$\w$ & A truth assignment (world) which maps \(\ps\) to $\{0,1\}$ \\
$x, X$ & Positive literals of Boolean variable \(X\) \\
$\n(x), \neg X$ & Negative literals of Boolean variable \(X\) \\
$\l, \nl$ & A literal (positive or negative) and its negation \\
$\w[\l]$ & If \(\nl \in \w\), \(\w[\l]\) is the result of replacing \(\nl\) with \(\l\) in \(\w\); otherwise, \(\w[\l]=\w\) \\
$\w \models \varphi$ & World \(\w\) satisfies formula \(\varphi\) \\
$\varphi \models \psi$ & Formula $\varphi$ implies formula $\psi$ \\
$\MS(\varphi)$ & The set of models of formula $\varphi$\\
$\BMS(\varphi)$ & The set of boundary models of formula $\varphi$\\
$\RS(\varphi)$ & The set of b-rules for formula $\varphi$
\end{tabular}
\color{black}
\end{center}
\caption{Some of the key notation adopted in this paper. \label{tab:notation}}
\end{table}

We start with some notational conventions; see also Table~\ref{tab:notation}.
We assume a finite set of Boolean variables \(\ps\) where \(x\) and \(\n(x)\) denote the positive and negative literals of Boolean variable \(X\).
We may also use \(X\) and \(\neg X\) to denote these literals. 
We use \(\top\) and \(\bot\) to denote the Boolean constants true and false. 
A {\em world,} typically denoted by \(\w\), is a truth assignment (a mapping from \(\ps\) to $\{0, 1\}$), 
often represented as a set of literals that contains exactly one literal for each Boolean variable in \(\ps\). Alternatively, a world 
can be represented by a sequence of literals, sometimes using commas as separators.
When a world \(\w\) satisfies a Boolean formula \(\varphi\) we say it is a {\em model} of the formula and write \(\w \models \varphi\).
Otherwise, $\w$ is a model of $\neg \varphi$, aka a {\em counter model} of $\varphi$.
We use \(\MS(\varphi)\) to denote the models of formula \(\varphi\). 
Whenever \(\MS(\varphi) \subseteq \MS(\psi)\) holds, formula $\psi$ is said to be a {\em logical consequence} 
of formula $\varphi$ (or equivalently, $\varphi$ implies $\psi$).
Furthermore, when \(\MS(\varphi) = \MS(\psi)\) holds, $\varphi$ and $\psi$ are said to be {\em logically equivalent}.
For a literal \(\l\) of variable \(X\), we 
use \(\w[\l]\) to denote the world that results from replacing the literal of variable \(X\) in \(\w\) with literal \(\l\).
For example, if \(\w = x y \n(z)\), then \(\w[\n(x)] = \n(x) y \n(z)\) and \(\w[x] = x y \n(z)\).

The following novel definition is fundamental as it will play a critical role when defining the semantics of literal quantification.

\begin{definition}\label{def:bmodel}
A world \(\w\) is said to be an \hl{\(\l\)-boundary model} of Boolean formula \(\varphi\) iff \(\w\) contains literal \(\l\),
\(\w\) is a model of \(\varphi\) and \(\w[\nl]\) is a model of \(\neg \varphi\).
\end{definition}
\noindent That is, model \(\w\) of formula \(\varphi\) becomes a model of \(\neg \varphi\) once we flip literal \(\l\) in \(\w\) to \(\nl\).
Consider the formula \(\varphi = (x \then y)\wedge(y \then x)\).  World \(\w = \{x,y\}\) is an \(x\)-boundary model
for formula \(\varphi\) as it contains literal \(x\), is a model of \(\varphi\) but becomes a model of its negation \(\neg \varphi\) 
once we flip literal \(x\) to \(\n(x)\). World \(\w\) is also a \(y\)-boundary model for formula \(\varphi\). 
We will use \(\BMS(\varphi)\) to denote the set of all boundary models of formula \(\varphi\). 

Boundary models will be used to define the semantics of single-literal quantification. 
The next, novel notion will be used to define the semantics of multiple-literal quantification.
\begin{definition}\label{def:imodel}
Let \(\alpha\) be a set of literals. 
A world \(\w\) is said to be an \hl{\(\alpha\)-independent model} of formula \(\varphi\) iff 
\(\alpha \subseteq \w\) and \(\w\setminus\alpha \models \varphi\).
\end{definition}
\indent Consider a world \(\w\) and literals \(\alpha\) such that \(\alpha \subseteq \w\).
If world \(\w\) is an \(\alpha\)-independent model of formula \(\varphi\), 
we can flip any literals of \(\alpha\) in world \(\w\) while maintaining the world as a model of \(\varphi\). 
%Every model of \(\varphi\) is \(\{\}\)-independent.
In this case, the model \(\w\) cannot be \(\l\)-boundary for any literal \(\l \in \alpha\). 
The converse is not true though. Flipping a single literal \(\l \in \alpha\) may maintain \(\w\) as a model of \(\varphi\) for all \(\l \in \alpha\),
yet flipping multiple literals in \(\alpha\) may not. Consider formula \(\varphi = (x \vee y)\wedge z\) and its model \(\w = \{x,y,z\}\).
This model is not \(x\)-boundary since \(\w[\n(x)]=\{\n(x),y,z\}\) is also a model of \(\varphi\). It is also not
\(y\)-boundary since \(\w[\n(y)]=\{x,\n(y),z\}\) is also a model of \(\varphi\). However,  $\w$
is not an \(\{x,y\}\)-independent model of $\varphi$ since \(\w[\n(x),\n(y)]=\{\n(x),\n(y),z\}\) is not a model of \(\varphi\). On the other
hand, world \(\{x,y,\n(z)\}\) is an \(\{x,y\}\)-independent model of \(\neg \varphi\) as we can flip literals
\(x\) and \(y\) in any manner while maintaining the world as a model for \(\neg \varphi\).

We next review a number of classical notions from Boolean logic.
A {\em term} is a set of literals over distinct variables and represents the conjunction of these literals.
A term is therefore consistent and the empty term represents \(\top\). We will sometimes denote a term or a world
such as $\{x, \n(y)\}$ using the tuple $x\n(y)$ or $(x, \n(y))$.
A {\em clause} is a set of literals over distinct variables and represents the disjunction of these literals.
A clause is therefore non-valid and the empty clause represents \(\bot\).
A {\em Disjunctive Normal Form (DNF)} is a set of terms representing the disjunction of these terms. 
A {\em Conjunctive Normal Form (CNF)} is a set of clauses representing the conjunction of the clauses.
A {\em Negation Normal Form (NNF)} is a formula which contains only the constants \(\top\), \(\bot\),
the connectives \(\wedge\), \(\vee\), \(\neg\) and where negations appear only next to constants or variables.
DNF and CNF are subsets of NNF.
An {\em NNF circuit} is a Boolean circuit satisfying the conditions of an NNF formula.
A formula/circuit is {\em monotone} iff it is equivalent to a formula in which literals $x$ and $\n(x)$ cannot both 
occur for any variable~\(X\). 

An {\em implicate} of a formula \(\varphi\) is a clause implied by \(\varphi\). A {\em prime implicate} of \(\varphi\)
is an implicate of \(\varphi\) that is not a superset of another implicate of \(\varphi\). A {\em implicant}
of \(\varphi\) is a term that implies \(\varphi\). A {\em prime implicant} of \(\varphi\) is an implicant of \(\varphi\)
that is not a superset of another implicant of \(\varphi\). The {\em resolution} of clauses \(x \vee \alpha\)
and \(\n(x) \vee \beta\) on variable $X$ leads to clause \(\alpha \vee \beta\) when \(\alpha\) and \(\beta\) share no complementary
literals. Adding clauses to a CNF using resolution does not change the models of the CNF. 
A CNF is {\em closed} under resolution on variable \(X\) iff the result of each resolution on \(X\) is in the CNF.
The {\em consensus} of terms \(x \wedge \alpha\) and \(\n(x) \wedge \beta\) on variable $X$ leads to term \(\alpha \wedge \beta\)
when \(\alpha\) and \(\beta\) share no complementary literals. 
Adding terms to a DNF using consensus does not change the models of the DNF.  
A DNF is {\em closed} under consensus on variable \(X\) iff the result of each consensus on \(X\) is in the DNF.

The occurrence of a variable in a Boolean formula is well defined, but the occurrence of a literal in a formula
can be ambiguous (e.g., whether the positive literal \(x\) occurs in \(\neg(\n(x) \vee y)\)). We will therefore refer
to literal occurrence only with respect to NNF since this is well defined. We conclude this 
section with three additional notions.
\begin{itemize}
\item[--] The {\em conditioning} of formula \(\varphi\) on positive literal \(x\), denoted \(\varphi \cd x\),
is obtained by replacing every occurrence of variable \(X\) in \(\varphi\) with \(\top\).
The conditioning of formula \(\varphi\) on negative literal \(\n(x)\), denoted \(\varphi \cd \n(x)\),
is obtained by replacing every occurrence of variable \(X\) in \(\varphi\) with \(\bot\).
Hence, variable \(X\) does not occur in \(\varphi \cd x\) or in \(\varphi \cd \n(x)\).

\item[--] A formula \(\varphi\) is {\em independent of variable} \(X\) iff \(\varphi\) is equivalent to some formula 
in which variable \(X\) does not occur. \textcolor{\BLUE}{For instance, $\varphi = (x \wedge y) \vee (\n(x) \wedge y)$
is independent of variable $X$ since $\varphi$ is equivalent to $y$ and the latter does not mention variable \(X\).}

\item[--] A formula \(\varphi\) is {\em independent of literal} \(\l\) iff \(\varphi\) is equivalent to some NNF in 
which literal \(\l\) does not occur~\cite{LangLM03}. \textcolor{\BLUE}{Thus, $\varphi = (x \wedge y) \vee (\n(x) \wedge y)$
is independent of literal $\n(x)$ since $\varphi$ is equivalent to NNF $y$ which does not mention literal \(\n(x)\).}
\end{itemize}
It follows that \(\varphi\) is independent of variable \(X\) iff \(\varphi\) is independent of literals \(x\) and \(\n(x)\).

\section{Variable Quantification}
\label{sec:vq}

A variable $X$ can be quantified either existentially or universally from a Boolean formula, leading to another formula
that does not depend on variable \(X\). The result of existentially quantifying variable \(X\) from formula \(\varphi\)
is denoted by \(\exists X \cdot \varphi\) and the result of universally quantifying it is denoted by \(\forall X \cdot \varphi\).
Existential quantification received more attention in the AI literature due to its more intuitive semantics which 
facilitated applications. According to this semantics, the existential quantification of variable \(X\) from
formula \(\varphi\) removes information from \(\varphi\) but it only removes information that depends 
on variable \(X\). Hence, every logical consequence of formula \(\varphi\) is also a logical consequence of its quantification 
\(\exists X \cdot \varphi\) as long as that consequence does not depend on variable \(X\). We next provide the formal definition of
existential variable quantification and its semantics.

\begin{definition}\label{def:evq}
\textcolor{\BLUE}{The existential quantification of variable \(X\) from Boolean formula \(\varphi\), denoted $\exists X \cdot \varphi$,
is defined as any Boolean formula that is logically equivalent to \((\varphi \cd x) \vee (\varphi \cd \n(x))\).}
%$$\exists X \cdot \varphi \defas (\varphi \cd x) \vee (\varphi \cd \n(x)).$$}
\end{definition}

\begin{proposition}\label{prop:evq}
\textcolor{\RED}{\(\exists X \cdot \varphi\) is a logically strongest formula that
is both independent of variable \(X\) and implied by formula \(\varphi\). 
\(\exists X \cdot \varphi\) is unique up to logical equivalence.}
\end{proposition}

Consider the formula \(\varphi = (x \then y) \wedge (y \then z)\). 
Existentially quantifying variable \(Y\) yields \(\exists Y \cdot \varphi = (x \then z)\) which removes only
the information that \(\varphi\) has about variable \(Y\).
This is why existential variable quantification is typically referred to as ``forgetting'' in the AI literature,
following~\cite{LinReiter94} who used existential quantification in a first-order setting to forget facts and relations;
see also~\cite{ki/EiterK19}. 
Existential variable quantification was employed earlier in~\cite{DBLP:conf/eds/Weber86} to
maintain consistency between two conflicting formulas (e.g., forgetting some
of our current beliefs when they conflict with new observations). 
A version of Proposition~\ref{prop:evq} appeared in~\cite{DBLP:conf/ijcai/KatsunoM89} who showed that existential 
variable quantification (referred to as {\em elimination}) is the only operator satisfying the properties stated in this proposition. 
\textcolor{\BLUE}{Existential variable quantification corresponds to a projection operation since the models
of  \(\exists X \cdot \varphi\) are precisely the models of $\varphi$ projected onto $\ps \setminus X$.
This operation has a number of applications in AI. For example, it is a key mechanism
for implementing elementary tasks (progression, regression) when reasoning about actions via
transition formulas so it has been considered for decades in planning within nondeterministic domains~\shortcite<e.g.,>{DBLP:conf/aips/CimattiRT98}.}

We next define the dual notion of universal variable quantification and its semantics.

\begin{definition}\label{def:uvq}
\textcolor{\BLUE}{The universal quantification of variable \(X\) from Boolean formula \(\varphi\), denoted \(\forall X \cdot \varphi\),
is defined as any Boolean formula that is logically equivalent to \((\varphi \cd x) \wedge (\varphi \cd \n(x))\).}
%$$\forall X \cdot \varphi \defas (\varphi \cd x) \wedge (\varphi \cd \n(x)).$$}
\end{definition}

\begin{proposition}\label{prop:uvq}
\textcolor{\RED}{\(\forall X \cdot \varphi\) is a logically weakest formula that
is both independent of variable \(X\) and implies formula \(\varphi\).
\(\forall X \cdot \varphi\) is unique up to logical equivalence.}
\end{proposition}

The following duality relates existential and universal variable quantification.
\begin{proposition}\label{prop:evquvq-duality}
\(\exists X \cdot \varphi = \neg(\forall X \cdot \neg \varphi)\) 
and \(\forall X \cdot \varphi = \neg(\exists X \cdot \neg \varphi)\).
\end{proposition}

George Boole used universal variable quantification in Chapter VII of his book~\cite{Boole1854},
which he also referred to as {\em elimination} (the chapter was titled ``On Elimination'').
Boole employed the property that \(\varphi\) is inconsistent only if \(\forall X \cdot \varphi\) is inconsistent 
to devise an inference rule which allowed him, for example, to infer \(y=y \wedge z\) 
from \(y = x \wedge z\).\footnote{Boole operated in an algebraic setting where the equivalence between \(y\) and \(x \wedge z\) 
would be written as \(y = xz\) and \(y - xz = 0\). He viewed such an expression as a function of \(x\),
\(f(x)=y - xz = 0\), and utilized universal quantification to write: \(f(x)=0\) only if \(f(0)f(1)=0\). Applying this to
the previous expression gives \(f(0)f(1)=(y-0\times z)(y-1\times z)=0\). This simplifies to \(y(y-z)=0\) 
and can be expressed as \(y=yz\), which exemplifies how Boole used universal quantification to perform deduction.}
He also observed that the order in which we quantify multiple variables does not matter.

As mentioned earlier, existential quantification received more attention than universal quantification in the AI literature.
This is largely due to its semantics as a forgetting operator and the central role that forgetting plays in 
managing inconsistent information and other areas such as planning in nondeterministic domains.
For example, a fairly general framework for reasoning with inconsistent information 
was proposed in~\cite{LangMarquis10}. The framework was based
on the notion of {\em recoveries,} which are sets of variables whose
forgetting enables one to restore consistency. Several criteria for defining 
preferred recoveries were proposed, depending on whether the focus is laid on
the relative relevance of variables or the relative entrenchment of certain information (or both). 
Forgetting has also been employed to resolve conflicts in logic programs including 
classical logic programs with negation as failure~\shortcite{Zhang-etal05,ZhangFoo06}.
Notions and techniques for forgetting in logic programs were also adapted to forgetting concepts in ontologies~\shortcite{Eiteretal06}.
Forgetting has also been widely used for defining update operators which incorporate
the effect of an action (expressed as a ``change formula'') into a base formula that represents current beliefs. 
Many update operators use the ``forget-then-conjoin'' scheme where one forgets every variable that the change formula depends on
and then conjoins the resulting base formula with the change formula~\shortcite<e.g.,>{Winslett90,Dohertyetal98,Dohertyetal00,Herzig96,HerzigRifi98,HerzigRifi99}.
Just as existential quantification can be understood as a forgetting operator, we will later show that
universal quantification can be understood as a {\em selection} operator. We will also argue that the notion of selection
is central to explainable AI as forgetting is central to belief revision and update.

\color{\RED}
Variable quantification in Boolean logic can be considered from two distinct points of view. According to one view, a
quantifier is an {\em elimination operator} which transforms one Boolean formula into another. According to
the second view, quantifiers are {\em connectives} which lead to a more general and succinct logical representation, 
known as {\em Quantified Boolean Formulas (QBFs)}. 
%The two views are fully compatible, and quantifications can be viewed as implicit operations to be achieved in a lazy fashion.
%Accordingly, QBFs with free variables appear as a valuable knowledge representation language for Boolean logic,
%and as such, some specific automated reasoning techniques have been designed for dealing with such formulas
%(see e.g., \cite{DBLP:conf/cp/KlieberJMC13,DBLP:journals/ai/FargierM14}).
While quantifiers have been mostly used as elimination operators in AI, they have been
mostly used as connectives in complexity theory~\cite<e.g.,>{Papadimitriou94}.
\color{black}
Quantification plays a key role in this context as 
the validity problem for QBFs is the canonical {\sf PSPACE}-complete
problem: there is a polynomial-space algorithm for deciding the validity of a QBF, and every decision problem 
which has a polynomial-space algorithm can be reduced efficiently into the validity problem for QBFs.
QBFs are particularly important for characterizing the polynomial hierarchy~\cite{Stockmeyer77},
where the notion of a prenex and closed QBF plays a central role. 
Consider a standard Boolean formula \(\varphi\) and let $\X_1, \ldots, \X_n$ be a partition of its
variables. A (prenex and closed) QBF has the form $Q_1 \X_1, \ldots, Q_n \X_n \cdot \varphi$ where 
$Q_1,\ldots,Q_n$ are alternating quantifiers in \(\{\forall, \exists\}$. 
%The sequence of quantifications $Q_1 \X_1, \ldots, Q_n \X_n$ is called the QBF {\em prefix} and \(\varphi\) is called the QBF  {\em matrix.} 
This formula is said to be {\em prenex} because $\varphi$ does not contain any quantifiers.
Moreover, it is said to be {\em closed} because every variable of $\varphi$ belongs to some $\X_i$.
Since every variable is quantified in a closed QBF, the QBF is either valid (equivalent to true)
or inconsistent (equivalent to false).\footnote{For instance, the QBF $\forall X \cdot (\exists Y \cdot (X \Leftrightarrow Y))$ 
is valid: whatever truth value is given to $X$, one can find a truth value for $Y$ that satisfies 
$X \Leftrightarrow Y$ (we can give $Y$ the same truth value given to $X$). 
Contrastingly, $\exists X \cdot (\forall Y \cdot (X \Leftrightarrow Y))$ is inconsistent: whatever truth value we give to $X$, 
there is a truth value for $Y$ that falsifies $X \Leftrightarrow Y$ (we can give \(Y\) the value not given to \(X\)). 
%A (standard) Boolean formula $\varphi$ with variables $\X$ is satisfiable iff the QBF
%$\exists \X \cdot \varphi$ is valid. Moreover, $\varphi$ is valid iff the QBF $\forall \X \cdot \varphi$ is valid.
}
%One can decide the validity of a QBF using Definitions~\ref{def:evq} and~\ref{def:uvq} but the brute-force procedure can take exponential time. 
A standard assumption of complexity theory is that the polynomial hierarchy does not collapse.
This is equivalent to saying that each addition of a block of quantifiers makes the validity problem computationally harder
since the validity problem for prenex and closed QBFs is $\Sigma_n^p$-complete if \(Q_1 = \exists\) 
and $\Pi_n^p$-complete if \(Q_1 = \forall\). 

While QBFs can be harder to decide compared to standard Boolean formulas, 
they do allow for exponentially more succinct encodings of problems in domains such as
planning, non-monotonic reasoning, formal verification and the synthesis of computing systems~\shortcite<e.g.,>{DBLP:series/faia/GiunchigliaMN09,DBLP:conf/ictai/ShuklaBPS19}. 
This explains the extensive efforts dedicated to developing QBF solvers in the past few decades (see \url{http://www.qbflib.org}).
\color{\RED}
QBFs with free variables are also a valuable knowledge representation language for Boolean logic.
As such, some specific automated reasoning techniques have been designed for dealing with such 
formulas~\shortcite<e.g.,>{DBLP:conf/cp/KlieberJMC13,DBLP:journals/ai/FargierM14}.
\color{black}

We will next present our treatment of literal quantification in Boolean logic which subsumes variable quantification. 
The results we shall present on literal quantification immediately translate into results on variable quantification
as we can quantify a variable by quantifying its two literals.
Our treatment of literal quantification is based on the novel notion of boundary models (Definition~\ref{def:bmodel}), 
which we investigate in the next section before using it to define the semantics of literal quantification in later sections.

\section{Boundary Models and Rules}
\label{sec:brules}

According to Definition~\ref{def:bmodel}, an \(\l\)-boundary model of formula \(\varphi\) becomes a model
for its negation \(\neg \varphi\) once we flip its literal \(\l\). As such, a model \(\w\) of \(\varphi\) can be boundary
with respect to multiple literals \(\l \in \w\). We will next introduce the notion of a {\em boundary rule} to
describe both a model and a particular literal that it is boundary on. We will then show that boundary rules
encode certain knowledge that a formula \(\varphi\) has about literals. In fact, we will show that such rules characterize all 
models of \(\varphi\) and hence they characterize its logical content.  
In later sections, we will show that quantifying a literal \(\l\) from a formula~\(\varphi\) erases
the knowledge that formula \(\varphi\) has about literal \(\l\), either by strengthening the formula (universal quantification) or by
weakening it (existential quantification). 

\begin{definition}\label{def:brule}
Let \(\ps = \{X_1, \ldots, X_n\}\) be the set of all Boolean variables. 
A \hl{boundary rule}  has the form \(\Rule({\l_1,\ldots,\l_{i-1},\l_{i+1},\ldots,\l_n},\l_i)\)
where \(\l_i\) is a literal for variable \(X_i\).
\end{definition}
\noindent Boundary rules will be referred to as {\em b-rules} for short. 
We will say that a b-rule {\em infers} the literal appearing in its consequent (\(\l_i\)). We will also say that a b-rule
{\em uses} the literals appearing in its antecedent (\(\l_1,\ldots,\l_{i-1},\l_{i+1},\ldots,\l_n\)). The set of legitimate b-rules depends on the set of
Boolean variables \(\ps\). In the following examples, and unless stated otherwise, we will assume that \(\ps\) is the set of Boolean
variables in the formulas under consideration.

\tikzstyle{nod}= [circle, draw,inner sep=0pt, minimum size=0.5cm]

\begin{figure}[tb]
\begin{center}

\scalebox{0.5}{
\begin{tikzpicture}
[every label/.append style={font=\Large}]
\pgfsetxvec{\pgfpoint{1.cm}{0.0cm}}
\pgfsetyvec{\pgfpoint{0.0cm}{1.0cm}}
\node[nod] at (0,0) [label=below:$\n(x)\n(y)\n(z)$] (000) {};
\node[nod] at (5,0) [label=below:$x\n(y)\n(z)$] (100) {};
\node[nod] at (0,5) [label=above:$\n(x)\n(y)z$] (001) {};
\node[nod] at (5,5) [label=above:$x\n(y)z$] (101) {};
\node[nod] at (2,2) [label=below:$\n(x)y\n(z)$] (010) {};
\node[nod] at (7,2) [label=below:$xy\n(z)$] (110) {};
\node[nod] at (2,7) [label=above:$\n(x)yz$] (011) {};
\node[nod] at (7,7) [label=above:$xyz$] (111) {};
\path (001)
edge (101)
edge (011)
edge (000)
(100)
edge (000)
edge (110)
edge (101)
(111)
edge (110)
edge (101)
edge (011);
\path (010)
edge (110)
edge (000)
edge (011);
\end{tikzpicture}
}
\qquad\qquad
\scalebox{0.5}{
\begin{tikzpicture}
[every label/.append style={font=\Large}]
\pgfsetxvec{\pgfpoint{1.cm}{0.0cm}}
\pgfsetyvec{\pgfpoint{0.0cm}{1.0cm}}
\node[nod,color=red,fill=red] at (0,0) [label=below:$\n(x)\n(y)\n(z)$] (000) {};
\node[nod,color=red,fill=red] at (5,0) [label=below:$x\n(y)\n(z)$] (100) {};
\node[nod,color=red,fill=red] at (0,5) [label=above:$\n(x)\n(y)z$] (001) {};
\node[nod,color=black,fill=black] at (5,5) [label=above:$x\n(y)z$] (101) {};
\node[nod,color=red,fill=red] at (2,2) [label=below:$\n(x)y\n(z)$] (010) {};
\node[nod,color=black,fill=black] at (7,2) [label=below:$xy\n(z)$] (110) {};
\node[nod,color=black,fill=black] at (2,7) [label=above:$\n(x)yz$] (011) {};
\node[nod,color=black,fill=black] at (7,7) [label=above:$xyz$] (111) {};
\path (001)
edge (000)
(100)
edge (000)
(111)
edge (011)
edge (101)
edge (110);
\path[->,color=gray, line width=3pt] (101)
edge (001)
edge (100)
(110)
edge (010)
edge (100)
(011)
edge (001)
edge (010);
\path (010)
edge (000);
\end{tikzpicture}
}
\end{center}
\caption{Left: Visualizing the worlds over variables \(X, Y, Z\). Two worlds are connected by an edge iff they disagree on a single variable.
Right: Visualizing the models and b-rules of formula $\varphi = (x \vee y) \wedge (x \vee z) \wedge (y \vee z).$
Black nodes are models of \(\varphi\) and red nodes are models of \(\neg \varphi\). Each highlighted edge corresponds
to a b-rule for \(\varphi\).
\label{fig:cube}}
\end{figure}
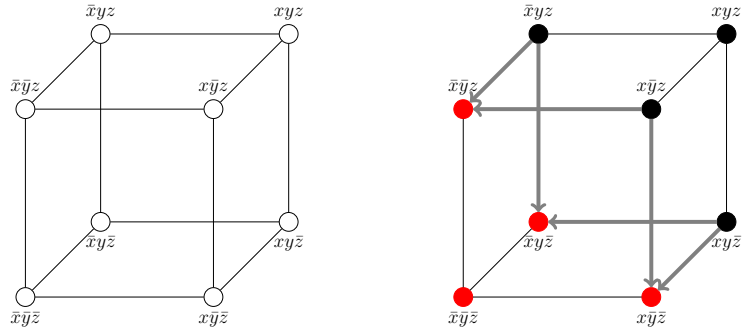

\begin{definition}\label{def:brule}
We say formula \(\varphi\) has b-rule \(\Rule(\alpha,\l)\) iff \(\alpha,\l\) is an \(\l\)-boundary model of \(\varphi\).
\end{definition}

\noindent This definition establishes b-rules as descriptors of boundary models. 
We will use \(\RS(\varphi)\) to denote the set of b-rules for formula \(\varphi\). 
Recall Definition~\ref{def:bmodel} which says that \(\alpha,\l\) is an \(\l\)-boundary model of \(\varphi\) 
precisely when \(\alpha,\l \models \varphi\) and \(\alpha,\nl \models \neg \varphi\).
%A boundary model may be described by multiple b-rules (at most \(n\) b-rules where \(n\) is the number of variables). 
Consider now the formula $\varphi = (x \vee y) \wedge (x \vee z) \wedge (y \vee z)$
which has four models \(\MS(\varphi)=\{\n(x)yz, x\n(y)z, xy\n(z), xyz\}$ and six b-rules $\RS(\varphi) =
\{\Rule({\n(x)z},y),$  
$\Rule({\n(y)z},x),$  
$\Rule({\n(x)y},z),$ 
$\Rule({x\n(y)},z),$
$\Rule({y\n(z)},x),$ 
$\Rule({x\n(z)},y)\}$.
Figure~\ref{fig:cube} (left) visualizes all eight worlds over variables \(X,Y,Z\) using a hypercube. 
Figure~\ref{fig:cube} (right) visualizes the models of formula \(\varphi\)
and its b-rules. Consider model \(\w=\n(x)yz\). This is a \(y\)-boundary model
and a \(z\)-boundary model, so it is described by two b-rules: 
\(\Rule({\n(x)z},y)\) and $\Rule({\n(x)y},z).$ The first b-rule tells us that \(\w\) will become
a model of \(\neg \varphi\) if we flip literal \(y\) (\(\n(x)\n(y)z \models \neg \varphi\)). 
The second b-rule tells us that \(\w\) will become a model of \(\neg \varphi\) if we flip literal \(z\) (\(\n(x)y\n(z)\models \neg \varphi\)).
Since \(\Rule({yz},{\n(x)})\) is not a b-rule of formula \(\varphi\) we know that flipping literal \(\n(x)\) will maintain \(\w\) 
as a model (\(xyz \models \varphi\)). Interestingly, \(xyz\) is not a boundary model of formula \(\varphi\) yet we 
were able to conclude that it is a model of \(\varphi\) using only the b-rules of \(\varphi\).
The next theorem provides a stronger result: one can compute all models of a consistent formula using its b-rules.
That is, the b-rules of a consistent formula characterize its logical content.

\begin{theorem}\label{theo:know}
Two consistent formulas \(\varphi_1\) and \(\varphi_2\) are equivalent iff they have the same set of b-rules:
\(\MS(\varphi_1)=\MS(\varphi_2)\) iff \(\RS(\varphi_1)=\RS(\varphi_2)\).
\end{theorem}

\noindent According to this result, the boundary models of a formula, as described by its b-rules, form a {\em generating set} from which 
all models can be constructed. The proof of Theorem~\ref{theo:know} provides an explicit construction. 

The following proposition shows that b-rules capture certain knowledge about literals.
\begin{proposition} \label{prop:brule}
Formula \(\varphi\) has b-rule \(\Rule(\alpha,\l)\) iff \(\varphi \wedge \alpha\) is consistent and \(\varphi \models \alpha \then \l\). 
\end{proposition}

\noindent That is, a b-rule \(\Rule(\alpha,\l)\) can be viewed as encoding a feasible scenario (\(\alpha\)) under which
literal \(\l\) can be inferred. Formula \(\varphi = x \Leftrightarrow y\) has four b-rules
\(\RS(\varphi) = \{\Rule(x,y),\:\Rule(y,x),\:\Rule({\n(x)},{\n(y)}),\:\Rule({\n(y)},{\n(x)})\}\)
and formula \(\phi = (x \vee y)\wedge z\) has five b-rules
\(\RS(\phi) = \{\Rule({\n(y)z},x),\:\Rule({\n(x)z},y),\:\Rule({xy},z),\:\Rule({x\n(y)},z),\:\Rule({\n(x)y},z)\}\). 
Even though literal \(z\) is implied by formula \(\phi\) and literals \(\n(x)\n(y)\), the formula does not
have \(\Rule({\n(x)\n(y)},z)\) as a b-rule since \(\n(x)\n(y)\) is not consistent with \(\phi\).
%That is, literal \(z\) is being implied trivially, not due to knowledge that \(\phi|\n(x)\n(y)\) has about \(z\).
Similarly, literal  \(\n(z)\) is implied by  formula \(\phi\) and literals 
\(\n(x)\n(y)\) but the formula does not have \(\Rule({\n(x)\n(y)},{\n(z)})\) as a b-rule.

A formula is independent of a literal (variable) precisely when the formula has no b-rules for that literal (variable).
\begin{proposition}\label{prop:rule-independence}
\textcolor{\BLUE}{Formula \(\varphi\) is independent of literal \(\ell\) iff it has no b-rules of the form \(\Rule(\alpha,\ell)\).
It is independent of variable \(X\) iff it has no b-rules of the form \(\Rule(\alpha,x)\) or \(\Rule(\alpha,{\n(x)})\).}
\end{proposition}
\noindent \textcolor{\BLUE}{Recall that formula \(\varphi\) is independent of literal \(\ell\) precisely when it can be expressed as an
NNF that does not mention literal \(\ell\).} Similarly, formula \(\varphi\) is independent of variable \(X\) precisely 
when it can expressed in a form that does not mention variable \(X\). 

%We close this section by noting that while models are compositional, b-rules are not. 
\textcolor{\BLUE}{While models are compositional, b-rules are not.
However, the b-rules of a formula are determined by the b-rules of its negation. This is shown by the next proposition.}

\begin{proposition}\label{prop:rconnect}
For formulas $\varphi$ and $\psi$, we have:
\begin{itemize}
\item[(a)] $\RS(\neg \varphi) = \{\alpha \rightarrow \l \mid \alpha \rightarrow \nl \in \RS(\varphi)\}.$
\item[(b)] $\RS(\varphi) \cap \RS(\psi) \subseteq \RS(\varphi \wedge \psi) \subseteq \RS(\varphi) \cup \RS(\psi).$
\item[(c)] $\RS(\varphi) \cap \RS(\psi) \subseteq \RS(\varphi \vee \psi) \subseteq \RS(\varphi) \cup \RS(\psi).$
\end{itemize}
\end{proposition}
\noindent According to this proposition, conjoining or disjoining two formulas \(\varphi\) and \(\psi\) preserves only their common b-rules.
Moreover, \(\RS(\varphi \wedge \psi)\) and \(\RS(\varphi \vee \psi)\) are not connected. Neither is contained
in the other and their intersection may be empty. Consider the formulas \(\varphi = x\) and \(\psi = y\). 
We have \(\RS(\varphi \wedge \psi) = \{\Rule(x,y), \Rule(y,x)\}\) and \(\RS(\varphi \vee \psi) = \{\Rule(\n(x),y), \Rule(\n(y),x)\}\).

\color{\BLUE}
We will conclude this section with a result that relates the number of b-rules for a Boolean formula to the number of its models, counter models
and boundary models.
%shows that the number of b-rules (and boundary models) can be exponentially smaller than the number of models.

\begin{proposition}\label{prop:rcount}
For formula \(\varphi\) and \(n=|\ps|\), we have
$$|\RS(\varphi)| \leq n \cdot |\BMS(\varphi)| \leq n \cdot \mathit{min}(|\MS(\varphi)|, |\MS(\overline{\varphi})|).$$
\end{proposition}

The bound \(|\RS(\varphi)| \leq n \cdot |\BMS(\varphi)|\) is tight. For instance, $\varphi = x_1\wedge \ldots\wedge x_n$
has a single boundary model but \(n\) b-rules, $\RS(\varphi) = \{(\bigwedge_{j = 1, \ldots, n \mid j \neq i} x_j) \rightarrow x_i \mid i = 1, \ldots, n\}.$
We also remark that \(|\RS(\varphi)|\) and \(|\BMS(\varphi)|\) can be exponentially smaller than \(|\MS(\varphi)|\).
Consider the formula $\varphi = x_1 \vee \ldots \vee x_n$ which has $2^n -1 = |\MS(\varphi)|$ models.
This formula has only \(n\) b-rules, $\RS(\varphi) = \{\Rule({(\bigwedge_{j = 1, \ldots, n \mid j \neq i} {\n(x)}_j)}, x_i) \mid i = 1, \ldots, n\}$,
and \(n\) boundary models, $\BMS(\varphi) = \{x_i \wedge \bigwedge_{j = 1, \ldots, n \mid j \neq i} {\n(x)}_j \mid i = 1, \ldots, n\}$.
Finally, the number of b-rules for a formula can be exponential in the number of its variables. This also holds for the
number of its models and the number of its counter models. For instance, the formula $\varphi = \oplus_{i=1}^n x_i$ has
$2^{n-1}$ models, $2^{n-1}$ counter models, and $2^{n-1} \cdot (n-1)$ b-rules.

\color{black}

\section{Literal Quantification}
\label{sec:lq}

We next discuss existential and universal literal quantification. We start by introducing and studying universal literal
quantification and then review and study further existential literal quantification. The latter type of quantification was 
first introduced and studied in~\cite{LangLM03} under the name of {\em literal forgetting.} Some of the results we shall 
present on universal literal quantification follow from known results on existential literal quantification due to a duality 
between the two notions. We will also present new and fundamental results, based on boundary and independent models,
which apply to both types of literal quantification, again due to the duality between them. Our main goal of the 
upcoming study is to develop an intuitive semantics for universal quantification (as a selection process) and then
use this semantics to show its central role in explainable AI. Additionally, our treatment will reveal new results on
the computation of universal quantification which have complexity implications for explainable AI queries.

\subsection{Universal Literal Quantification}
\label{sec:ulq}

Before we define the universal quantification of literal \(\l\) from a formula \(\varphi\), we note that~\(\varphi\) 
can be expanded as \(\varphi = (\l \vee (\varphi \cd \nl)) \wedge (\nl \vee (\varphi \cd \l))\), which is equivalent to
what is known as Boole's or Shannon's expansion, $\varphi = (\l \wedge (\varphi \cd \l)) \vee (\nl \wedge (\varphi \cd \nl))$.

\begin{definition} \label{def:ulq}
Universally quantifying literal \(\l\) from formula \(\varphi\) is defined as follows:
\[\forall \l \cdot \varphi \defas (\l \vee (\varphi \cd \nl)) \wedge  (\varphi \cd \l). \]
\end{definition}
\noindent That is, the operator \(\forall \l\)  drops literal \(\nl\) from the expansion \(\varphi = (\l \vee (\varphi \cd \nl)) \wedge (\nl \vee (\varphi \cd \l))\).
Consider the formula \(\varphi = (x \then y) \wedge (y\then x)\) which says that variables \(X\) and \(Y\) are equivalent.
We have \(\forall x \cdot \varphi = x \wedge y\) and \(\forall \n(x) \cdot \varphi = \n(x) \wedge \n(y)\). Note, however,
that \(\forall X \cdot \varphi = \bot\). Moreover, \(\forall x (\forall \n(x) \cdot \varphi) = \forall \n(x) (\forall x \cdot \varphi) = \bot\).
Consider now the formula \(\phi = (x \then y)\).
We now have \(\forall x \cdot \phi = y\), \(\forall \n(x) \cdot \phi = (x \then y)\) and \(\forall X \cdot \phi = y\).
Moreover, we have 
\(\forall x (\forall \n(x) \cdot \phi) = \forall \n(x) (\forall x \cdot \phi) = y\).

%Pierre: should we give some credit to a reviewer for this remark?
\textcolor{\BLUE}{Since $(\l \vee (\varphi \cd \nl)) \wedge  (\varphi \cd \l)$ is equivalent to $(\l \wedge  (\varphi \cd \l)) \vee ((\varphi \cd \nl) \wedge  (\varphi \cd \l))$, and
since $(\varphi \cd \nl) \wedge  (\varphi \cd \l)$ is equivalent to $\forall X \cdot \varphi,$ where $X$ is the variable of literal $\l$, we get
$\forall \l \cdot \varphi = (\l \wedge  (\varphi \cd \l)) \vee (\forall X \cdot \varphi)$. This shows that $\forall \l \cdot \varphi$ can be obtained by adding to 
$\forall X \cdot \varphi$ all models of $\varphi$ that contain literal $\l$ and that are ruled out by variable quantification.}

We next provide a number of results on universal literal quantification which have novel counterparts for existential
literal quantification that we present in Section~\ref{sec:elq}.
The first result provides a {\em semantical characterization} of universal literal quantification based on the notion of boundary models. 

\begin{theorem}\label{theo:ulq-semantics-b}
For formula \(\varphi\) and literal \(\l\), we have \(\MS(\forall \l \cdot \varphi) \subseteq \MS(\varphi)\).
Moreover, \(\w \in \MS(\varphi)\) and \(\w \not \in \MS(\forall \l \cdot \varphi)\) iff \(\w\) is an \(\nl\)-boundary model of \(\varphi\).
\end{theorem}
\noindent That is, universally quantifying literal \(\l\) from formula \(\varphi\) strengthens the formula by dropping its \(\nl\)-boundary models.
These models contain and depend on literal \(\nl\): they cease to be models of \(\varphi\) if we were to flip literal \(\nl\).
Hence, we can view the operator \(\forall \l\) as a selection operator which picks models of \(\varphi\) that do not depend on literal \(\nl\). 
The selected models either contain literal \(\l\) or an irrelevant literal \(\nl\). We will later present a more general
result that provides selection semantics when universally quantifying multiple literals. 

Our second result provides a {\em syntactic characterization} of universal literal quantification based on the notion of b-rules.
In particular, it characterizes what gets added to a formula upon strengthening it by universal quantification. 

\begin{theorem}\label{theo:ulq-syntax}
For formula $\varphi$ and literal \(\l\), we have 
$\forall \l \cdot \varphi = \varphi \wedge \bigwedge_{\Rule(\alpha,\nl) \in \RS(\varphi)} \neg \alpha.$
\end{theorem}
\color{\RED}
\noindent That is, the operator \(\forall \l\) adds \(\neg \alpha\) to formula \(\varphi\) for each b-rule \(\Rule(\alpha,\nl)\) of the formula. 
This erases all these b-rules as their antecedents \(\alpha\) will no longer be consistent with the quantified formula,
therefore erasing the knowledge that formula \(\varphi\) has about literal \(\nl\).
This also makes the quantified formula independent of literal \(\nl,\) as shown by Proposition~\ref{prop:rule-independence},
which leads to the elimination of literal \(\nl\). 
\color{black}

The third result says that the universal quantification of literal \(\l\) preserves implicants that contain literal \(\l\).
\begin{proposition}\label{prop:ulq-preserve}
For formula \(\varphi\), term \(\gamma \models \varphi\) and literal \(\l \in \gamma\), we have \(\gamma \models \forall \l \cdot \varphi\).
\end{proposition}

The next result says that universal literal quantification preserves logical implication.
\begin{proposition}\label{prop:ulq-imply}
For formulas \(\varphi\), \(\phi\) and literal \(\l\), we have
\(\varphi \models \phi\) only if \(\forall \l \cdot \varphi \models \forall \l \cdot \phi\).
\end{proposition}

The following three results (Propositions~\ref{prop:ulq-semantics-old}-\ref{prop:ulq-uvq}) parallel ones that are known for existential quantification.
The first result provides a semantical characterization of universal literal quantification based on the notion of literal independence.

\begin{proposition}\label{prop:ulq-semantics-old}
\(\forall \l \cdot \varphi\) is the logically weakest formula that is independent of literal \(\nl\) and that also implies formula \(\varphi\).
\end{proposition}

The second result shows that a set of literals can be universally quantified in any order, therefore justifying the notation \(\forall \{\l_1, \ldots, \l_n\} \cdot \varphi\).

\begin{proposition}\label{prop:ulq-order}
For literals \(\l_1\), \(\l_2\) and formula \(\varphi\), we have \(\forall \l_1 (\forall \l_2 \cdot \varphi) = \forall \l_2 (\forall \l_1 \cdot \varphi)\).
\end{proposition}
\noindent We should note that literals \(\l_1\) and \(\l_2\) can be for the same variable and hence conflicting.

The third result shows that universal literal quantification is more fine-grained than universal variable quantification
as we can universally quantify variable \(X\) by universally quantifying literals \(x\) and \(\n(x)\) in any order. 
\begin{proposition}\label{prop:ulq-uvq}
For variable \(X\) and formula \(\varphi\), we have \(\forall X \cdot \varphi = \forall \{x, \n(x)\} \cdot \varphi\).
\end{proposition}
\noindent Together with Proposition~\ref{prop:ulq-order}, this result shows that we can universally quantify a set of
literals and variables in any order. For example, the following quantifications are all legitimate and equivalent:
\(\forall x, \n(x), Y \cdot \varphi\),  \(\forall x, Y, \n(x) \cdot \varphi\) and \(\forall Y, X \cdot \varphi\).

We are now ready to present our {\em selection semantics} for universally quantifying
a set of (possibly conflicting)  literals. This result will play a major role when discussing the applications of universal quantification to explainable
AI in Section~\ref{sec:XAI}. It invokes the notion of an \(\alpha\)-independent model introduced in Definition~\ref{def:imodel}.
This is a model that contains term \(\alpha\) and that remains a model if we were to flip any literals of~\(\alpha\). 

\begin{theorem}\label{theo:ulq-semantics}
Let \(\varphi\) be a formula, \(\l_1, \ldots, \l_n\) be literals, \(\w\) be a world and \(\alpha = \w \cap \{\nl_1,\ldots,\nl_n\}\).
Then
\(\w \models \forall \l_1, \ldots, \l_n \cdot \varphi\) iff \(\w\) is an \(\alpha\)-independent model of \(\varphi\).
\end{theorem}

\color{\RED}
Consider $\varphi = (x \vee y \vee z) \wedge (\n(x) \vee y \vee t)$, literals \(x,y,z\) and world $\omega = \n(x)\n(y)zt$,
leading to $\alpha = \w\cap\{\n(x),\n(y),\n(z)\}= \n(x)\n(y)$. Then $\omega$ is an $\alpha$-independent 
model of $\varphi$ since $zt \models \varphi$. Hence, $\omega \models \forall x, y, z \cdot \varphi$ which can be verified since
$\forall x, y, z \cdot \varphi = (x \vee y \vee z) \wedge (y \vee t)$.
For literals \(x, \n(y), \n(z)\), we get \(\alpha = \w \cap \{\n(x),y,z \} = \n(x)z\) so \(\w\)
is not an $\alpha$-independent model of $\varphi$ since $\n(y)t \not \models \varphi$ (indeed, $\n(x)\n(y)\n(z)t \models \nvarphi$).
Hence, $\omega \not \models \forall x, \n(y), \n(z) \cdot \varphi$ which can be verified since $\forall x, \n(y), \n(z) \cdot \varphi = x \wedge t$.
\color{black}

According to Theorem~\ref{theo:ulq-semantics}, when universally quantifying literals \(\l_1, \ldots, \l_n\) from formula \(\varphi\), 
we are ``selecting'' all (and only) models of \(\varphi\) that do not depend on literals \(\nl_1, \ldots, \nl_n\).
That is, the models of \(\forall \l_1, \ldots, \l_n \cdot \varphi\) are precisely the models of \(\varphi\) which 
continue to be models of \(\varphi\) if we were to flip any literals they may have 
in \(\nl_1, \ldots, \nl_n\).\footnote{Each model of \(\forall \l_1, \ldots, \l_n \cdot \varphi\) is also an \(\alpha\)-independent model of 
\(\forall \l_1, \ldots, \l_n \cdot \varphi\); see Lemma~\ref{lem:imodel-w2s}.}
We will revisit Theorem \ref{theo:ulq-semantics} in detail when we discuss explainable AI in Section~\ref{sec:XAI},
but we note here that the presence of \(\alpha\)-independent models indicates the absence of certain 
b-rules.\footnote{Let us say that term \(\gamma\) is an \(\alpha\)-independent implicant
of formula \(\varphi\) iff \(\gamma \models \varphi\),  \(\alpha \subseteq \gamma\) and \(\gamma\setminus\alpha \models \varphi\). Then
Theorem~\ref{theo:ulq-semantics} will hold if we replace ``world" by ``term" and ``model" by ``implicant."}

\begin{proposition}\label{prop:imodel-brule}
Let \(\w = \alpha,\beta\) be a model of formula \(\varphi\) where \(\alpha\) and \(\beta\) are disjoint terms.
Then \(\w\) is an \(\alpha\)-independent model of \(\varphi\) iff \(\varphi\) has no b-rules of the form \(\Rule({\beta,\gamma},\l)\).
\end{proposition}

\color{\BLUE}
In Appendix~\ref{app:b-rule-dynamics}, we provide a complete characterization of which b-rules are deleted, introduced or preserved
when universally quantifying a literal. This allows us to define quantification as a process of b-rule transformation. 
\color{black}

\subsection{Existential Literal Quantification} 
\label{sec:elq}

We next review and study further the existential quantification of literals which was first introduced 
and studied in~\cite{LangLM03} under the name of {\em literal forgetting.}
Before we define the existential quantification of literal \(\l\) from a formula \(\varphi\), we recall again
the well known 
Boole's or Shannon's expansion, $\varphi = (\l \wedge (\varphi \cd \l)) \vee (\nl \wedge (\varphi \cd \nl))$.

\begin{definition}\label{def:elq}
Existentially quantifying literal \(\l\) from formula \(\varphi\) is defined as follows:
\[\exists \l \cdot \varphi \defas (\varphi \cd \l) \vee (\nl \wedge (\varphi \cd \nl)). \]
\end{definition}
\noindent That is, the operator \(\exists \l\) drops literal \(\l\) from the expansion \(\varphi = (\l \wedge (\varphi \cd \l)) \vee (\nl \wedge (\varphi \cd \nl))\).

Consider the formula \(\varphi = (x \then y)\wedge(y\then x)\) which says that variables \(X\) and \(Y\) are equivalent.
We have \(\exists x \cdot \varphi = (x \then y) = (\n(x)\vee y)\) which is independent of literal \(x\).
We also have \(\exists \n(x) \cdot \varphi = (y \then x) = (\n(y)\vee x)\) which is independent of literal \(\n(x)\).
Note, however, that \(\exists X \cdot \varphi = \top\).
We finally have \(\exists x (\exists \n(x) \cdot \varphi) = \exists \n(x) (\exists x \cdot \varphi) = \top\).

\shrink{
%%%Adnan: I removed this as it is not clear
\textcolor{\RED}{Since $(\varphi \cd \l) \vee (\nl \wedge (\varphi \cd \nl))$ is equivalent to 
$((\varphi \cd \l) \vee \nl) \wedge ((\varphi \cd \l) \vee  (\varphi \cd \nl))$, and
since $((\varphi \cd \l) \vee  (\varphi \cd \nl))$ is equivalent to $\exists X \cdot \varphi$ where $X$ is the variable of $\l$, we get
$\exists \l \cdot \varphi = ((\varphi \cd \l) \vee \nl) \wedge (\exists X \cdot \varphi)$.
This shows that $\exists \l \cdot \varphi$ can be obtained from 
$\exists X \cdot \varphi$ by removing all models that contain literals $\l$ and that are not a model of $\varphi$.}
}

Existential and universal literal quantification are related by the following duality.
\begin{theorem}\label{theo:lq-duality}
For literal \(\l\) and formula \(\varphi\),
\(\exists \l \cdot \varphi = \neg(\forall \l \cdot \neg \varphi)\) and 
\(\forall \l \cdot \varphi = \neg(\exists \l \cdot \neg \varphi)\).
\end{theorem}
\noindent This is symmetric to the variable quantification duality: \(\exists X \cdot \varphi = \neg (\forall X \cdot \neg \varphi)\)
and \(\forall X \cdot \varphi = \neg (\exists X \cdot \neg \varphi)\). In both cases, pushing a negation through
a quantifier flips that quantifier.

The next three results come from~\cite{LangLM03}.
The first result provides a semantical characterization of existential literal quantification using the notion of independence.
\begin{proposition}\label{prop:elq-semantics-old}
\(\exists \l \cdot \varphi\) is the logically strongest formula that is independent of literal \(\l\) and that is implied by formula \(\varphi\). 
\end{proposition}
Similar to existential variable quantification, existentially quantifying literal \(\l\) removes information from formula \(\varphi\)
but it only removes information that depends on literal \(\l\). That is, every consequence of \(\varphi\) is also a consequence 
of \(\exists \l \cdot \varphi\) as long as that consequence does not depend on literal \(\l\). Literal quantification
provides a more fine-grained notion of forgetting, which enables more refined treatments. This is particularly the case when managing 
inconsistent information as we can now forget less information from one formula to make it consistent with another. 
Literal forgetting was employed in extended logic programs~\cite{Wang-etal05} and disjunctive logic programs~\cite{EiterWang06,EiterWang08}. 
It was also used to define more refined update operators using the ``forget-then-conjoin'' scheme that we discussed earlier~\cite{Herzigetal13}.

The second results from~\cite{LangLM03} shows that the order in which literals are existentially quantified does not matter,
which justifies the notation \(\exists \{\l_1, \ldots, \l_n\} \cdot \varphi\).
\begin{proposition}\label{prop:elq-order}
For literals \(\l_1\), \(\l_2\) and formula \(\varphi\), we have \(\exists \l_1 (\exists \l_2 \cdot \varphi) = \exists \l_2 (\exists \l_1 \cdot \varphi)\).
\end{proposition}

The third result shows that existential literal quantification is more primitive than existential variable quantification
as we can existentially quantify a variable \(X\) by existentially quantifying its two literals \(x\) and \(\n(x)\) in any order.
\begin{proposition}\label{prop:elq-evq}
For variable \(X\) and formula \(\varphi\), we have \(\exists X \cdot \varphi = \exists \{x, \n(x)\} \cdot \varphi\).
\end{proposition}
\noindent Together with Proposition~\ref{prop:elq-order}, this result shows that we can existentially quantify a set of
literals and variables in any order. 

The next set of results are new and follow from results we presented for universal literal quantification
using the duality between existential and universal literal quantification.
The first result provides a {\em semantical characterization} of existential literal quantification.

\begin{theorem}\label{theo:elq-semantics-b}
For formula $\varphi$ and literal $\l$, we have $M(\varphi) \subseteq M(\exists \l \cdot \varphi)$. Moreover,
$\w \in M(\exists \l \cdot \varphi)$ and $\w \not \in M(\varphi)$ iff $\w$ is an $\nl$-boundary model of $\neg\varphi$.
%iff $\w[\l]$ is an $\l$-boundary model of $\varphi$.
\end{theorem}
\noindent That is, existentially quantifying literal \(\l\) from formula \(\varphi\) adds the
\(\nl\)-boundary models of \(\neg\varphi\). These models must contain literal \(\nl\) so no model that contains
literal \(\l\) is added. 

The second result provides a {\em syntactic characterization} of existential literal quantification using the notion of b-rules.

\begin{theorem}\label{theo:elq-syntax}
For formula $\varphi$ and literal \(\l\), we have 
$\exists \l \cdot \varphi = \varphi \vee \bigvee_{\Rule(\alpha,\l) \in \RS(\varphi)} \alpha.$
\end{theorem}

\color{\RED}
\noindent That is, existentially quantifying literal \(\l\) from formula \(\varphi\) amounts to disjoining \(\varphi\)
with the antecedent \(\alpha\) of each of its b-rules \(\Rule(\alpha,\l)\). Given b-rule \(\Rule(\alpha,\l)\), world
\(\alpha,\l\) is a model of \(\varphi\) but world \(\alpha,\nl\) is not a model. 
Disjoining with the antecedent \(\alpha\) adds \(\alpha,\nl\) as a model. This erases b-rules \(\Rule(\alpha,\l)\) and the knowledge
that \(\varphi\) has about literal \(\l\).
This also makes the quantified formula independent of literal \(\l\), therefore eliminating this literal.
\color{black}

The third result says that existentially quantifying literal \(\l\) preserves implicates that contain literal \(\neg \l\).
\begin{proposition}\label{prop:elq-preserve}
If formula $\varphi$ implies clause $\beta$ and literal $\nl \in \beta$, then $\exists \l \cdot \varphi \models \beta$.
\end{proposition}

The last result says that existential literal quantification preserves logical implication.
\begin{proposition}\label{prop:elq-imply}
For formulas \(\varphi\), \(\phi\) and literal \(\l\), we have \(\varphi \models \phi\) only if \(\exists \l \cdot \varphi \models \exists \l \cdot \phi\).
\end{proposition}

One can use the duality theorem to state further results such as the parallel of Theorem~\ref{theo:ulq-semantics}. We will
refrain from discussing these additional results though as our goal is to focus more on universal literal quantification and its
applications to explainable AI.

\color{\RED}
\subsection{Is it Quantification or Elimination?}

As mentioned earlier, variable quantification in Boolean logic can be viewed as an elimination process that transforms a formula
so it becomes independent of the quantified variable. In particular, quantifying a variable leads to a new formula
that can be expressed without mentioning the variable. In fact, George Boole used
the term ``elimination'' but the term ``quantification'' is now more commonly used. Moreover, ``variable elimination'' has become a
synonym for ``existential variable quantification'' in some parts of the literature, including many works on satisfiability
where existential variable quantification is used as a preprocessing technique~\cite<e.g.,>{sat/EenB05,NiVER}. 
The distinction between elimination and quantification becomes more relevant though when treating literals instead of variables as we show next.  
Let us consider variables first.
\(\exists X . \varphi\) eliminates variable \(X\) by weakening formula \(\varphi\): It {\em adds} world \(\w\) as a model iff 
\(\w[\l]\) is a model of \(\varphi\) for {\em some} literal \(\l\) of variable \(X\). 
\(\forall X . \varphi\) eliminates variable \(X\) by strengthening formula \(\varphi\): It {\em keeps} world \(\w\) as a model iff 
\(\w[\l]\) is a model of \(\varphi\) for {\em all} literals \(\l\) of \(X\).
The fundamental operation here is one of elimination, which can be achieved in a unique way (up to logical equivalence)
through either weakening (adding models) or strengthening (dropping models). Moreover, the conditions for deciding
which models to add or drop are based on the two states of variable \(X\), which are considered either existentially or universally. 
Eliminating literals is also done by weakening or strengthening a formula in a unique way, but the conditions for deciding which
models to add or drop consider only one state of the variable. In particular,
\(\exists \l . \varphi\) eliminates literal \(\l\) by adding world \(\w\) as a model iff \(\w[\l]\) is a model of \(\varphi\). 
Moreover, \(\forall \l . \varphi\) eliminates literal \(\nl\) by keeping world \(\w\) as a model iff \(\w[\l]\) is a model of \(\varphi\).
In a sense, the term elimination is perhaps more appropriate than quantification in this case, and the symbols \(\exists\) and \(\forall\)
are more indicative of elimination by weakening and strengthening than anything else. One could have adopted different symbols
and terminology for literals, but we opted to keep the same ones used for variables as this emphasizes
the symmetries between variable and literal elimination, as revealed by the many results we discussed earlier.
\color{black}

\section{Tractable Literal Quantification}
\label{sec:tractable}

We identify in this section some classes of Boolean formulas and circuits that allow one to quantify literals efficiently.
We consider CNFs, DNFs and some of their subsets. 
We also consider two circuit types: Decision-DNNFs~\cite{jair/HuangD07} and SDDs~\cite{ijcai/Darwiche11},
which are strict supersets of OBDDs~\cite{tc/Bryant86}.
All of these circuit types are subsets of NNF circuits~\cite{jair/DarwicheM02}.

We start with two results which provide the foundation of further results. 
The first shows how to quantify literals out of constants and literals. 
The second identifies conditions under which literal  quantification can be distributed through conjunctions and disjunctions.
\textcolor{\RED}{When these conditions are met, quantification can be done in linear time on NNF formulas and circuits.
The second result corresponds to a well-known result for variable quantification which we show does
extend to the more fine-grained notion of literal quantification.}

\begin{proposition}\label{prop:quantify-base}
For literal \(\l\), we have \(\exists \l \cdot \top = \forall \l \cdot \top = \top\) and \(\exists \l \cdot \bot = \forall \l \cdot \bot = \bot\).
Moreover, for literals \(\l_1\) and \(\l_2\), we have:
\[
\exists \l_1 \cdot \l_2 = \left\{
\begin{array}{ll}
\top, & \mbox{if \(\l_1 = \l_2\)} \\
\l_2, & \mbox{otherwise.}
\end{array}
\right.
\qquad 
\mbox{and} 
\qquad
\forall \l_1 \cdot \l_2 = \left\{
\begin{array}{ll}
\bot, & \mbox{if \(\l_1 = {\n(\l)}_2\)} \\
\l_2, & \mbox{otherwise.}
\end{array}
\right.
\]
\end{proposition}

\begin{proposition}\label{prop:quantify-compound}
Consider literal \(\l\) and formulas \(\alpha\) and \(\beta\). We then have
\begin{itemize}
\item[(a)] \(\exists \l \cdot  (\alpha \vee \beta) = (\exists \l \cdot \alpha) \vee (\exists \l \cdot \beta)\)
\item[(b)] \(\forall \l  \cdot   (\alpha \wedge \beta) = (\forall \l \cdot  \alpha) \wedge (\forall \l \cdot  \beta)\)
\end{itemize}
Moreover, if the variable of literal \(\l\) is not shared between \(\alpha\) and \(\beta\), then
\begin{itemize}
\item[(c)] \(\exists \l  \cdot   (\alpha \wedge \beta) = (\exists \l \cdot \alpha) \wedge (\exists \l \cdot  \beta)\)
\item[(d)] \(\forall \l  \cdot  (\alpha \vee \beta) = (\forall \l \cdot  \alpha) \vee (\forall \l \cdot \beta)\)
\end{itemize}
\end{proposition}
\noindent This proposition holds also for variable quantification as it can be implemented through successive 
literal quantification as shown by Propositions~\ref{prop:ulq-uvq} and~\ref{prop:elq-evq}.

These propositions imply direct methods for quantifying literals
out of clauses and terms which can be useful when working with CNFs and DNFs.
Existentially quantifying literal \(\l\) from a clause leads to \(\top\) if literal \(\l\) appears in the clause, otherwise the clause is left intact.
Universally quantifying literal \(\l\) drops literal \(\nl\) from the clause. 
Existentially quantifying a literal \(\l\) from a term drops literal \(\l\) from the term.
Universally quantifying literal \(\l\) leads to \(\bot\) if \(\nl\) appears in the term, otherwise the term is left intact.

We now consider some classes of Boolean formulas and circuits that allow one to
quantify literals efficiently. We start by considering CNF, DNF and some of their subsets. 
Recall that we do not allow trivial clauses or trivial terms that include complementary literals.
Hence, clauses cannot be valid and terms cannot be inconsistent.

We first consider CNFs while noting that the procedures discussed below guarantee that the result of quantifying
literals is also a CNF.
\begin{proposition}[CNF]\label{prop:q-cnf}
Consider a CNF \(\Delta = \alpha_1 \wedge \ldots \wedge \alpha_n\). 
We can obtain \(\forall \l \cdot \Delta\) by removing literal \(\nl\) from every clause \(\alpha_i\) in \(\Delta\).
Moreover, if \(\Delta\) is closed under resolution on the variable of literal \(\l\), then
we can obtain \(\exists \l \cdot \Delta\) by removing from \(\Delta\) every clause \(\alpha_i\) that contains literal \(\l\).
\end{proposition}

\begin{corollary}\hspace{-2mm}\footnote{This corollary slightly extends Proposition 19 in~\cite{LangLM03},
which focuses on the case when $\Delta$ is given by the set of all its prime implicates.}
\label{coro:q-cnf}
Consider a CNF \(\Delta = \alpha_1 \wedge \ldots \wedge \alpha_n\) where each clause \(\alpha_i\) is a
prime implicate of \(\Delta\). We can obtain \(\exists \l \cdot \Delta\) by removing from \(\Delta\)
clauses that contain literal \(\l\).
\end{corollary}

We next consider DNFs. The procedures discussed below also guarantee that the result of quantifying literals is a DNF.
\begin{proposition}[DNF]\label{prop:q-dnf}
Consider a DNF \(\Delta = \alpha_1 \vee \ldots \vee \alpha_n\). 
We can obtain \(\exists \l \cdot \Delta\) by removing literal \(\l\) from every term \(\alpha_i\) in \(\Delta\).
Moreover, if \(\Delta\) is closed under consensus on the variable of literal~\(\l\), 
then we can obtain \(\forall \l \cdot \Delta\) by removing from \(\Delta\) every term \(\alpha_i\) that contains literal \(\nl\).
\end{proposition}

\begin{corollary}\label{cor:q-dnf}
Consider a DNF \(\Delta = \alpha_1 \vee \ldots \vee \alpha_n\) where each term \(\alpha_i\) is a
prime implicant of \(\Delta\). We can obtain \(\forall \l \cdot \Delta\) by removing from \(\Delta\)
terms that contain literal \(\nl\).
\end{corollary}

To summarize, the following quantifications can be performed in linear time:
universal literal quantification on CNF, universal literal quantification on prime implicants,
existential literal quantification on DNF and existential literal quantification on prime implicates.
\color{\RED}
These results parallel ones that are well-known for variable quantification. To be more precise,
the elimination of universal variable quantifiers from a CNF is a special case of the so-called universal reduction
rule for QBFs~\shortcite{DBLP:journals/iandc/BuningKF95}. Moreover, the elimination of existential variable quantifiers from a 
DNF is a special case of the so-called existential reduction rule for QBFs. Both rules are used in search-based QBF solvers.
Similarly, the use of consensus and resolution to eliminate universal and existential quantifiers in CNFs and
DNFs, respectively, has been known for a while.
\color{black}

\begin{figure*}[tb]
    \centering
    \begin{subfigure}[t]{0.45\textwidth}
        \centering
\scalebox{0.8}{
        \begin{tikzpicture}[scale=.5]    
    \node(root) at (9.5,9){$\vee$};
    \node(n1) at (4,7){$\wedge$};
    \node(n2) at (15,7){$\wedge$};  
    \node(n11) at (2,5){$\overline{d}$};  
    \node(n12) at (6,5){$\vee$};
    \node(n21) at (13,5){$\wedge$};  
    \node(n22) at (17,5){$d$};   
    \node(n121) at (3,3){$\wedge$};  
    \node(n122) at (9,3){$\wedge$};
    \node(n212) at (15,3){$g$};        
    \node(n1211) at (1,1){$h$};     
    \node(n1212) at (5,1){$\top$}; 
    \node(n1221) at (7,1){$\overline{h}$};
    \node(n1222) at (11,1){$i$};    
              
    \draw(root) -- (n1);
    \draw(root) -- (n2);
    \draw(n1) -- (n11);   
    \draw(n1) -- (n12);
    \draw(n2) -- (n21);    
    \draw(n2) -- (n22);
    \draw(n12) -- (n121);   
    \draw(n12) -- (n122);  
    \draw(n21) -- (n1222);       
    \draw(n21) -- (n212); 
    \draw(n121) -- (n1211);   
    \draw(n121) -- (n1212);     
    \draw(n122) -- (n1221);     
    \draw(n122) -- (n1222);               
    \end{tikzpicture}
}
\caption{Decision-DNNF $\Delta$ \label{fig:Decision-DNNF}}
\end{subfigure}
\hfill
    \begin{subfigure}[t]{0.45\textwidth}
        \centering
    \scalebox{0.8}{
         \begin{tikzpicture}[scale=.5]    
    \node(root) at (9.5,9){$\vee$};
    \node(n1) at (4,7){$\wedge$};
    \node(n2) at (15,7){$\wedge$};  
    \node(n11) at (2,5){$\overline{d}$};  
    \node(n12) at (6,5){$\vee$};
    \node(n21) at (13,5){$\wedge$};  
    \node(n22) at (17,5){\textcolor{red}{$\top$}};   
    \node(n121) at (3,3){$\wedge$};  
    \node(n122) at (9,3){$\wedge$};
    \node(n212) at (15,3){$g$};        
    \node(n1211) at (1,1){$h$};     
    \node(n1212) at (5,1){$\top$}; 
    \node(n1221) at (7,1){$\overline{h}$};
    \node(n1222) at (11,1){$i$};    
              
    \draw(root) -- (n1);
    \draw(root) -- (n2);
    \draw(n1) -- (n11);   
    \draw(n1) -- (n12);
    \draw(n2) -- (n21);    
    \draw(n2) -- (n22);
    \draw(n12) -- (n121);   
    \draw(n12) -- (n122);  
    \draw(n21) -- (n1222);       
    \draw(n21) -- (n212); 
    \draw(n121) -- (n1211);   
    \draw(n121) -- (n1212);     
    \draw(n122) -- (n1221);     
    \draw(n122) -- (n1222);               
    \end{tikzpicture}
}
\caption{$\exists d \cdot \Delta$ \label{fig:exists-d}}
\end{subfigure}
\caption{\color{\BLUE}Existentially quantifying literal \(d\) from Decision-DNNF \(\Delta\). 
This amounts to replacing every occurrence of literal \(d\) in the Decision-DNNF by \(\top\).
The result is a DNNF.\color{black}
\label{fig:exists-example}}
\end{figure*}
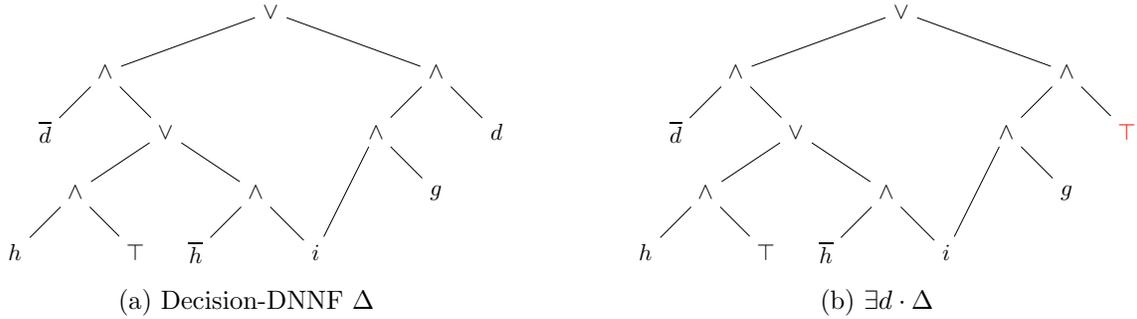

We next show that existential and universal literal quantification can be performed in linear time on Decision-DNNF 
circuits~\cite{jair/HuangD07}. These are NNF circuits which
satisfy two properties: {\em decision} and {\em decomposability.}
The decision property says that every or-node has the form \((\l \wedge \alpha) \vee (\nl \wedge \beta)\), where \(\l\) is a literal.
The decomposability property says that for every and-node, the sets of variables of its children \(\alpha_1, \ldots, \alpha_n\)
are pairwise disjoint. 
\color{\RED} Figure~\ref{fig:Decision-DNNF} depicts an example Decision-DNNF \(\Delta\) for the CNF $(h \vee i) \wedge (\overline{d} \vee g) \wedge (\overline{d} \vee i)$
over variables $D,G,H,I$. 
Decision-DNNF circuits are a strict superset of OBDD circuits.
\color{black}

%have no variables in common, \(variables(\alpha_i)\cap variables(\alpha_j) = \emptyset\) for \(i \neq j\).

We first discuss existential literal quantification which yields DNNF circuits~\cite{jacm/Darwiche01}. These are NNF circuits that 
satisfy only the decomposability property. 

\begin{proposition}[\(\exists\), Decision-DNNF] 
\label{prop:elq-D-DNNF}
Literals can be existentially quantified from a Decision-DNNF circuit in time linear in the circuit size while yielding a DNNF circuit.
\end{proposition}
The quantification algorithm follows directly from Propositions~\ref{prop:quantify-base} and~\ref{prop:quantify-compound}(a,c) since existential
quantification distributes through both the conjunctions and disjunctions of a Decision-DNNF.
\color{\RED} Figure~\ref{fig:exists-d} depicts the DNNF circuit which results from existentially quantifying literal \(d\) from the Decision-DNNF \(\Delta\) of
Figure~\ref{fig:Decision-DNNF}. 
\color{black}

\begin{figure*}[ht]
    \centering
    \begin{subfigure}[t]{0.45\textwidth}
        \centering
\scalebox{0.8}{
         \begin{tikzpicture}[scale=.5]    
    \node(root) at (9.5,9){$\wedge$};
    \node(n1) at (4,7){$\vee$};
    \node(n2) at (15,7){$\vee$};  
    \node(n11) at (2,5){$d$};  
    \node(n12) at (6,5){$\wedge$};
    \node(n21) at (13,5){$\wedge$};  
    \node(n22) at (17,5){$\overline{d}$};   
    \node(n121) at (3,3){$\vee$};  
    \node(n122) at (9,3){$\vee$};
    \node(n212) at (15,3){$g$};        
    \node(n1211) at (1,1){$\overline{h}$};     
    \node(n1212) at (5,1){$\top$}; 
    \node(n1221) at (7,1){$h$};
    \node(n1222) at (11,1){$i$};    
              
    \draw(root) -- (n1);
    \draw(root) -- (n2);
    \draw(n1) -- (n11);   
    \draw(n1) -- (n12);
    \draw(n2) -- (n21);    
    \draw(n2) -- (n22);
    \draw(n12) -- (n121);   
    \draw(n12) -- (n122);  
    \draw(n21) -- (n1222);       
    \draw(n21) -- (n212); 
    \draw(n121) -- (n1211);   
    \draw(n121) -- (n1212);     
    \draw(n122) -- (n1221);     
    \draw(n122) -- (n1222);               
    \end{tikzpicture}
}
\caption{NNF circuit $\Gamma$ \label{fig:transform-Decision-DNNF}}
\end{subfigure}
\hfill
    \begin{subfigure}[t]{0.45\textwidth}
        \centering
    \scalebox{0.8}{
         \begin{tikzpicture}[scale=.5]    
    \node(root) at (9.5,9){$\wedge$};
    \node(n1) at (4,7){$\vee$};
    \node(n2) at (15,7){$\vee$};  
    \node(n11) at (2,5){$d$};  
    \node(n12) at (6,5){$\wedge$};
    \node(n21) at (13,5){$\wedge$};  
    \node(n22) at (17,5){\textcolor{red}{$\bot$}};   
    \node(n121) at (3,3){$\vee$};  
    \node(n122) at (9,3){$\vee$};
    \node(n212) at (15,3){$g$};        
    \node(n1211) at (1,1){$\overline{h}$};     
    \node(n1212) at (5,1){$\top$}; 
    \node(n1221) at (7,1){$h$};
    \node(n1222) at (11,1){$i$};    
              
    \draw(root) -- (n1);
    \draw(root) -- (n2);
    \draw(n1) -- (n11);   
    \draw(n1) -- (n12);
    \draw(n2) -- (n21);    
    \draw(n2) -- (n22);
    \draw(n12) -- (n121);   
    \draw(n12) -- (n122);  
    \draw(n21) -- (n1222);       
    \draw(n21) -- (n212); 
    \draw(n121) -- (n1211);   
    \draw(n121) -- (n1212);     
    \draw(n122) -- (n1221);     
    \draw(n122) -- (n1222);               
    \end{tikzpicture}}
\caption{$\forall d \cdot \Gamma$ \label{fig:forall-d}}
\end{subfigure}
\caption{\color{\BLUE}Universally quantifying literal $d$ from the NNF circuit $\Gamma$, which is obtained from 
the Decision-DNNF \(\Delta\) of Figure~\ref{fig:Decision-DNNF} using Proposition~\ref{prop:ulq-D-DNNF-transform}. 
The quantification process amounts to replacing every occurrence of literal \(\n(d)\) in \(\Gamma\) with \(\bot\). 
Proposition~\ref{prop:ulq-D-DNNF-transform} guarantees that
the two circuits \(\Delta\) and \(\Gamma\) are equivalent.\color{black}
\label{fig:forall-example}}
\end{figure*}
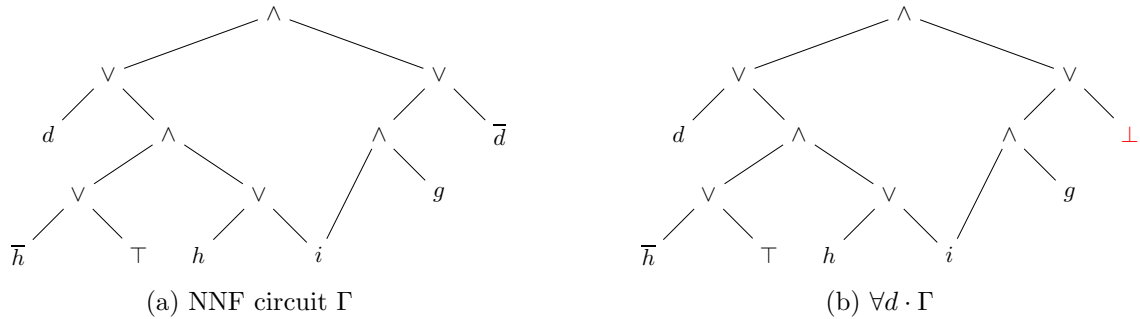

We next show a similar result for universally quantifying literals in linear time, except that the output is only
guaranteed to be an NNF circuit in this case (i.e., not necessarily decomposable). 
This is a two-step procedure, where each step requires linear time processing. The first step is described by the following
proposition. It transforms the Decision-DNNF into a form that allows us to invoke 
Proposition~\ref{prop:quantify-compound}(d) so we can distribute universal quantification through disjunctions
(this quantification always distributes through conjunctions).

\begin{proposition} \label{prop:ulq-D-DNNF-transform}
Let \(\Delta\) be a Decision-DNNF circuit and let \(\Gamma\) be the result of replacing every
fragment \((\l \wedge \alpha) \vee (\nl \wedge \beta)\) in \(\Delta\) with \((\l \vee \beta) \wedge (\nl \vee \alpha)\).
Then 
(1) \(\Gamma\) is an NNF circuit that is equivalent to \(\Delta\);
(2) \(\Gamma\) can be obtained from \(\Delta\) in time linear in the size of \(\Delta\); and
(3) for every disjunction \(\alpha \vee \beta\) in \(\Gamma\), the disjuncts \(\alpha\) and \(\beta\) do not share variables. 
\end{proposition}

\color{\RED}
Figure~\ref{fig:transform-Decision-DNNF} depicts an example of this first step. It shows an NNF circuit \(\Gamma\) which is obtained 
by transforming the Decision-DNNF \(\Delta\) of Figure~\ref{fig:Decision-DNNF}.
The resulting circuit \(\Gamma\) is not a Decision-DNNF, yet no variables are shared between  disjuncts in this circuit.
\color{black}

The second step directly applies Propositions~\ref{prop:quantify-base} and~\ref{prop:quantify-compound}(b,d) to the result of the first
step, now that disjuncts no longer share variables. 

\begin{proposition}[\(\forall\), Decision-DNNF] 
\label{prop:ulq-D-DNNF}
Literals can be universally quantified from a Decision-DNNF circuit in time linear in the circuit size while yielding an NNF circuit.  
\end{proposition}

\color{\RED}
Figure~\ref{fig:forall-d} depicts an example of this second step. 
It shows the result of universally quantifying literal \(d\) from the NNF circuit \(\Gamma\) of Figure~\ref{fig:transform-Decision-DNNF}.
Given Propositions~\ref{prop:quantify-base} and~\ref{prop:quantify-compound}(b,d), we obtain \(\forall d \cdot \Gamma\) by simply replacing every
occurrence of literal \(\n(d)\) in \(\Gamma\) with \(\bot\) as shown in the figure. We finally note that  $\forall d \cdot \Gamma = i \wedge g$ in this case.
\color{black}

\color{\RED}
The two-step, linear-time procedure suggested by Propositions~\ref{prop:ulq-D-DNNF-transform} and~\ref{prop:ulq-D-DNNF} 
was actually used implicitly in~\cite{DarwicheH20} 
\color{black}
for explaining the decisions made by Boolean classifiers
on instances---a process which corresponds to universally quantifying all literals in the instance (see Section~\ref{sec:XAI}). 
The procedure proposed in~\cite{DarwicheH20} transformed fragments \((\l \wedge \alpha) \vee (\nl \wedge \beta)\)
in the Decision-DNNF into \((\l \wedge \alpha) \vee (\nl \wedge \beta) \vee (\alpha \wedge \beta)\),
calling this  a {\em consensus} operation as it resembles the consensus operation on DNF. 
It then transformed this further into \((\l \wedge \alpha) \vee (\alpha \wedge \beta)\) when literal \(\l\) appeared in the instance,
calling this a {\em filtering} operation.
\color{\RED}
This is equivalent to \((\l \vee \beta) \wedge \alpha\), which is the result of universally quantifying literal \(\l\) from
fragment \((\l \vee \beta) \wedge (\nl \vee \alpha) = (\l \wedge \alpha) \vee (\nl \wedge \beta)\). 
\color{black}
These transformations correspond to the two-step procedure described 
above except that~\cite{DarwicheH20} did not realize that their algorithm was actually universally quantifying
literals as such quantification was not introduced yet. But~\cite{DarwicheH20} did observe that their procedure 
yields a monotone circuit, a property which is guaranteed to hold when universally
quantifying a literal for each variable.\footnote{The circuit is monotone since for every variable \(X\), the literals \(x\) and \(\n(x)\) cannot both appear in the circuit.}

%Pierre
\medskip
We next show that similar linear time algorithms can be used on SDD circuits~\cite{ijcai/Darwiche11}.
These are NNF circuits which are composed of fragments having the form \((p_1 \wedge s_1) \vee \ldots \vee (p_n \wedge s_n)\),
where \(p_i\) are called {\em primes} and \(s_i\) are called {\em subs.} SDDs satisfy the decomposability
property. Moreover, the primes \(p_1, \ldots, p_n\) form a partition: \(p_i \neq \bot\), \(p_i \wedge p_j = \bot\) for \(i \neq j\) and
\(p_1 \vee \ldots \vee p_n = \top\). SDDs actually satisfy a stronger version of decomposability, 
called {\em structured decomposability}~\cite{aaai/PipatsrisawatD08}, but
we do not need this stronger property for the following results. 
\color{\RED}
Like Decision-DNNFs, SDDs are a strict superset of OBDDs. We note, however, that Decision-DNNF and SDD circuits are
not comparable in terms of succinctness; that is, neither is strictly more succinct than the other~\cite{mst/BolligB19,uai/BeameL15}.
\color{black}

\begin{proposition}[\(\exists\), SDD] \label{prop:elq-SDD}
One can existentially quantify a set of literals from an SDD circuit in time linear in the circuit size,
with the result being a DNNF circuit.
\end{proposition}
Again, the algorithm follows directly from Propositions~\ref{prop:quantify-base} and~\ref{prop:quantify-compound}(a,c) since existential
quantification can be distributed through both conjunctions and disjunctions in this case.

Universally quantifying literals from an SDD circuit is also based on a two-step procedure, 
where each step requires linear time processing. Similar to Decision-DNNF circuits, 
the first step transforms the SDD into an equivalent NNF circuit in which disjuncts do not share variables, therefore making it 
directly amenable to Propositions~\ref{prop:quantify-base} and~\ref{prop:quantify-compound}(b,d).

\begin{proposition} \label{prop:sdd-transform}
Let \(\Delta\) be an SDD circuit and \(\Gamma\) be the result of replacing every fragment
 \((p_1 \wedge s_1) \vee \ldots \vee (p_n \wedge s_n)\) in \(\Delta\) with \((\neg p_1 \vee s_1) \wedge \ldots \wedge (\neg p_n \vee s_n)\).
Then
(1) \(\Gamma\) is an NNF circuit that is equivalent to \(\Delta\); 
(2) \(\Gamma\) can be obtained from \(\Delta\) in time linear in the size of \(\Delta\); and
(3) disjuncts in \(\Gamma\) do not share variables. 
\end{proposition}

\begin{proposition}[\(\forall\), SDD] \label{prop:ulq-SDD}
One can universally quantify a set of literals from an SDD circuit in time linear in the circuit size, with the result being an NNF circuit.
\end{proposition}

In summary, the existential and universal quantification of multiple literals can be performed in linear time 
on Decision-DNNF and SDD circuits. \color{\RED} Existential literal quantification yields DNNF circuits, while
universal literal quantification yields NNF circuits that may not be decomposable. Hence, the above procedures 
on Decision-DNNF and SDD circuits cannot be used to interleave literal quantifications of different types. This should not be
surprising though since the existence of a polynomial-time algorithm for applying a sequence
of quantifiers of different types to a tractable circuit (even an OBDD) would imply {\sf P = NP};
see~\shortcite{DBLP:journals/jsat/Coste-MarquisBLM06}.
Nonetheless, it is worth noting that for both types of literal quantification, the NNF circuit that results from applying the 
above procedures is monotone, if we quantify a literal for each variable that appears in the input circuit.\footnote{This follows because Proposition~\ref{prop:quantify-base} tells us that \(\exists \l \cdot \l = \top\) and \(\exists \l \cdot \nl = \nl\);
similarly \(\forall \l \cdot \l = \l\) and \(\forall \l \cdot \nl = \bot\). Hence, the quantified NNF circuit will not
contain complementary literals for any variable.}
Hence, one can test efficiently whether the resulting circuit is satisfiable, or whether it is valid, which does not hold for either 
CNF or DNF unless {\sf P = NP} (deciding the satisfiability of a CNF is {\sf NP}-complete and
deciding the validity of a DNF is {\sf coNP}-complete).
\color{black}

\section{Universal Quantification for Explainable AI}
\label{sec:XAI}

We now consider a number of questions that arise in explainable AI and show how they can be answered using universal 
literal and variable quantification. These questions include 
(1)~finding the {\em culprit} behind a decision (a minimal set of characteristics that can trigger the decision); 
(2)~assessing whether a decision is {\em biased} (depends on protected features); and 
(3)~identifying instances with some {\em irrelevant} features or characteristics (do not play a role in the decisions made on these instances). 
Some of these questions have been treated in the literature~\shortcite<e.g.,>{ShihCD18,IgnatievNM19a,IgnatievNM19b,ignatiev2019validating,DarwicheH20,kr/AudemardKM20}, 
but we provide new formulations based on quantification which allow a more refined and general treatment. 
Moreover, given the results in Section~\ref{sec:tractable}, our treatment will shed light on the syntactic forms of classifiers that 
facilitate the computation of explainable AI queries.

\subsection{Classifiers and Decisions}
Our focus is on Boolean classifiers, which correspond to Boolean functions \(f(X_1, \ldots, X_n)\) that map literals \(\l_1, \ldots, \l_n\) 
of variables \(X_1, \ldots, X_n\) into \(\{0,1\}\). A set of literals \(\l_1, \ldots, \l_n\) will be called an {\em instance},
which is {\em positive} when \(f(\l_1, \ldots, \l_n) = 1\) and {\em negative} when \(f(\l_1, \ldots, \l_n) = 0\). 
A term \(\gamma\) over some variables in \(X_1, \ldots, X_n\) will be called a {\em population} as it characterizes a 
set of instances (those compatible with term \(\gamma\)). An instance corresponds to a singleton population, so 
claims about populations apply to instances but the converse is not true. For example, a classifier will always make a 
decision on an instance but it may not be able to make a collective decision on a population as the population may 
contain both positive and negative instances.

We will represent a Boolean classifier by a Boolean formula \(\Delta\), where the
models of \(\Delta\) correspond to positive instances and the models of \(\neg \Delta\) correspond to negative instances.
Hence, the syntactic forms of both \(\Delta\) and its negation \(\neg \Delta\) are relevant when computing explainable
AI queries.

\begin{definition}\label{def:decision}
Let \(\Delta\) be a Boolean formula over variables \(X_1, \ldots, X_n\).
We call \(\Delta\) a \hl{classifier}, variable \(X_i\) a \hl{feature}, literal \(\l_i\) 
a \hl{characteristic}, and \(\delta = \{\l_1, \ldots, \l_n\}\) an \hl{instance.} 
The \hl{decision} of classifier \(\Delta\) on instance \(\delta\) is denoted \(\Delta(\delta)\).
It is defined as \(\Delta(\delta) = 1\) if \(\delta \models \Delta\) (positive decision) 
and \(\Delta(\delta) = 0\) if \(\delta \models  \neg\Delta\) (negative decision).
We further define \(\Delta_\delta = \Delta\) when the instance \(\delta\) is positive and
\(\Delta_\delta = \neg \Delta\) when the instance \(\delta\) is negative.
\end{definition}
Since \(\Delta\) captures positive instances and its negation \(\neg \Delta\) captures
negative instances, we usually work with \(\Delta\) when reasoning about 
positive decisions and with \(\neg \Delta\) when reasoning about negative decisions.
This explains the significance of the notation \(\Delta_\delta\).

%Pierre
\medskip
Consider the following classifier which decides whether an applicant should be granted a loan based on four features:
whether they defaulted on a previous loan (\(D\)), have a guarantor (\(G\)), own a home (\(H\)) or
have a high income (\(I\)). This classifier is specified by \textcolor{\BLUE}{the following formula $\Delta$ and its negation $\neg \Delta$, both in CNF:}
\begin{equation}
\Delta  = (h \vee i) \wedge (\n(d) \vee g) \wedge (\n(d) \vee i) \mbox{ and }
\neg \Delta = (d \vee \n(i)) \wedge (d \vee \n(h)) \wedge (\n(g) \vee \n(i)).
\label{c:loan}
\end{equation}
Classifier \(\Delta\) will grant a loan to an applicant who never defaulted on a previous loan, owns a home,
has a high income but does not have a guarantor, \(\Delta(\n(d),\n(g),h,i)=1\). But it will not
grant a loan to such an applicant if they defaulted on a previous loan, \(\Delta(d,\n(g),h,i)=0\).
Instance \(\delta_1 = \n(d),\n(g),h,i\) is positive and instance
\(\delta_2 = d,\n(g),h,i\) is negative. Using our notation, we have 
\(\Delta_{\delta_1} = \Delta\) since \(\delta_1 \models \Delta\) and
 \(\Delta_{\delta_2} = \neg \Delta\) since \(\delta_2 \models \neg \Delta\).

The notion of a decision in Definition~\ref{def:decision} can be extended to populations.

\begin{definition}
A population \(\gamma\) is \hl{decided} by a classifier \(\Delta\) iff \(\gamma\models\Delta\) or \(\gamma\models\neg\Delta\).
In the first case, the decision is positive and we write \(\Delta(\gamma)=1\). In the second case, the decision is negative
and we write \(\Delta(\gamma)=0\). Otherwise, the decision \(\Delta(\gamma)\) is undefined.
\end{definition}

\noindent Contrary to instances, it is possible that no decision is made on a population \(\gamma\) as we may have neither
\(\gamma \models \Delta\) nor \(\gamma \models  \neg\Delta\).
Consider the population of applicants who have guarantors and a high income, \(\gamma_1 = g,i\).
All members of this population will be granted loans since \(\gamma_1 \models \Delta\), and
hence \(\Delta(\gamma_1)=1\), making this a positive population. 
The population of applicants who defaulted on a previous loan and do not
have a high income, \(\gamma_2 = d,\n(i)\), is negative. No member of this population will be granted a loan
since \(\gamma_2 \models \neg \Delta\) and hence \(\Delta(\gamma_2)=0\). Consider now
the population of applicants who defaulted on a loan and are not home owners, \(\gamma_3 = d,\n(h)\).
This population cannot be decided as it contains some members who will be granted loans (e.g., \(d,g,\n(h),i\))
and others who will not (e.g., \(d,g,\n(h),\n(i)\)). 
We then have \(\gamma_3 \not \models \Delta\) and \(\gamma_3 \not \models \neg \Delta\), 
causing the decision \(\Delta(\gamma_3)\) to be undefined. 
This will never happen for an instance \(\delta\) as we must have either \(\delta \models \Delta\) or \(\delta \models  \neg\Delta\) 
since the instance \(\delta\) must contain a characteristic for each feature. 

We will find it useful to talk about containment when analyzing the decisions made on populations. 
We will say that population \(\gamma\) {\em contains} population \(\beta\) iff \(\beta \models \gamma\). 
For example, the population of home owners (\(\gamma = h\)) contains the 
population of home owners with a high income (\(\beta = h,i\)).
We say in this case that \(\gamma\) is a {\em super-population} of \(\beta\) and \(\beta\) is a {\em sub-population} 
of \(\gamma\). For example, when explaining the decision on an instance \(\delta\), one is typically interested in
finding all maximal super-populations of instance \(\delta\) that are decided similarly as \(\delta\). 
As we show later, this is precisely the approach proposed in~\cite{ShihCD18}, which we will 
generalize to explain decisions on populations as well. 

\subsection{Decision Making and Universal Quantification}

Before we start discussing explainable AI queries in Section~\ref{sec:ifeature}, we will first provide some insights
on the fundamental role that universal literal quantification plays when reasoning about decisions. 
We will interpret literal quantification as {\em characteristic quantification} and show how such quantification
{\em selects} instances in ways that can be useful for answering various queries of interest to explainable AI.

When universally quantifying characteristic \(\l\) from classifier \(\Delta\),
we are {\em filtering out} all positive instances for which characteristic \(\nl\) is essential for the decisions on these instances.
We call these positive {\em \(\nl\)-boundary instances} as they are positive instances with characteristic \(\nl\) but will become negative
if this characteristic is flipped (Theorem~\ref{theo:ulq-semantics-b}). 
Universally quantifying characteristic \(\l\) will not filter out any instance with characteristic \(\l\) (Proposition~\ref{prop:ulq-preserve}).
We can therefore view the application of operator \(\forall \l\) to \(\Delta\) as a process of {\em selecting} all
positive instances that do not require characteristic \(\nl\) for their positiveness. If any of these instances has characteristic
\(\nl\) then that characteristic is irrelevant to the decision made on the instance. These are precisely the instances
characterized by \(\forall \l \cdot \Delta\).
A complementary situation arises when universally quantifying characteristic \(\l\) from \(\neg \Delta\).
The instances characterized by \(\forall \l \cdot \neg \Delta\) are precisely the negative instances characterized
by \(\neg \Delta\) which do not require characteristic \(\nl\) for their negativeness.
%Each such instance has characteristic \(\l\) or the irrelevant characteristic \(\nl\).

\color{\BLUE}
Consider again the classifier defined by Equation \ref{c:loan}. The positive instance $\delta_1 = \n(d), \n(g), h, i$ 
is not a model of $\forall d \cdot \Delta$ (which is equivalent to $i \wedge g$ as discussed in Section \ref{sec:tractable}). Thus, characteristic
$\n(d)$ of $\delta_1$ is essential for the positiveness of $\delta_1$ (instance $d, \n(g), h, i$ is negative).
Contrastingly, the positive instance $\delta_2 = \n(d), g, \n(h), i$ is a model of $\forall d \cdot \Delta$ so
characteristic $\n(d)$ is not essential for the positiveness of $\delta_2$ (instance $d, g, \n(h), i$ is also positive).
\color{black}

We now turn to interpreting b-rules as descriptors of {\em boundary instances} that get filtered out by universal quantification.
Each b-rule \(\Rule(\alpha,\l)\) for classifier \(\Delta\) identifies a positive instance \(\delta = \alpha,\l\) 
which becomes negative if we flip its characteristic \(\l\) (Definition~\ref{def:brule}). 
That is, b-rules \(\Rule(\alpha,\l)\) for \(\Delta\) identify positive \(\l\)-boundary instances. 
Similarly, b-rules \(\Rule(\alpha,\l)\) for \(\neg \Delta\) identify negative \(\l\)-boundary instances.
Hence, b-rules characterize instances that are selected (or filtered out) when universally quantifying a characteristic.

Suppose now that we are universally quantifying a set of characteristics \(\l_1, \ldots, \l_n\)
from \(\Delta\) to yield \(\forall \l_1, \ldots, \l_n \cdot \Delta\). In this case, \(\forall \l_1, \ldots, \l_n \cdot \Delta\)
captures all positive instances~\(\delta\) where the characteristics they have in \(\nl_1, \ldots, \nl_n\) are
irrelevant to how these instances are decided. 
That is, we can flip any of the characteristics in \(\delta \cap \{\nl_1, \ldots, \nl_n\}\) without changing 
the decision (Theorem~\ref{theo:ulq-semantics}).
As we shall see later, when characteristics \(\l_1, \ldots, \l_n\) define a positive instance, 
\(\forall \l_1, \ldots, \l_n \cdot \Delta\) will characterize all
super-populations of this instance that are decided similarly to the instance. Moreover,
these super-populations can be viewed as explanations of the decision made on
instance \(\l_1, \ldots, \l_n\). We will also see how selecting instances
based on quantifying both characteristics and features can be used to answer further
queries such as those relating to decision and classifier bias.

\subsection{Irrelevant Features}
\label{sec:ifeature}

The first application we shall consider for universal quantification concerns the selection of instances which can be
decided independently of a given set of features. 
For example, we may wish to identify all individuals who will be granted a loan regardless of their income and home ownership. 
These features may be relevant to some applicants but not others. Our interest is therefore in identifying all instances
that can be decided independently of some given features while capturing these instances using a Boolean formula. 
We next show how universal quantification can be used to achieve this. 

\begin{definition}\label{def:erase}
To \hl{erase} variables \(X_1, \ldots, X_n\) from term \(\gamma\)
is to remove the literals of variables \(X_1, \ldots, X_n\) from term \(\gamma\).
The resulting term is denoted $\erase{\gamma}{X_1, \ldots, X_n}$.\footnote{We can also define this operator
using variable quantification since $\erase{\gamma}{X_1, \ldots, X_n} = \exists X_1, \ldots, X_n \cdot \gamma.$
We use Definition~\ref{def:erase} instead as it is more direct.}
\end{definition}
\noindent For example, if \(\gamma = x \n(y) \n(z) w\) then \(\erase{\gamma}{Z,W} = x \n(y)\).
When term \(\gamma\) represents an instance or a population, the erase operator creates a super-population of \(\gamma\). 

\begin{definition}\label{def:ifeature}
Let \(\gamma\) be a population decided by classifier \(\Delta\). We say that decision \(\Delta(\gamma)\) is \hl{independent}
of features \(X_1, \ldots, X_n\) precisely when \(\Delta(\gamma) = \Delta(\erase{\gamma}{X_1,\ldots,X_n})\). 
\end{definition}
\noindent We also say in this case that features \(X_1, \ldots, X_n\) are {\em irrelevant} to the decision \(\Delta(\gamma)\).

Consider again the loan classifier in~(\ref{c:loan}) and the population of home owners who have a high income but
never defaulted on previous loan, \(\gamma = \n(d),h,i\). This is a positive population, \(\Delta(\gamma)=1\), as all
its members will be granted loans.  The decision on this population is independent of feature \(I\) though since
\(\Delta(\erase{\gamma}{I})=1\) where \(\erase{\gamma}{I} = \n(d) h\).

If our goal is to check whether a decision \(\Delta(\gamma)\) is independent of features \(X_1, \ldots, X_n\)
then we can simply check whether \(\Delta(\gamma) = \Delta(\erase{\gamma}{X_1,\ldots,X_n})\). 
But universal quantification can be used to characterize all instances that are decided independently of some features.

\begin{theorem}\label{theo:ifeature}
Let \(\Delta\) be a classifier and \(\delta\) be an instance. 
Decision \(\Delta(\delta)\) is independent of features \(X_1, \ldots, X_n\) iff \(\delta \models \forall X_1, \ldots, X_n \cdot~\Delta_\delta\).
\end{theorem}

\noindent According to this result, \(\forall X_1, \ldots, X_n \cdot~\Delta\) characterizes all instances that will be decided positively 
independently of features \(X_1, \ldots, X_n\) and \(\forall X_1, \ldots, X_n \cdot \neg \Delta\) characterizes all instances that 
will be decided negatively independently of these features. For example, the following Boolean formula characterizes all
applicants who will be granted a loan independently of whether they  own a home or have defaulted on a previous loan:

\[ \forall D,H \cdot \Delta = g \wedge i,  \mbox{ where } \Delta = (h \vee i) \wedge (\n(d) \vee g) \wedge (\n(d) \vee i). \]

This formula captures applicants who have a guarantor and a high income. Each member of this population
will be granted a loan regardless of their features \(D\) and \(H\). The expression \(\forall D,H \cdot \Delta\) 
can be easily evaluated since \(\Delta\) is given as a CNF: 
we just remove literals \(d\), \(\n(d)\), \(h\) and \(\n(h)\) from every clause of the CNF (see Proposition~\ref{prop:ulq-uvq} and~\ref{prop:q-cnf}). 

It is possible that \(\forall X_1, \ldots, X_n \cdot~\Delta = \bot\),
which means that every positive decision must depend on features \(X_1, \ldots, X_n\). 
Similarly, it is possible that \(\forall X_1, \ldots, X_n \cdot \neg \Delta = \bot\), which means that every negative decision must
depend on these features. Moreover, it is possible that one decision type is independent of some features but the other is not. 
This is illustrated by the following example:
\[
\begin{array}{ll}
\forall D,G \cdot \Delta = \bot, & \mbox{where \(\Delta = (h \vee i) \wedge (\n(d) \vee g) \wedge (\n(d) \vee i)\)} \\
\forall D,G \cdot \neg \Delta = \n(h)\wedge\n(i), & \mbox{where \(\neg \Delta = (d \vee \n(i)) \wedge (d \vee \n(h)) \wedge (\n(g) \vee \n(i)).\)}
\end{array}
\]
No applicant will be granted a loan without considering whether they defaulted on a previous loan and whether they
have a guarantor (features \(D\) and \(G\)). However, some applicants will be denied a loan without considering
these features. In particular, any applicant who does not own a home and does not have a high income will be declined. 

\shrink{
Consider for example classifier \(\Delta = x \wedge y\), which makes one positive decision on instance \(\{x,y\}\)
and three negative decisions on instances \(\{\n(x),y\}\), \(\{x,\n(y)\}\) and \(\{\n(x),\n(y)\}\). 
We have \(\forall X \cdot \Delta = \bot\) but \(\forall X \cdot \neg \Delta = \n(y)\).
The single positive decision of this classifier depends on feature \(X\), yet the negative decisions it makes on 
instances \(\{x,\n(y)\}\) and \(\{\n(x),\n(y)\}\) are independent of this feature. 
}

In contrast, we may have \(\forall X_1, \ldots, X_n \cdot \Delta = \Delta\) indicating that every positive decision
is independent of features \(X_1, \ldots, X_n\). This is equivalent to \(\forall X_1, \ldots, X_n \cdot \neg \Delta = \neg \Delta\),
which means that all negative decisions will also be independent of these features.\footnote{To show this, observe that
\(\forall X_1, \ldots, X_n \cdot \Delta = \Delta\) is equivalent to $\Delta$ being independent of \(X_1, \ldots, X_n\). 
This is equivalent to $\neg \Delta$ being independent of \(X_1, \ldots, X_n\), which is equivalent to
\(\forall X_1, \ldots, X_n \cdot \neg \Delta = \neg \Delta\).} The loan classifier depends on all its features.

\shrink{
It is possible that \(\forall X_1, \ldots, X_n \cdot \Delta = \bot\). In this case, no instance can be 
decided positively without considering features \(X_1, \ldots, X_n\). Similarly if \(\forall X_1, \ldots, X_n \cdot \neg \Delta = \bot\).
Moreover, if \(\forall X_1, \ldots, X_n \cdot \Delta = \Delta\) then these features are irrelevant for every decision. 
}

\subsection{Irrelevant Characteristics}
\label{sec:icharacteristic}

We may also be interested in instances whose classification does not depend on some characteristics (in contrast to features). 
For example, we may be interested in applicants who will be granted a loan but not due to their high income or home
ownership. These applicants may not have any of these characteristics but if they do then these characteristics
are irrelevant to how their application is decided.
We next show how universal literal quantification can be used to select instances with irrelevant characteristics
and then further contrast irrelevant characteristics with irrelevant features. 

\begin{definition}\label{def:icharacteristic}
Let \(\gamma\) be a population decided by classifier \(\Delta\) and \(\alpha\) be a set of characteristics.
We say that decision \(\Delta(\gamma)\) is \hl{independent}
of characteristics~\(\alpha\) precisely when \(\Delta(\gamma) = \Delta(\gamma\setminus\alpha)\). 
\end{definition}
\noindent We also say in this case that characteristics \(\alpha\) are {\em irrelevant} to decision \(\Delta(\gamma)\).
This definition does not require every characteristic of \(\alpha\) to appear in population \(\gamma\).
Moreover, a characteristic \(\l\) and its negation \(\nl\) may both appear in \(\alpha\).

Consider again the loan classifier and an applicant who defaulted on a previous loan, owns a home,
has a guarantor but does not have a high income, \(\delta=d,h,g,\n(i)\). 
This applicant will be denied a loan, \(\Delta(\delta)=0\), but the decision is independent of 
characteristics \(\n(d)\) and \(h\) since \(\Delta(\delta\setminus\{\n(d),h\})=0\). The decision is independent
of characteristic \(\n(d)\) since the applicant did default on a previous loan. 
It is independent of characteristic \(h\) since the applicant will still be denied a loan if they did not
own a home, \(\Delta(d,\n(h),g,\n(i))=0\).
Note, however, that this decision is not independent of features \(D\) and \(H\).
For example, the applicant will be granted a loan if they did not default on a previous loan,
\(\Delta(\n(d),h,g,\n(i))=1\).

\begin{theorem}\label{theo:icharacteristic}
Let \(\Delta\) be a classifier, \(\l_1, \ldots, \l_n\) be characteristics and \(\delta\) be an instance.
Then \\ 
\(\delta \models \forall \l_1, \ldots, \l_n \cdot \Delta_\delta\) iff
decision \(\Delta(\delta)\) is independent of characteristics \(\nl_1, \ldots, \nl_n\).
\end{theorem}

\noindent According to this result, \(\forall \l_1, \ldots, \l_n \cdot \Delta\) characterizes all instances that are decided positively but
not due to any characteristic they may have in \(\nl_1, \ldots, \nl_n\). 
An instance captured by \(\forall \l_1, \ldots, \l_n \cdot \Delta\) may not have any characteristic in 
\(\nl_1, \ldots, \nl_n\). But if it does, then those characteristics are irrelevant: 
we can change them in any manner without changing the decision.
Similarly,  \(\forall \l_1, \ldots, \l_n \cdot \neg \Delta\) 
captures all instances that are decided negatively but not due to any characteristic they may have in \(\nl_1, \ldots, \nl_n\).

For the loan classifier, there are four applicants who will be denied a loan independently of
characteristics \(d\) and \(h\). These applicants are characterized by the following CNF:
\[
\forall \n(d), \n(h) \cdot \neg \Delta = \n(h)\wedge\n(i), 
\mbox{ where \(\neg \Delta = (d\vee \n(i))\wedge(d\vee\n(h))\wedge(\n(g)\vee\n(i))\).}
\]
None of these applicants owns a home. Moreover, if any of them defaulted on a previous loan they will still
be denied if they did not default.
But there are no applicants who will be denied a loan independently of features \(D\) and \(H\):
\[
\forall D, H \cdot \neg \Delta = \bot.
\]
That is, these features are relevant for every negative decision.

Again, expressions \(\forall \n(d), \n(h) \cdot \neg \Delta\) and \(\forall D, H \cdot \neg \Delta\) 
can be easily evaluated using Propositions~\ref{prop:ulq-uvq} and~\ref{prop:q-cnf} since \(\neg \Delta\) is given as a CNF.
For the first expression, we just remove literals \(d\) and \(h\) from all clauses.
For the second expression, we remove literals \(d\), \(\n(d)\), \(h\) and \(\n(h)\).

It is possible that \(\forall \l_1, \ldots, \l_n \cdot \Delta = \bot\). This indicates that every positive instance with some characteristics
in \(\nl_1, \ldots, \nl_n\) can become negative if we flip some of these characteristics so these characteristics
are relevant for all positive instances. Similarly, if \(\forall \l_1, \ldots, \l_n \cdot \neg \Delta = \bot\),
then any negative instance with some characteristics in \(\nl_1, \ldots, \nl_n\) may become positive 
if we flip some of these characteristics. 
A set of characteristics may be relevant for one type of decision but irrelevant for the other type
since \(\forall \l_1, \ldots, \l_n \cdot \Delta\) may be inconsistent but \(\forall \l_1, \ldots, \l_n \cdot \neg \Delta\) 
may be consistent (and vice versa).

The following result relates irrelevant characteristics and features.
It shows that if some features are irrelevant to a decision then any corresponding characteristics
are also irrelevant to the decision. The converse is not true though as we have seen earlier.

\begin{proposition}\label{prop:ifc}
Let \(\Delta\) be a classifier, \(X_1, \ldots, X_n\) be features and \(\l_1, \ldots, \l_n\) be corresponding characteristics.
Then \(\forall X_1, \ldots, X_n \cdot \Delta \models \forall \l_1, \ldots, \l_n \cdot \Delta\).
\end{proposition}

\subsection{Explaining Decisions}
\label{sec:explain}

We will now consider the application of universal literal and variable quantification to explaining the decisions of classifiers
on both instances and populations. 

Consider a decision \(\Delta(\delta)\) made by classifier \(\Delta\) on instance \(\delta\). One way to explain
this decision is to identify a minimal set of characteristics \(\gamma \subseteq \delta\) that is sufficient to trigger 
the decision, \(\Delta(\gamma)=\Delta(\delta)\). 
This notion of explanation was introduced in~\cite{ShihCD18} under the name of a {\em PI-explanation}
and was later called a {\em sufficient reason} for the decision in~\cite{DarwicheH20}.  We will next
provide some examples of this notion using a classifier from~\cite{DarwicheH20} for admitting 
students into an academic program.

This classifier makes its decision based on five features: whether an applicant passed the entrance exam (\(E\)),
is a first time applicant (\(F\)), has good grades (\(G\)), has work experience (\(W\)) and 
comes from a rich hometown (\(R\)). The classifier is specified by the following CNFs:
\begin{eqnarray*}
\Delta & = & 
(e \vee g)\wedge(e \vee r)\wedge(e \vee w)\wedge(f \vee r)\wedge(\n(f) \vee g \vee w)  \\
\neg \Delta & = & 
(\n(e)\vee f \vee \n(r))\wedge(\n(e)\vee\n(f)\vee\n(g))\wedge(\n(e)\vee\n(f)\vee\n(w))\wedge(\n(g)\vee\n(r)\vee\n(w)) 
\end{eqnarray*}
Consider now an applicant who does not come from a rich hometown but satisfies all other
requirements. This applicant will be admitted, \(\Delta(e,f,g,w,\n(r))=1\), 
and there are two sufficient reasons for this decision, (\(e,f,g\)) and (\(e,f,w\)). 
Passing the entrance exam~(\(e\)), being a first time applicant (\(f\)) and having good grades (\(g\))
will guarantee admission but no strict subset of these characteristics will.
Having work experience instead of good grades will also guarantee admission but again
no strict subset of characteristics (\(e,f,w\)) provides such a guarantee.
If this applicant were to come from a rich hometown, then there would be more sufficient
reasons for the admission decision,  \(\Delta(e,f,g,w,r)=1\):
\[
(e,f,g)\quad
(e,f,w)\quad
(e,g,r)\quad
(e,r,w)\quad
(g,r,w)
\]
For example, passing the entrance exam and having good grades will guarantee admission for
an applicant who comes from a rich hometown.

A decision may have an exponential number of sufficient reasons. Moreover,
some explainable AI queries, such as ones relating to decision bias, are based on checking whether 
the sufficient reasons for a decision satisfy some properties.
These sufficient reasons can sometimes be represented compactly using the notion of a {\em complete reason} 
introduced in~\cite{DarwicheH20}, which showed a number of results relating to this notion.
These results include
(1)~The complete reason for a decision can be computed efficiently when the classifier is represented using 
a tractable circuit of appropriate type;
(2)~The sufficient reasons for a decision correspond to the prime implicants of its complete reason; and
(3)~Some properties of sufficient reasons can be checked in time linear in the size of a complete reason, 
again, when it is represented using a tractable circuit of appropriate type.

We will next show how the notion of a complete reason can be formulated using universal quantification,
particularly the selection semantics of such quantification. This formulation has a number of
implications, which include generalizing this notion to decisions on populations and opening new pathways
for the efficient computation of complete reasons. 

The main insight behind our formulation is to try to find a necessary and sufficient condition for 
why a decision was made. 
Consider a decision on population \(\gamma = \l_1, \ldots, \l_m\) and let \(X_{m+1},\ldots,X_n\) be
all features not mentioned in population \(\gamma\). We will do this by finding all instances that
are decided similarly to population \(\gamma\) but based only on the information used to decide \(\gamma\);
that is, characteristics \(\l_1, \ldots, \l_m\). We will first select instances that are decided
similarly to \(\gamma\) but independently of features \(X_{m+1},\ldots,X_n\) as these features did not
play a role in the decision on population \(\gamma\). These instances
are characterized by the formula \(\forall X_{m+1},\ldots,X_n \cdot \Delta_\gamma\)
as shown in Section~\ref{sec:ifeature}. From these instances,  we will now select
those that are decided similarly to \(\gamma\) but independently of characteristics \(\nl_1, \ldots, \nl_m\)
as these characteristics did not play a role in the decision either.
As shown in Section~\ref{sec:icharacteristic},  these instances are characterized by the formula 
\(\forall \l_1, \ldots, \l_m (\forall X_{m+1},\ldots,X_n \cdot \Delta_\gamma)\).
This formula can be thought of as a necessary and sufficient reason for the decision on
population \(\gamma\) so we shall call it the complete reason for the decision (more on this later).

\begin{definition}[Complete Reason]
 \label{def:creason}
Let \(\gamma = \l_1, \ldots, \l_m\) be a population decided by classifier \(\Delta\) and \(X_{m+1},  \ldots, X_n\)
be all classifier features not mentioned in \(\gamma\). The \hl{complete reason} for decision \(\Delta(\gamma)\)
is defined as the formula \(\forall \l_1, \ldots, \l_m, X_{m+1}, \ldots, X_n \cdot \Delta_\gamma\). 
\end{definition}
\begin{definition}[Sufficient Reason]
\label{def:sreason}
The \hl{sufficient reasons} for a decision are defined as the prime implicants of its complete reason.
\end{definition}

We now have the following result, which establishes our definition of complete reason as a
generalization of the one given in~\cite{DarwicheH20} and our definition of a sufficient reason
as a generalization of the PI-explanation introduced in~\cite{ShihCD18}.

\begin{theorem}\label{theo:reasons}
Let \(\gamma\) be a population decided by classifier \(\Delta\).
Then \(\gamma^\star\) is a sufficient reason for decision \(\Delta(\gamma)\) iff
\(\gamma^\star\) is a minimal subset of \(\gamma\) that satisfies \(\Delta(\gamma^\star)=\Delta(\gamma)\).
\end{theorem}

\noindent According to this result, the sufficient reasons for decision \(\Delta(\gamma)\) are the maximal 
super-populations of \(\gamma\) that are decided similarly to population \(\gamma\). 
Moreover, these super-populations are precisely the prime implicants of the complete reason for the decision.
As mentioned earlier, we may have an exponential number of such super-populations but they are now
encoded by the complete reason which may not be exponentially sized (depending on its syntactic form). 
When the population is a singleton (instance), 
the complete reason for decision \(\Delta(\gamma)\) reduces to \(\forall \l_1, \ldots, \l_m \cdot \Delta_\gamma\)
which provides an alternative definition to the one given in~\cite{DarwicheH20}.
This is an expression that we studied in Section~\ref{sec:icharacteristic}, except that we now have the 
condition \(\l_1, \ldots, \l_m \models \Delta_\gamma\) which we did not assume in that section.
This additional condition provides further selection semantics for universal literal quantification:
the complete reason
\(\forall \l_1, \ldots, \l_m \cdot \Delta_\gamma\) characterizes all instances \(\delta\) that are decided similarly to
\(\gamma\) but due only to the characteristics they have in common with \(\gamma\);
that is, independently of their characteristics \(\delta \cap \{\nl_1, \ldots, \nl_m\}\).

Consider the admission classifier and an applicant who passed the entrance exam, has good grades 
and work experience, comes from a rich hometown but is not a first time applicant. The classifier will
admit this applicant, \(\Delta(e,\n(f),g,r,w)=1\), due to the following complete reason expressed as a CNF:
\[
\forall e,\n(f),g,r,w \cdot \Delta = (e\vee g)\wedge(e \vee w)\wedge(r)\wedge(\n(f)\vee g \vee w)
\]
This formula has four prime implicants representing the sufficient reasons for this decision:
\[
(e,g,r) \quad (e,r,w) \quad (e,\n(f),r) \quad (g,r,w)
\]
Consider now the population of applicants who are applying again but did not pass the entrance exam
and do not come from a rich hometown. Members of this population will be denied admission,
\(\Delta(\n(e),\n(f),\n(r))=0\), for the following complete reason
\[
\forall \n(e),\n(f),\n(r),G,W \cdot \neg \Delta = (\n(e)\vee\n(f))\wedge(\n(r))
\]
which has two sufficient reasons (\(\n(e),\n(r)\)) and (\(\n(f),\n(r)\)).

When a classifier \(\Delta\) and its negation \(\neg \Delta\) are represented by CNFs, 
the complete reason for a decision \(\Delta(\gamma)\) can be
computed in linear time using Propositions~\ref{prop:ulq-uvq} and~\ref{prop:q-cnf}: We just drop every
characteristic from the CNF \(\Delta_\gamma\) if that characteristic does not appear in population (term)~\(\gamma\).
As a result, the complete reason will be a {\em monotone} CNF, allowing one to enumerate 
sufficient reasons using a quasi-polynomial 
time algorithm~\cite{DBLP:journals/dam/GurvichK99}.\footnote{
%A monotone formula is a formula that satisfies
%the following property: For every variable \(X\), literals \(x\) and \(\n(x)\) cannot both occur in the formula. 
\cite{DBLP:journals/dam/GurvichK99} presented a quasi-polynomial time algorithm for the incremental 
enumeration of the prime implicants of a monotone CNF formula. 
This algorithm is based on the algorithm for the dualization problem (i.e., testing the duality of a pair of 
monotone DNF formulas) reported in~\cite{DBLP:journals/jal/FredmanK96}. These dualization  
algorithms have been implemented~\shortcite{DBLP:journals/dam/KhachiyanBEG06}, 
evaluated~\shortcite{DBLP:conf/alenex/HagenHM09}, and improved recently~\shortcite{DBLP:conf/wea/SedaghatSC18}.}
This also provides another characterization of the complete reason for
decision \(\Delta(\gamma)\) as the weakest CNF \(\Gamma\) that contains only characteristics that appears in
\(\gamma\) and that satisfies \(\gamma \models \Gamma \models \Delta_\gamma\).
In this sense, the complete reason is the most general abstraction of population \(\gamma\)
that justifies the decision made on \(\gamma\). 
Every aspect (set of characteristics) of \(\gamma\) that can trigger decision \(\Delta(\gamma)\)
is an implicant of the complete reason \(\Gamma\). Moreover, every prime implicant of 
\(\Gamma\) is an aspect of \(\gamma\). This justifies viewing
the complete reason for a decision as a necessary and sufficient condition for explaining the decision.

\subsection{Decision Bias}
\label{sec:bias}

We next show how universal quantification can be used to characterize biased decisions. 
Following~\cite{DarwicheH20}, these are decisions 
that are made based on features some of which are designated as {\em protected.} A classifier is biased precisely
when it makes at least one biased decision. Hence, a biased classifier may still make some unbiased decisions.

\begin{definition}
Let \(\Delta\) be a classifier where some of its features are designated as \hl{protected.}
A decision \(\Delta(\delta)\) on instance \(\delta\) is \hl{biased} iff \(\Delta(\delta) \neq \Delta(\delta^\star)\) for some instance
\(\delta^\star\) obtained from \(\delta\) by changing only the values of some protected features.
\end{definition}
We next show how variable quantification can be used to characterize all biased decisions that a classifier may make.

\begin{theorem}\label{theo:bias}
Let \(\Delta\) be a classifier and let \(X_1, \ldots, X_n\) be its protected features.
Then \(\Delta \wedge \neg (\forall X_1,$ $\ldots,$ $X_n \cdot \Delta)\) characterizes all positive instances on which \(\Delta\) 
will make a biased decision. Moreover,  \(\neg \Delta \wedge \neg (\forall X_1, \ldots, X_n \cdot \neg \Delta)\) characterizes
all negative instances on which \(\Delta\) will make a biased decision.
\end{theorem}

\cite{DarwicheH20} provided an efficient procedure for deciding whether a decision is biased,
assuming the classifier is represented using an appropriate tractable circuit. 
The above theorem suggests an efficient procedure for detecting decision bias assuming classifier \(\Delta\)
and its negation \(\neg \Delta\) are in CNF. To check whether a biased, positive decision is
made on instance \(\delta\), we just need to check whether \(\delta \models \Delta \wedge \neg (\forall X_1, \ldots, X_n \cdot \Delta)\).
Since \(\Delta\) is a CNF, \(\forall X_1, \ldots, X_n \cdot \Delta\) can be computed in linear time
so the previous test can be performed efficiently. A similar procedure can be used to detect biased negative decisions. 

Our treatment of decision bias extends the one in~\cite{DarwicheH20} not only computationally, but also in terms of scope.
Instead of only testing whether a particular decision is biased, we can now characterize all biased decisions which allows us to
entertain further questions relating to bias. Consider again the admission classifier and suppose that
feature \(R\) (rich hometown) is protected. The following expression will then characterize all
applicants on whom a biased, positive decision will be made:
\[
\Delta \wedge \neg (\forall R \cdot \Delta) = 
(e \vee g)\wedge(e \vee w)\wedge(r)\wedge(\n(f)\vee g)\wedge(\n(f)\vee w)\wedge(\n(e)\vee\n(f))
\]
There are six classes of applicants that satisfy the above formula. All will be admitted but they will be denied admission if
they were not to come from a rich hometown. We can now find out if any of these applicants could have failed the
entrance exam by computing:
\[
\n(e) \wedge (\Delta \wedge \neg (\forall R \cdot \Delta)) = \n(e)\wedge g \wedge r \wedge w
\]
The answer is affirmative. Moreover, the above result tells us that admitted applicants who fail the entrance 
exam and whose admission depends critically on coming from a rich hometown must have good grades 
and a work experience.

\section{Concluding Remarks}
\label{sec:conclusion}

We formalized and studied the universal quantification of literals in Boolean logic, together with its applications
to explainable AI. Our treatment was based on the novel notion of boundary models, which stands to have
implications on the study of Boolean logic beyond quantification. A major contribution of our work is the 
interpretation of universal quantification as a selection process, which we hope will expand the applications
of this form of quantification in AI and beyond. Another major contribution is the complexity results relating
to the computation of existential and universal quantification on various logical forms. While we were
driven by understanding universal literal quantification, our findings have furthered
our understanding of Boolean logic quantification more broadly. 

As to explainable AI, we have shown how to analyze classifiers and their decisions through the systematic construction 
of Boolean formulas, using universal literal and variable quantification. We provided some prototypical queries in 
earlier sections but one can go further beyond them. In a sense, the combination of universal quantifiers
with classical Boolean connectives provides a query language for interrogating classifiers and for gathering various insights
about how they make decisions and why they make these decisions. 

Our treatment of literal quantification can be extended to propositional formulas over discrete 
variables~\cite{MVLBook}, allowing one to reason about the behavior of discrete classifiers which arise in a number of contexts 
including decision trees and random forests~\cite<e.g.,>{ignatiev2019validating,kr/AudemardKM20,ChoiShihGoyankaDarwiche20}.
Consider a propositional formula \(\varphi\) over discrete variables and let \(X\) be a variable that has values \(x_1, \ldots, x_n\). 
We can generalize Definitions~\ref{def:ulq} and~\ref{def:elq} to quantify a literal \(x_i\) as follows:\footnote{The 
operator \(\exists x_i\) drops literal \(x_i\) from the expansion \(\varphi = \bigvee_{j} (x_j \wedge (\varphi | x_j))\).
The formula \(\varphi\) can also be expanded as 
\(\varphi = \bigwedge_{j} (\bigvee_{k \neq j} x_k \vee (\varphi | x_j))\); see the proof of Proposition~\ref{prop:sdd-transform}.
The operator \(\forall x_i\) drops all literals but \(x_i\) from this expansion.}
\begin{eqnarray*}
\forall x_i \cdot \varphi & = & (\varphi | x_i) \wedge \bigwedge_{j \neq i} (x_i \vee (\varphi | x_j)) \\
\exists x_i \cdot \varphi & = & (\varphi | x_i) \vee \bigvee_{j \neq i} (x_j \wedge (\varphi | x_j)).
\end{eqnarray*}
These quantifiers are also dual and have selection and forgetting semantics
as in the Boolean setting, therefore expanding the utility of our treatment beyond Boolean logic.

\color{\RED}
We close this section by a remark on further connections of our work to QBFs. Let \(\varphi\) be a Boolean formula
and  \(\L_1, \ldots, \L_n\) be a partition of all literals.
One can define another class of prenex and closed QBFs, $Q_1 \L_1, \ldots, Q_n \L_n \cdot \varphi,$ where 
$Q_1,\ldots,Q_n$ are alternating quantifiers in \(\{\forall, \exists\}$.
These QBFs are also guaranteed to be either valid or inconsistent. It remains to be explored though whether these QBFs
will also admit the notion of a solution that is normally defined for classical, prenex and closed QBFs. This is a subject for future work.
\color{black}

\acks{We thank the anonymous reviewers for their valuable feedback. 
This work has benefited from the support of the International Research Project MAKC (``Modern Approaches to
Knowledge Compilation'') shared between the Automated Reasoning Group of the University of California at Los Angeles (UCLA) 
and the Centre de Recherche en Informatique de Lens (CRIL UMR 8188 CNRS - Artois University). 
The work has been partially supported by NSF grant IIS-1910317, DARPA grant N66001-17-2-4032,
AI Chair EXPE{\bf KC}TATION (ANR-19-CHIA-0005-01) of the French National Research Agency (ANR)
and by TAILOR, a project funded by EU Horizon 2020 research and innovation programme under GA No 952215.}

\appendix 

\section{Characterizing the Dynamics of Boundary Rules}
\label{app:b-rule-dynamics}

\def\us{{\underline{\hspace{2mm}}}}

We provide in this appendix a complete characterization of which b-rules are deleted, introduced or preserved
when universally quantifying a literal. This allows us to define universal quantification as a process of b-rule 
transformation (one can use the duality theorem to prove similar results for existential literal quantification). 
These results are mostly meant for completeness as they are not essential for the main storyline in the paper.

\color{\RED}
To simplify notation, we will use the symbol $\us$ in b-rules to denote any term such that the result is a b-rule according to Definition \ref{def:brule}.
\color{black}
The operator \(\forall \l_i\) erases every b-rule of the form \(\Rule(\us,\nl_i)\). 
It also preserves all b-rules of the form \(\Rule(\us,\l_i)\) or \(\Rule({\us,\l_i},\us)\).
Finally, it preserves b-rules of the form \(\Rule({\us,\nl_i},\l_j)\) as long as there are no b-rules of the form \(\Rule({\us,\l_j},\nl_j)\).
We present this result formally next and then interpret it using boundary models.

\begin{theorem} \label{theo:brules-del-keep}
For formula \(\varphi\) and literals \(\l_i\) and \(\l_j\) where \(i \neq j\), we have
\begin{itemize}
\item[(a)] If \(\Rule(\alpha,\l_i) \in \RS(\varphi)\), then \(\Rule(\alpha,\l_i) \in \RS(\forall \l_i \cdot \varphi)\).
\item[(b)] \(\Rule(\alpha,\nl_i) \not \in \RS(\forall \l_i \cdot \varphi)\).
\item[(c)] If \(\Rule({\alpha,\l_i},\l_j) \in \RS(\varphi)\), then \(\Rule({\alpha,\l_i},\l_j) \in \RS(\forall \l_i \cdot \varphi)\).
\item[(d)] If \(\Rule({\alpha,\nl_i},\l_j) \in \RS(\varphi)\), then \(\Rule({\alpha,\nl_i},\l_j) \in \RS(\forall \l_i \cdot \varphi)\)
iff \(\Rule({\alpha,\l_j},\nl_i) \not \in \RS(\varphi)\).
\end{itemize}
\end{theorem}
\noindent According to this result, 
(a)~\(\l_i\)-boundary models of \(\varphi\) are \(\l_i\)-boundary models of \(\forall \l_i \cdot \varphi\),
(b)~\(\forall \l_i \cdot \varphi\) does not have any \(\nl_i\)-boundary models and
(c,d)~\(\l_j\)-boundary models of \(\varphi\) are \(\l_j\)-boundary models of \(\forall \l_i \cdot \varphi\) iff
they are not \(\nl_i\)-boundary models of \(\varphi\).

We now turn to the introduction of new b-rules which leads to introducing boundary models.
The following result says that the operator \(\forall \l_i\) can only introduce b-rules of the form \(\Rule({\us,\nl_i},\us)\)
so it will only introduce \(\l_j\)-boundary models that contain literal \(\nl_i\) (\(i \neq j\)).

\begin{theorem}\label{theo:brules-add1}
If \(r \not \in \RS(\varphi)\) and \(r \in \RS(\forall \l_i \cdot \varphi)\), then b-rule \(r\) has the form \(\Rule({\alpha,\nl_i},\l_j)\).
\end{theorem}

\noindent The next result provides a complete characterization for when universal literal quantification will introduce new b-rules (of the above form).

\begin{theorem}\label{theo:brules-add2}
Let b-rule \(r = \Rule({\alpha,\nl_i},\l_j)\). We then have
\(r \not \in \RS(\varphi)\) and \(r \in \RS(\forall \l_i \cdot \varphi)\) iff\\
\(\Rule({\alpha,\l_i},\l_j) \in \RS(\varphi)\),
\(\Rule({\alpha,\nl_j},\nl_i) \in \RS(\varphi)\),
\(\Rule({\alpha,\l_j},\l_i) \not \in \RS(\varphi)\) and
\(\Rule({\alpha,\nl_i},\nl_j) \not \in \RS(\varphi)\).
\end{theorem}
\noindent The proof of this theorem provides a characterization for when a model of \(\varphi\) that is not \(\l_j\)-boundary
will become an \(\l_j\)-boundary model of \(\forall \l_i \cdot  \varphi\): Flipping either literal \(\nl_i\) or literal \(\l_j\) will
preserve it as a model of \(\varphi\) but flipping both will not.

\section{Lemmas}
\label{app:lemmas}

The following lemmas are used in proofs of propositions and theorems.
The first lemma says that an \(\l\)-boundary model for a formula is preserved when dropping other models of the formula
(but some non-boundary models may become boundary).

\begin{lemma}\label{lem:drop}
Let \(\varphi\) and \(\phi\) be formulas and \(\w\) be a world. 
If \(\w \models \varphi\) and \(\varphi \models \phi\) and \(\w\) is an \(\l\)-boundary model for \(\phi\),
then \(\w\) is also an \(\l\)-boundary model for \(\varphi\).
\end{lemma}
\begin{proof}
Suppose \(\w \models \varphi\) and \(\varphi \models \phi\) and \(\w\) is an \(l\)-boundary model for \(\phi\). 
Then \(\w\) has the form \(\alpha,\l\)
where \(\alpha,\nl \models \neg \phi\). Then \(\phi \models \alpha \then \l\) and hence \(\varphi \models \alpha \then \l\)
and \(\varphi \wedge \alpha \wedge \nl\) is inconsistent. 
Therefore \(\alpha,\nl\) is not a model of \(\varphi\) and hence \(\w = \alpha,\l\) is an \(l\)-boundary model for \(\varphi\).
\end{proof}

\shrink{
%%% USEFUL but not used
\begin{lemma}
\label{lem:ncd}
For literal \(\l\) and formula \(\varphi\), we have \(\l \vee \varphi = \l \vee (\varphi \cd \nl)\).
\end{lemma}
\begin{proof}
\(\l \vee \varphi = \l \vee (\nl \wedge \varphi) = \l \vee ((\nl \wedge (\varphi \cd \nl)) = \l \vee (\varphi \cd \nl)\).
\end{proof}
}

\shrink{
%%% USEFUL but not used
\begin{lemma}\label{lem:ulq-false}
For formula \(\varphi\) and term \(\gamma = \l_1, \ldots, \l_n\),  we have
\(\gamma \models \neg \varphi\) only if \(\forall \l_1, \ldots, \l_n \cdot \varphi = \bot\).
\end{lemma}
\begin{proof}
Suppose \(\gamma = \l_1, \ldots, \l_n\), \(\gamma \models \neg \varphi\) but 
\(\forall \l_1, \ldots, \l_n \cdot \varphi \neq \bot\). We must then have a world
\(\w \models \forall \l_1, \ldots, \l_n \cdot \varphi\).
By Theorem~\ref{theo:ulq-semantics}, \(\w\) is an \(\alpha\)-independent model of \(\varphi\)
where \(\alpha = \w \cap \{\nl_1, \ldots, \nl_n\}\). 
We then have \(\w\setminus\alpha \models \varphi\).
Let \(\beta =  \{\l_i \mid \nl_i \in \alpha\}\).  
Then \((\w\setminus\alpha)\cup\beta \models \varphi\). 
Since \(\gamma \subseteq (\w\setminus\alpha) \cup \beta\) and  \(\gamma \models \neg \varphi\), we also have
\((\w\setminus\alpha)\cup\beta \models \neg \varphi\) which is a contradiction.
Hence, \(\forall \l_1, \ldots, \l_n \cdot \varphi = \bot\).
\end{proof}
}

\noindent This lemma says that universally quantifying a literal from a term keeps it intact if the negation of that literal does not appear in the term.
\begin{lemma}\label{lem:ulq-term}
Let \(\gamma\) be a term and \(\l\) be a literal such that \(\nl \not \in \gamma\).
Then \(\forall \l \cdot \gamma = \gamma\).
\end{lemma}
\begin{proof}
By Defintion~\ref{def:ulq}, \(\forall \l \cdot \gamma = (\l \vee (\gamma|\nl))\wedge(\gamma|\l)\).
If \(\l \not \in \gamma\),
\((\l \vee (\gamma|\nl))\wedge(\gamma|\l) = (\l \vee \gamma)\wedge\gamma = \gamma\).
If \(\l \in \gamma\),
\((\l \vee (\gamma|\nl))\wedge(\gamma|\l) = (\l \vee \bot)\wedge(\gamma\setminus\{\l\}) = \gamma\). 
Hence, \(\forall \l \cdot \gamma = \gamma\).
\end{proof}

\noindent This lemma provides another syntactic characterization of universal literal quantification.
\begin{lemma}\label{lem:ulq-syntax2}
For variable \(X\) and formula \(\varphi\), we have \(\forall x \cdot \varphi = (\forall X \cdot \varphi) \vee (x \wedge \varphi)\)
and \(\forall \n(x) \cdot \varphi = (\forall X \cdot \varphi) \vee (\n(x) \wedge \varphi)\).
\end{lemma}
\begin{proof}
\(\forall x \cdot \varphi  = (x \vee (\varphi \cd \n(x))) \wedge  (\varphi \cd x) = (x \wedge (\varphi \cd x)) \vee ((\varphi \cd \n(x))\wedge(\varphi \cd x))
= (x \wedge \varphi) \vee (\forall X \cdot \varphi)\). We can similarly show \(\forall \n(x) \cdot \varphi = (\forall X \cdot \varphi) \vee (\n(x) \wedge \varphi)\).
\end{proof}

\noindent This lemma says that the \(\alpha\)-independent models of a formula are preserved when adding more models to the formula. 
\begin{lemma}\label{lem:imodel-s2w}
Let \(\varphi\) and \(\phi\) be formulas and \(\w\) be a world such that
\(\w \models \phi \models \varphi\). If \(\w\) is an \(\alpha\)-independent model of \(\phi\),
then \(\w\) is also an \(\alpha\)-independent model of \(\varphi\).
\end{lemma}
\begin{proof}
Suppose \(\w \models \phi \models \varphi\) and \(\w\) is an \(\alpha\)-independent model of \(\phi\).
Then \(\alpha \subseteq \w\) and \(\w\setminus\alpha \models \phi\). 
Hence, \(\w\setminus\alpha \models \varphi\) which establishes 
\(\w\) as an \(\alpha\)-independent model of \(\varphi\).
\end{proof}

\section{Proofs}
\label{app:proofs}

%\subsection*{Section~\ref{sec:vq}}

The next three results are folklore. We provide proof sketches for the sake of completeness.

\begin{proof}[\bf Proof of Proposition~\ref{prop:evq}]
This result has been known. See for example Theorem~6.1 in~\cite{DBLP:conf/ijcai/KatsunoM89} and Corollary 6 in~\cite{LangLM03}. 
Since $\exists X \cdot \varphi = (\varphi \cd x) \vee (\varphi \cd \n(x))$, we get that $\exists X \cdot \varphi$
is independent of $X$ since literals $x$ and $\n(x)$ do not occur in $(\varphi \cd x) \vee (\varphi \cd \n(x))$.
Suppose now that there exists a formula $\psi$ such that $\varphi \models \psi$, $\exists X \cdot \varphi \not \models \psi$ and variable \(X\)
does not occur in $\psi$. We will next show a contradiction.
Since any $\psi$ can be put into an equivalent CNF, we can assume without loss of generality that
$\psi$ is a clause. Since $\varphi \models \psi$, every model of $\varphi$ must satisfy a literal $\l$ of $\psi$. Since $\exists X \cdot \varphi \not \models \psi$,
there must be a model $\w$ of $\exists X \cdot \varphi$ that does not satisfy any literal of $\psi$. The models of $\exists X \cdot \varphi$
are the models of $\varphi$ plus all those worlds that differ from a model of $\varphi$ on $X$ only. Since $\w$ cannot be a model of $\varphi$, there must be
a model $\w'$ of $\varphi$ such that $\w$ and $\w'$ coincide on every variable but $X$. Since $\psi$ does not contain any occurrence of $X$,
the variable of $\l$ is different from $X$, and as a consequence, since $\w'$ is a model of $\varphi$, $\w$ satisfies $\varphi$ as well.
This is a contradiction.
\end{proof}

\begin{proof}[\bf Proof of Proposition~\ref{prop:uvq}]
Follows directly from Propositions~\ref{prop:evq} and~\ref{prop:evquvq-duality}.
\end{proof}

\begin{proof}[\bf Proof of Proposition~\ref{prop:evquvq-duality}]
Follows from Definitions~\ref{def:evq} and~\ref{def:uvq} using De Morgan's law.
\end{proof}

%\subsection*{Section~\ref{sec:brules}}

\def\S{{\cal S}}
\def\L{{\cal L}}

\begin{proof}[\bf Proof of Theorem~\ref{theo:know}]~

(\(\then\))~Suppose $\MS(\varphi_1) = \MS(\varphi_2)$. 
Then $\RS(\varphi_1) = \RS(\varphi_2)$ by Definition~\ref{def:brule} (and Definition~\ref{def:bmodel}).

(\(\bthen\))~It suffices to show that for a consistent formula \(\varphi\), its models $\MS(\varphi)$ are fully characterized by its b-rules $\RS(\varphi)$.
We will show this by defining an operator \(\L\) that depends only on b-rules $\RS(\varphi)$ and then show that
$\MS(\varphi)$ is the stationary point of the sequence $(\L^i(\BMS(\varphi)))_{i \in \mathbb{N}}$,
where $\L^0(W) = W$ and $\L^{i+1}(W) = \L(\L^{i}(W)).$ Note that
the boundary models of \(\varphi\) also depend only on b-rules $\RS(\varphi)$ since \(\BMS(\varphi) = \{ \alpha,\l \mid \Rule(\alpha,\l) \in \RS(\varphi)\}\).
The operator \(\L: \S \mapsto \S$ is defined as follows, where $\S$ consists of all sets of worlds which contain boundary models $\BMS(\varphi)$:
\begin{eqnarray*}
\S      &  = & \{W \mid \BMS(\varphi) \subseteq W \subseteq {2^\ps}\} \\
\L(W) & = & W \cup \{\alpha, \nl \mid \alpha, \l \in W \mbox{ and } \Rule(\alpha,\l) \not \in \RS(\varphi)\}
\end{eqnarray*}
The operator \(\L\) grows the set of worlds \(W\) as follows. 
For each world \(\alpha, \l \in W\),  it adds world \(\alpha, \nl\) in case \(\Rule(\alpha,\l) \not \in \RS(\varphi)\).
This condition is equivalent to: \(\alpha, \l \in \MS(\varphi)\) only if \(\alpha, \nl \in \MS(\varphi)\). 
If we apply the operator \(\L\) to the boundary models \(\BMS(\varphi)\), it will infer additional models of \(\varphi\) and add them to \(W\). 
Applying \(\L\) again to the result will infer/add more models and so on. This is precisely what the sequence $(\L^i(\BMS(\varphi)))_{i \in \mathbb{N}}$ does.
It suffices now to show that $\MS(\varphi)$ is the least fixed point of operator $\L$. We need a few lemmas for this,
the first shows that \(\L\) must have a least fixed point.

The pair \((\S,\subseteq)\) forms a complete lattice with $\BMS(\varphi)$ as the least element and $2^\ps$ as the greatest element.
If $\L$ is monotonic, then by the Knaster–Tarski Theorem~\cite{Knaster28,Tarski55},
the set of fixed points of $\L$ is not empty and also forms a complete lattice so it does have a least element.

\renewcommand{\quote}{\list{}{\rightmargin=\leftmargin\topsep=0pt}\item\relax}

\begingroup
\addtolength\leftmargini{-7mm}
\begin{quote}
\begin{lemma}\label{lem:monotonic}
$\L: \S \mapsto \S$ is monotonic: $W \subseteq W'$ only if $\L(W) \subseteq \L(W')$ for every $W, W' \in \S$.
\end{lemma}
\end{quote}
\endgroup

\begin{proof}
Suppose $W \subseteq W'$. 
If $\w \in \L(W)$, then by definition of \(\L\) we have either (1)~$\w \in W$ or 
(2)~$\w = \alpha, \nl$ where $\alpha, \l \in W$ and $\alpha \rightarrow \l \not \in \RS(\varphi)$. 
Case~(1) implies $\w \in W'$ and hence $\w \in \L(W')$. 
Case~(2) implies $\alpha, \l \in W'$ and hence $\w = \alpha, \nl \in \L(W')$ since $\alpha \rightarrow \l \not \in \RS(\varphi)$. 
\end{proof}

\shrink{
%%%. USEFUL but not used
\begin{lemma}\label{lem:fp}
$\MS(\varphi)$ is a fixed point of $\L$:  $\L(\MS(\varphi))=\MS(\varphi)$.
\end{lemma}

\begin{proof}
\(\MS(\varphi)\) is in the domain \(\S\) of $\L$ since \(\MS(\varphi) \supseteq \BMS(\varphi)\).
By definition of  $\L$, $\MS(\varphi) \subseteq \L(\MS(\varphi))$. 
Suppose now that there is a world $\alpha, \nl \in \L(\MS(\varphi)) \setminus \MS(\varphi)$.
Since $\alpha, \nl \in \L(\MS(\varphi))$, there exists a world $\alpha, \l \in \MS(\varphi)$
such that $\alpha \rightarrow \l \not \in \RS(\varphi)$. But if $\alpha, \l \in \MS(\varphi)$ and $\alpha, \nl \not \in \MS(\varphi)$ then
$\alpha \rightarrow \l \in \RS(\varphi)$, which is a contradiction. Hence,
$\L(\MS(\varphi)) \setminus \MS(\varphi) = \emptyset$ and therefore $\L(\MS(\varphi)) \subseteq \MS(\varphi)$.
\end{proof}
}

The next two lemmas use the notion of a {\em Hamming path (H-path).} 
An H-path from world \(\w_1\) to world \(\w_d\) is a sequence of worlds $\w_1, \ldots, \w_d$ where world \(\w_i\), \(i > 1\), 
is obtained from world \(\w_{i-1}\) by flipping the value of a single variable. This H-path has length \(d\).

\begingroup
\addtolength\leftmargini{-7mm}
\begin{quote}
\begin{lemma}\label{lem:path}
If \(\w_1, \ldots, \w_d\) is a shortest H-path from \(\BMS(\varphi)\) to \(\w_d \in \MS(\varphi)\), then \(\w_i \in \MS(\varphi)\) for \(i=1, \ldots, d\).
\end{lemma}
\end{quote}
\endgroup

\begin{proof}
The lemma holds trivially for \(d \in \{1,2\}\). The proof for \(d \geq 3\) is by contradiction. 
Let \(k\) be the largest index such that \(\w_k \not \in \MS(\varphi)\).  Then \(1 < k < d\).
Since \(\w_{k+1} \in \MS(\varphi)\) and \(\w_{k}\) is obtained by flipping a single variable in \(\w_{k+1}\), we have
\(\w_{k+1} \in \BMS(\varphi)\). Moreover, \(\w_{k+1}, \ldots, \w_{d}\) is an H-path from \(\BMS(\varphi)\)
to \(\w_d\) which has length \(d-k < d\), a contradiction. Hence, \(\w_i \in \MS(\varphi)\) for \(i=1,\ldots, d\).
\end{proof}

\begingroup
\addtolength\leftmargini{-7mm}
\begin{quote}
\begin{lemma}\label{lem:lfp}
If $\w_1, \ldots, \w_d$ is a shortest H-path from $\BMS(\varphi)$ to $\w_d \in \MS(\varphi)$, then 
$\w_d \in \L^{d-1}(\BMS(\varphi))$.
\end{lemma}
\end{quote}
\endgroup

\begin{proof}
By induction on $d$. For \(d=1\) (base case), we have \(\w_1=\w_d \in \BMS(\varphi)$,
$\L^{0}(\BMS(\varphi)) = \BMS(\varphi)$ and hence \(\w_d \in \L^{d-1}(\BMS(\varphi))\).
For $d > 1$ (inductive step), the subsequence $\w_1, \ldots, \w_{d-1}$ is a shortest H-path from $\BMS(\varphi)$ to $\w_{d-1}$.
Moreover, \(\w_{d-1} \in \MS(\varphi)\) by Lemma~\ref{lem:path}.
Hence, $\w_{d-1} \in \L^{d-2}(\BMS(\varphi))$ by the induction hypothesis.
Worlds \(\w_{d-1}\) and \(\w_d\) must have the forms \(\alpha,\l\) and \(\alpha,\nl\). 
Since both are in \(\MS(\varphi)\), then \(\Rule(\alpha,\l) \not \in \RS(\varphi)\). 
Hence \(\w_d \in \L^{d-1}(\BMS(\varphi))$ since $\L^{d-1}(\BMS(\varphi))=\L(\L^{d-2}(\BMS(\varphi)))$.
\end{proof}

We will now finish the proof by showing that \(\MS(\varphi)\) is the least fixed point of \(\L\).
By Lemma~\ref{lem:lfp}, $\MS(\varphi) \subseteq \L^{k}(\BMS(\varphi))$ for some \(k \geq 0\).
Suppose that \(\BMS(\varphi) \subseteq W\) and \(\L(W)=W\) (a fixed point).
Then $\L^{k}(W) = W$. Since $\L$ is monotonic (Lemma \ref{lem:monotonic}), 
we have $\L^{k}(\BMS(\varphi)) \subseteq \L^{k}(W)$ and hence $\MS(\varphi) \subseteq \L^{k}(\BMS(\varphi)) \subseteq \L^{k}(W) = W$.
This implies $\MS(\varphi) \subseteq W$ so $\MS(\varphi)$ is the least fixed point of~\(\L\)
and therefore the stationary point of the sequence $(\L^i(\BMS(\varphi)))_{i \in \mathbb{N}}$.
\end{proof}

\begin{proof}[\bf Proof of Proposition~\ref{prop:brule}]~

(\(\then\))~Suppose \(\varphi\) has b-rule \(\Rule(\alpha,\l)\). By Definition~\ref{def:brule}, \(\alpha,\l\) is an \(\l\)-boundary model of \(\varphi\).
By Definition~\ref{def:bmodel}, \(\alpha,\l \models \varphi\) and \(\alpha,\nl \models \neg \varphi\).
From \(\alpha,\l \models \varphi\), we conclude that \(\varphi \wedge \alpha\) is consistent.
From \(\alpha,\nl \models \neg \varphi\), we conclude that \(\alpha \wedge \nl \wedge \varphi\) is inconsistent and hence 
\(\varphi \models \neg(\alpha \wedge \nl)\) and further \(\varphi \models \alpha\then\l\).

(\(\bthen\))~Suppose \(\varphi\wedge\alpha\) is consistent and \(\varphi \models \alpha\then\l\).
From \(\varphi\wedge\alpha\) being consistent, we conclude that \(\alpha,\l \models \varphi\) or \(\alpha,\nl \models \varphi\).
From \(\varphi \models \alpha\then\l\), we conclude \(\alpha,\nl\ \models \neg \varphi\) and hence \(\alpha,\l \models \varphi\).
We now have that \(\alpha,\l\) is an \(\l\)-boundary model of \(\varphi\) by Definition~\ref{def:bmodel},
and hence \(\varphi\) has b-rule \(\Rule(\alpha,\l)\) by Definition~\ref{def:brule}.
\end{proof}

\begin{proof}[\bf Proof of Proposition~\ref{prop:rule-independence}]~
\color{\BLUE}

{\em The literal case.}
(\(\then\))~Suppose \(\varphi\) is independent of literal \(\ell\). There must exist some NNF \(\psi = \varphi\) that does not contain literal \(\ell\).
Let \(\alpha\) be a term that does not mention the variable \(X\) of $\ell$ and is such that \(\psi \wedge \alpha\) is consistent.
Then \(\psi \wedge \alpha\) is a consistent NNF that does not mention literal \(\ell\)
so \((\psi \wedge \alpha)|\n(\ell)\) is also consistent and hence \(\psi \wedge \alpha \not \models \ell\). 
This means \(\varphi\) cannot have a b-rule of the form \(\Rule(\alpha,\ell)\).

(\(\bthen\))~Suppose \(\varphi\) has no b-rule of the form \(\Rule(\alpha,\ell)\). If \(\varphi\) has a model of the form \(\alpha,\ell\),
it also has \(\alpha,\n(\ell)\) as a model; otherwise, \(\varphi\) will have rule \(\Rule(\alpha,\ell)\). Since 
\((\alpha,\ell) \vee (\alpha,\n(\ell)) = \alpha\), we can express \(\varphi\) as a DNF that does not contain literal \(\ell\). Hence,
\(\varphi\) is independent of literal~\(\ell\).

\color{black}

{\em The variable case.} Follows from the above case given that \(\varphi\) is independent of variable \(X\) iff it
is independent of literal \(x\) and independent of literal \(\n(x)\).
\end{proof}

\begin{proof}[\bf Proof of Proposition~\ref{prop:rconnect}]
This proof invokes Definition~\ref{def:brule} frequently.
\begin{itemize}
\item [(a)] We have $\alpha \rightarrow \l \in \RS(\overline{\varphi})$ 
iff $\alpha, \l \models \overline{\varphi}$ and $\alpha, \nl \models \overline{\overline{\varphi}}$  
iff  $\alpha \rightarrow \nl \in \RS(\varphi)$.

\item [(b)] 
\begin{itemize}

\item If $\alpha \rightarrow \l \in \RS(\varphi) \cap \RS(\psi)$,
then $\alpha, \l \models \varphi$, $\alpha, \nl \models \overline{\varphi}$, $\alpha, \l \models \psi$ and $\alpha, \nl \models \overline{\psi}$. 
Thus, $\alpha, \l \models \varphi \wedge \psi$ and $\alpha, \nl \models \overline{\varphi} \vee \overline{\psi}$.
Since $\overline{\varphi} \vee \overline{\psi} = \overline{\varphi \wedge \psi}$, we get
 $\Rule(\alpha,\l) \in \RS(\varphi \wedge \psi)$ and therefore $\RS(\varphi) \cap \RS(\psi) \subseteq \RS(\varphi \wedge \psi)$.

\item If $\alpha \rightarrow \l \in \RS(\varphi \wedge \psi)$, then $\alpha, \l \models \varphi \wedge \psi$ and $\alpha, \nl \models \overline{\varphi} \vee \overline{\psi}$.
Hence, $\alpha, \l \models \varphi$ and  $\alpha, \l \models \psi$. 
Moreover, $\alpha, \nl \models \overline{\varphi}$ or  $\alpha, \nl \models \overline{\psi}$.
This gives  $\Rule(\alpha,\l) \in \RS(\varphi)$ or  $\Rule(\alpha,\l) \in \RS(\psi)$, and 
therefore $\RS(\varphi \wedge \psi) \subseteq \RS(\varphi) \cup \RS(\psi)$.

\end{itemize}

\item [(c)] 
\begin{itemize}

\item If $\alpha \rightarrow \l \in \RS(\varphi) \cap \RS(\psi)$, then $\alpha, \l \models \varphi$, $\alpha, \nl \models \overline{\varphi}$, $\alpha, \l \models \psi$
and $\alpha, \nl \models \overline{\psi}$. 
Thus, $\alpha, \l \models \varphi \vee \psi$ and $\alpha, \nl \models \overline{\varphi} \wedge \overline{\psi}$. 
Since $\overline{\varphi} \wedge \overline{\psi} = \overline{\varphi \vee \psi}$, 
we get $\Rule(\alpha,\l) \in \RS(\varphi \vee \psi)$ and therefore $\RS(\varphi) \cap \RS(\psi) \subseteq \RS(\varphi \vee \psi)$.

\item If $\alpha \rightarrow \l \in \RS(\varphi \vee \psi)$, then $\alpha, \l \models \varphi \vee \psi$ and $\alpha, \nl \models \overline{\varphi} \wedge \overline{\psi}$. 
Hence, $\alpha, \l \models  \overline{\varphi}$ and  $\alpha, \l \models \overline{\psi}$. 
Moreover, $\alpha, \l \models \varphi$ or  $\alpha, \l \models \psi$.
This gives  $\Rule(\alpha,\l) \in \RS(\varphi)$ or  $\Rule(\alpha,\l) \in \RS(\psi)$, and 
therefore $\RS(\varphi \vee \psi) \subseteq \RS(\varphi) \cup \RS(\psi)$. \qedhere
\end{itemize}

\end{itemize}
\end{proof}

\color{\BLUE}
\begin{proof}[\bf Proof of Proposition~\ref{prop:rcount}]
By Definition~\ref{def:brule}, a boundary model can generate at most \(n\) b-rules, hence \(|\RS(\varphi)| \leq n \cdot |\BMS(\varphi)|\).
The bound \(n \cdot |\BMS(\varphi)| \leq n \cdot |\MS(\varphi)|\) is trivial since \(\BMS(\varphi) \subseteq \MS(\varphi)\). 
Considering $\overline{\varphi}$ instead of $\varphi$, we have $|\RS(\overline{\varphi})| \leq n \cdot |\BMS(\overline{\varphi})| \leq n \cdot |\MS(\overline{\varphi})|$.
Finally, Proposition~\ref{prop:rconnect}(a) shows that $|\RS(\overline{\varphi})| = |\RS(\varphi)|$.
%Pierre
%The bound \(|\RS(\varphi)| \leq n \cdot |\BMS(\varphi)|\) is tight since for instance $\varphi = x_1, \ldots, x_n$
%has a single boundary model but \(n\) b-rules, $\RS(\varphi) = \{(\bigwedge_{j = 1, \ldots, n \mid j \neq i} x_j) \rightarrow x_i \mid i = 1, \ldots, n\}.$
%Consider now the formula $\varphi = x_1 \vee \ldots \vee x_n$ which has $2^n -1 = |\MS(\varphi)|$ models.
%This formula has only \(n\) b-rules, $\RS(\varphi) = \{\Rule({(\bigwedge_{j = 1, \ldots, n \mid j \neq i} {\n(x)}_j)}, x_i) \mid i = 1, \ldots, n\}$,
%and \(n\) boundary models, $\BMS(\varphi) = \{x_i \wedge \bigwedge_{j = 1, \ldots, n \mid j \neq i} {\n(x)}_j \mid i = 1, \ldots, n\}$.
\end{proof}
\color{black}

%\subsection*{Section~\ref{sec:lq}}

%%%%%%%%%%%%%%
%% Universal
%%%%%%%%%%%%%%

\begin{proof}[\bf Proof of Theorem~\ref{theo:ulq-semantics-b}]
By Boole's expansion, we have \(\varphi = (\l \wedge (\varphi \cd \l)) \vee (\nl \wedge (\varphi \cd \nl))\), 
which can be expanded using consensus into 
\(\varphi = (\l \wedge (\varphi \cd \l)) \vee (\nl \wedge (\varphi \cd \nl)) \vee ((\varphi \cd \l)\wedge(\varphi \cd \nl))\).
By Definition~\ref{def:ulq}, we have 
\(\forall \l \cdot \varphi = (\l \vee (\varphi \cd \nl)) \wedge  (\varphi \cd \l)\), which can be expanded into
\(\forall \l \cdot \varphi = (\l \wedge (\varphi \cd \l)) \vee ((\varphi \cd \l)\wedge(\varphi \cd \nl))\). 
This gives \(\forall \l \cdot \varphi \models \varphi\) and hence \(\MS(\forall \l \cdot \varphi) \subseteq \MS(\varphi)\).
We next prove the two directions of the second part of the theorem using the expansions:
\begin{eqnarray}
\varphi & = & (\l \wedge (\varphi \cd \l)) \vee (\nl \wedge (\varphi \cd \nl)) \label{exp1} \\
\varphi & = & (\l \wedge (\varphi \cd \l)) \vee (\nl \wedge (\varphi \cd \nl)) \vee ((\varphi \cd \l)\wedge(\varphi \cd \nl)) \label{exp1b} \\
\forall \l \cdot \varphi & = &  (\l \wedge (\varphi \cd \l)) \vee ((\varphi \cd \l)\wedge(\varphi \cd \nl)) \label{exp2}
\end{eqnarray}
(\(\then\))~Suppose \(\w \in \MS(\varphi)\) and \(\w \not \in \MS(\forall \l \cdot \varphi)\);
that is, \(\w \models \varphi\) and \(\w \not \models \forall \l \cdot \varphi\). 
Given Expansions~(\ref{exp1}) and~(\ref{exp2}), this implies \(\w \models \nl \wedge (\varphi \cd \nl)\).
Suppose \(\w\) is not an \(\nl\)-boundary model of \(\varphi\). 
Then \(\w[\l] \models \varphi\) and hence \(\w[\l] \models \l \wedge \varphi\) and \(\w[\l] \models \l \wedge (\varphi \cd \l)\).
We now have \(\w \models \varphi \cd \nl\) and \(\w[\l] \models \varphi \cd \l\). 
Since \(\varphi \cd \l\) does not mention the variable of literal \(\l\), we also have \(\w \models (\varphi \cd \l) \wedge  (\varphi \cd \nl)\)
which implies \(\w \models \forall \l \cdot \varphi\) by Expansion~(\ref{exp2}).
This is a contradiction with the supposition \(\w \not \in \MS(\forall \l \cdot \varphi)\) so \(\w\) must be an \(\nl\)-boundary model of \(\varphi\).\\
(\(\bthen\))~Suppose \(\w\) is an \(\nl\)-boundary model of \(\varphi\). Then \(\nl \in \w\), \(\w \models \varphi\)
and \(\w[\l] \not \models  \varphi\). We then have \(\w \in \MS(\varphi)\) so we just need to show that \(\w \not \in \MS(\forall \l \cdot \varphi)\).
Suppose \(\w \in \MS(\forall \l \cdot \varphi)\). Since \(\nl \in \w\), we have \(\w \models (\varphi \cd \l)\wedge(\varphi \cd \nl)\)
by Expansion~(\ref{exp2}). Since \(\varphi \cd \l\) and \(\varphi \cd \nl\) do not mention the variable
of literal \(\nl\), we also have \(\w[\l] \models (\varphi \cd \l)\wedge(\varphi \cd \nl)\) and hence \(\w[\l] \models \varphi\) 
by Expansion~(\ref{exp1b}). This is a contradiction so \(\w \not \in \MS(\forall \l \cdot \varphi)\).
\end{proof}

\begin{proof}[\bf Proof of Theorem~\ref{theo:ulq-syntax}]
We first prove $\forall \l . \varphi \models \varphi \wedge \bigwedge_{\alpha \rightarrow \nl \in \RS(\varphi)} \overline{\alpha}$.
By Theorem~\ref{theo:ulq-semantics-b}, $\forall \l . \varphi \models \varphi$.
We next show that $\forall \l . \varphi \models \overline{\alpha}$ whenever 
$\alpha \rightarrow \nl \in \RS(\varphi)$. If $\alpha \rightarrow \nl \in \RS(\varphi)$, we have $\alpha, \l \models \overline{\varphi}$
by Definition~\ref{def:brule} and hence $\varphi \models \overline{\alpha} \vee \nl$. By Proposition~\ref{prop:ulq-imply},
$\forall \l . \varphi \models \forall \l . (\overline{\alpha} \vee \nl)$. Since $\alpha$ does not contain the variable of literal $\l$ by construction,
$\forall \l . (\overline{\alpha} \vee \nl)=(\forall \l . \overline{\alpha}) \vee (\forall \l . \nl)$
by Proposition~\ref{prop:quantify-compound}(d). 
Since $\forall \l . \nl = \bot$ we get $\forall \l . \varphi \models \forall \l . \overline{\alpha}$. 
By Theorem~\ref{theo:ulq-semantics-b}, $\forall \l . \overline{\alpha} \models \overline{\alpha}$ so \(\forall \l . \varphi \models \overline{\alpha}\) 
and therefore $\forall \l . \varphi \models \varphi \wedge \bigwedge_{\alpha \rightarrow \nl \in \RS(\varphi)} \overline{\alpha}$.

We now prove $\varphi \wedge \bigwedge_{\alpha \rightarrow \nl \in \RS(\varphi)} \overline{\alpha} \models \forall \l . \varphi$.
By Definition \ref{def:ulq}, $\forall \l . \varphi = (\l \vee (\varphi \mid \nl)) \wedge (\varphi \mid \l)$.
Consider a model $\w \models \varphi \wedge \bigwedge_{\alpha \rightarrow \nl \in \RS(\varphi)} \overline{\alpha}$.
We will next show that \(\w \models  \forall \l . \varphi\).
If $\w \models \l$, then $\w \models \varphi \wedge \l$. Since $\varphi \wedge \l \models \varphi \mid \l$, we now have 
$\w \models \varphi \mid \l\) and hence $\w \models \forall \l . \varphi$.
If $\w \models \nl$, then $\w \models \varphi \wedge \nl$ and hence  $\w \models \varphi \mid \nl$. 
We next show $\w \models \varphi \mid \l$ which gives us $\w \models \forall \l . \varphi$, therefore concluding the proof.
Let $\w = \alpha', \nl$. We must have $\alpha', \l \models \varphi$, otherwise $\Rule(\alpha',\nl) \in \RS(\varphi)$ 
and then \(\w \models \overline{\alpha'}\) which is a contradiction. We now have \(\alpha' \models  \varphi\).
Since $\varphi=(\l \vee (\varphi \mid \nl)) \wedge (\nl \vee (\varphi \mid \l))$,
we have $\alpha' \models \nl \vee (\varphi \mid \l)$. Since $\nl \vee (\varphi \mid \l)$ is independent of \(\l\), we also have
 $\alpha' \models \forall \l . (\nl \vee (\varphi \mid \l))$. 
Since $\nl$ and $\varphi \mid \l$ do not share variables, $\forall \l . (\nl \vee (\varphi \mid \l))=(\forall \l . \nl) \vee (\forall \l . (\varphi \mid \l))$
by Proposition~\ref{prop:quantify-compound}(d).
Moreover, $(\forall \l . \nl) \vee (\forall \l . (\varphi \mid \l))=\varphi \mid \l$ since $\forall \l . \nl=\bot$ 
and $\varphi \mid \l$ is independent of $\l$. We now have \(\alpha' \models \varphi \mid \l\) and hence $\w \models \varphi \mid \l$.
Therefore $\varphi \wedge \bigwedge_{\alpha \rightarrow \nl \in \RS(\varphi)} \overline{\alpha} \models \forall \l . \varphi$.
\end{proof}

\begin{proof}[\bf Proof of Proposition~\ref{prop:ulq-preserve}]
By Lemma~\ref{lem:ulq-syntax2}, we have \(\forall \l \cdot \varphi = (\forall X \cdot \varphi) \vee (\l \wedge \varphi)\), where
\(X\) is the variable of literal \(\l\). Any implicant of \(\varphi\) that contains literal \(\l\) will also be an implicant of \(\l \wedge \varphi\) 
and hence an implicant of \(\forall \l \cdot \varphi\).
\end{proof}

\begin{proof}[\bf Proof of Proposition~\ref{prop:ulq-imply}]
By Theorem~\ref{theo:ulq-semantics}, \(\forall \l \cdot \varphi \models \varphi\).
Given \(\varphi \models \phi\), we now have \(\forall \l \cdot \varphi \models \phi\).
Let \(\w\) be a model of \(\forall \l \cdot \varphi\). Then \(\w \models \forall \l \cdot \varphi \models \varphi \models \phi\).
If \(\w \not  \models \forall \l \cdot \phi\) then \(\w\) is an \(\nl\)-boundary model of \(\phi\) by Theorem~\ref{theo:ulq-semantics}
and hence \(\w\) is an \(\nl\)-boundary model of \(\varphi\) by Lemma~\ref{lem:drop}.
By Theorem~\ref{theo:ulq-semantics}, \(\w \not \models \forall \l \cdot \varphi\) which is a contradiction.
We then have \(\w \models \forall \l \cdot \phi\) and therefore \(\forall \l \cdot \varphi \models \forall \l \cdot \phi\).
\end{proof}

\begin{proof}[\bf Proof of Proposition~\ref{prop:ulq-semantics-old}]
Proposition 2(4) in~\cite{LangLM03} shows that $\varphi$ is independent of literal $\l$ if and only if $\neg \varphi$ is independent of literal $\nl$.
Hence, the result follows directly from Proposition~\ref{prop:elq-semantics-old} (Proposition 16 of~\cite{LangLM03}) and Theorem \ref{theo:lq-duality} (duality).
\end{proof}

\begin{proof}[\bf Proof of Proposition~\ref{prop:ulq-order}]
Follows directly from Proposition~\ref{prop:elq-order} and Theorem \ref{theo:lq-duality}.
\end{proof}

\begin{proof}[\bf Proof of Proposition~\ref{prop:ulq-uvq}]
Follows directly from Proposition~\ref{prop:elq-evq} and Theorem \ref{theo:lq-duality}.
\end{proof}
\color{black}

\begin{proof}[\bf Proof of Theorem~\ref{theo:ulq-semantics}]
\color{\BLUE}
The proof of this theorem is based on two lemmas that we state and prove next.
The first lemma says that the result of universally quantifying literals is independent of their complements.
\begin{lemma}\label{lem:ulq-independence}
Let \(\varphi\) be a formula, \(\l_1,\ldots,\l_n\) be literals and \(\w\) be a world.
If \(\w\models \forall \l_1,\ldots,\l_n \cdot \varphi\), then
\(\w\setminus\{\nl_1,\ldots,\nl_n\} \models \forall \l_1,\ldots,\l_n \cdot \varphi\).
\end{lemma}
\begin{proof}
Suppose \(\w\models \forall \l_1,\ldots,\l_n \cdot \varphi\).
By Proposition~\ref{prop:ulq-semantics-old}, \(\forall \l_1,\ldots,\l_n \cdot \varphi\) is independent of literals \(\nl_1,\ldots,\nl_n\).
Thus, by Proposition~\ref{prop:elq-semantics-old}, \(\exists \nl_1, \ldots, \nl_n \cdot \w\models \forall \l_1,\ldots,\l_n \cdot \varphi\).
Since \(\exists \nl_1, \ldots, \nl_n \cdot \w = \w\setminus\{\nl_1, \ldots, \nl_n\}\), we finally have 
\(\w\setminus\{\nl_1,\ldots,\nl_n\} \models \forall \l_1,\ldots,\l_n \cdot \varphi\).
\end{proof}
The second lemma identifies a class of \(\alpha\)-independent models that are preserved by universal literal quantification.
\begin{lemma}\label{lem:imodel-w2s}
Let \(\varphi\) be a formula, \(\l_1,\ldots,\l_n\) be literals, \(\w\) be a world and \(\alpha=\w\cap\{\nl_1,\ldots,\nl_n\}\).
If \(\w\) is an \(\alpha\)-independent model of \(\varphi\), then \(\w\) is an \(\alpha\)-independent model 
of \(\forall \l_1,\ldots,\l_n \cdot \varphi\).
\end{lemma}
\begin{proof}
Suppose \(\w\) is an \(\alpha\)-independent model of \(\varphi\).
Then \(\alpha \subseteq \w\) and \(\w\setminus\alpha \models \varphi\).
By Proposition~\ref{prop:ulq-imply}, \(\forall \l_1,\ldots,\l_n \cdot (\w\setminus\alpha) \models \forall \l_1,\ldots,\l_n \cdot \varphi\).
Since \(\nl_i \not \in (\w\setminus\alpha)\) for \(i=1,\ldots,n\), we get
\(\forall \l_1, \ldots, \l_n \cdot (\w\setminus\alpha) = \w\setminus\alpha\) by Lemma~\ref{lem:ulq-term}.
Hence, \(\w\setminus\alpha \models  \forall \l_1,\ldots,\l_n \cdot\varphi\)
so \(\w\) is an \(\alpha\)-independent model of \(\forall \l_1,\ldots,\l_n \cdot\varphi\).
\end{proof}
\color{black}

We are now ready to prove the theorem. Let \(\w\) be a world and \(\alpha=\w\cap\{\nl_1,\ldots,\nl_n\}\).

(\(\then\))~Suppose \(\w \models \forall \l_1,\ldots,\l_n \cdot \varphi\). 
By Lemma~\ref{lem:ulq-independence}, \(\w\setminus\{\nl_1,\ldots,\nl_n\} \models \forall \l_1,\ldots,\l_n \cdot \varphi\)
and hence \(\w\setminus\alpha \models \forall \l_1,\ldots,\l_n \cdot \varphi\).
Since \(\alpha \subseteq \w\), we get that \(\w\) is an \(\alpha\)-independent model of \(\forall \l_1,\ldots,\l_n \cdot \varphi\).
By Theorem~\ref{theo:ulq-semantics-b}, \(\forall \l_1,\ldots,\l_n \cdot \varphi \models \varphi\).
By Lemma~\ref{lem:imodel-s2w} and
\(\w \models \forall \l_1,\ldots,\l_n \cdot \varphi \models \varphi\),
we get that \(\w\) is an \(\alpha\)-independent model of \(\varphi\).

(\(\bthen\)) Suppose \(\w\) is an \(\alpha\)-independent model of \(\varphi\).
By Lemma~\ref{lem:imodel-w2s}, \(\w\) is an \(\alpha\)-independent model of \(\forall \l_1,\ldots,\l_n \cdot \varphi\)
and hence \(\w \models \forall \l_1,\ldots,\l_n \cdot \varphi\).
\end{proof}

\begin{proof}[\bf Proof of Proposition~\ref{prop:imodel-brule}]
Let \(\w = \alpha,\beta\) be a model of formula \(\varphi\) where \(\alpha\) and \(\beta\) are disjoint terms.

(\(\then\))~Suppose \(\w\) is an \(\alpha\)-independent model of \(\varphi\).
Then \(\beta \models \varphi\). Moreover, for every world \(\beta,\gamma,\l\) (where \(\beta\), \(\gamma\) and \(\l\) are disjoint) we must have
\(\beta,\gamma,\l \models \varphi\) and \(\beta,\gamma,\nl  \models \varphi\).
By Definition~\ref{def:brule}, \(\varphi\) cannot then have a b-rule of the form \(\Rule({\beta,\gamma},\l)\) or \(\Rule({\beta,\gamma},\nl)\).

(\(\bthen\))~Suppose \(\varphi\) has no b-rules of the form \(\Rule({\beta,\gamma},\l)\).
By Definition~\ref{def:brule}, for every world of the form \(\beta,\gamma,\l\), 
we have \(\beta,\gamma,\l \models \varphi\) only if \(\beta,\gamma,\nl \models \varphi\).
Any world \(\w^\star \supseteq \beta\) can be obtained from world \(\w=\beta,\alpha\) by a sequence of single flips to the variables of \(\alpha\).
Hence, any such world \(\w^\star\) is a model of \(\varphi\), which implies
\(\beta \models \varphi\). Therefore, \(\w\) must be an \(\alpha\)-independent model of \(\varphi\).
\end{proof}

%%%%%%%%%%%%%%
%% Existential
%%%%%%%%%%%%%%

\begin{proof}[\bf Proof of Theorem~\ref{theo:lq-duality}]
We have \(\neg(\forall \l \cdot \neg \varphi)
 = \neg((\l \vee (\neg \varphi \cd \nl)) \wedge (\neg \varphi \cd \l)) 
 =  (\nl \wedge (\varphi \cd \nl)) \vee (\varphi \cd \l) 
=  \exists \l \cdot \varphi.\)
We can similarly show \(\forall \l \cdot \varphi = \neg(\exists \l \cdot \neg \varphi)\).
\end{proof}

\begin{proof}[\bf Proof of Proposition~\ref{prop:elq-semantics-old}]
See Proposition 16 of~\cite{LangLM03}.
\end{proof}

\begin{proof}[\bf Proof of Proposition~\ref{prop:elq-order}]
See footnote 4 in~\cite{LangLM03}.
\end{proof}

\begin{proof}[\bf Proof of Proposition~\ref{prop:elq-evq}]
See Proposition 20 of~\cite{LangLM03}.
\end{proof}

\begin{proof}[\bf Proof of Theorem~\ref{theo:elq-semantics-b}]
By Proposition~\ref{prop:elq-semantics-old}, $M(\varphi) \subseteq M(\exists \l \cdot \varphi)$.
By Theorem~\ref{theo:lq-duality} (duality), $\w \in M(\exists \l \cdot \varphi)$ iff $\w \not \in M(\forall \l \cdot \nvarphi)$.
Moreover, $\w \not \in M(\varphi)$  iff $\w \in M(\nvarphi)$. 
Thus, $\w \in M(\exists \l \cdot \varphi)$ and $\w \not \in M(\varphi)$ iff $\w \not \in M(\forall \l \cdot \nvarphi)$ and $\w \in M(\nvarphi)$
iff $\w$ is an $\nl$-boundary model of $\nvarphi$ (by second part of Theorem~\ref{theo:ulq-semantics-b}).
\end{proof}

\begin{proof}[\bf Proof of Theorem~\ref{theo:elq-syntax}]
By Theorem~\ref{theo:ulq-syntax}, $\forall \l . \overline{\varphi}=\overline{\varphi} \wedge \bigwedge_{\alpha \rightarrow \nl \in \RS(\overline{\varphi})}\overline{\alpha}$. Negating the two sides, we get $\neg (\forall \l . \overline{\varphi})=\varphi \vee \bigvee_{\alpha \rightarrow \nl \in \RS(\overline{\varphi})} \alpha$. 
By Theorem~\ref{theo:lq-duality} (duality), $\neg (\forall \l . \overline{\varphi})=\exists \l . \varphi$.
By Proposition~\ref{prop:rconnect}, $\alpha \rightarrow \nl \in \RS(\overline{\varphi})$ is equivalent to $\alpha \rightarrow \l \in \RS(\varphi)$.
This concludes the proof.
\end{proof}

\begin{proof}[\bf Proof of Proposition~\ref{prop:elq-preserve}]
Suppose $\varphi \models \beta$ and $\nl \in \beta$.
Let $\gamma$ be the term containing the complements of literals in clause $\beta$ (\(\gamma = \overline{\beta}\)). 
Then $\gamma \models \nvarphi$ and $\l \in \gamma$. 
Moreover, $\gamma \models \forall \l \cdot \nvarphi$ by Proposition~\ref{prop:ulq-preserve},
and  $\exists \l \cdot \varphi \models \beta$ by contraposition and Theorem~\ref{theo:lq-duality} (duality).
\end{proof}

\begin{proof}[\bf Proof of Proposition~\ref{prop:elq-imply}]
From $\varphi \models \phi$, we get $\overline{\phi} \models \overline{\varphi}$. 
By Proposition~\ref{prop:ulq-imply}, $\forall \l \cdot \overline{\phi} \models \forall \l \cdot \overline{\varphi}$.
Finally, \(\exists \l \cdot \varphi \models \exists \l \cdot \phi\) by contraposition and Theorem~\ref{theo:lq-duality} (duality).
\end{proof}

%\subsection*{Section~\ref{sec:tractable}}

\begin{proof}[\bf Proof of Proposition~\ref{prop:quantify-base}]
The results for $\top$/$\bot$ follow directly from Definitions~\ref{def:ulq} and~\ref{def:elq}.

Using Definition \ref{def:elq}, $\exists \l_1 . \l_2 = (\l_2 \cd \l_1) \vee (\nl_1 \wedge (\l_2 \cd \nl_1))$.
If $\l_1 = \l_2$, we have $\l_2 \cd \l_1 = \top$ and then $\exists \l_1 . \l_2 = \top$.
If $\l_1 \neq \l_2$, then either $\l_1 = \nl_2$ or $\l_1 \neq \nl_2$.
If $\l_1 = \nl_2$, then $\l_2 \cd \l_1 = \bot$ and $\l_2 \cd \nl_1 = \top$ leading to $\exists \l_1 . \l_2 = \nl_1 = \l_2$.
If $\l_1 \neq \nl_2$, then $\l_2 \cd \l_1 = \l_2 \cd \nl_1 = \l_2$ and hence $\exists \l_1 . \l_2 = \l_2 \vee (\nl_1 \wedge \l_2) = \l_2$.

The results for $\forall \l_1 . \l_2$ follow from the results for $\exists \l_1 . \l_2$ using Theorem~\ref{theo:lq-duality} (duality).
\end{proof}

\begin{proof}[\bf Proof of Proposition~\ref{prop:quantify-compound}]~
\begin{itemize}

\item [(a)] By Definition~\ref{def:elq}, we have
\begin{eqnarray*}
\exists \l (\alpha \vee \beta) 
& = & ((\alpha \vee \beta)|\l)  \vee (\nl \wedge ((\alpha \vee \beta)|\nl)) \\
& = & (\alpha|\l) \vee (\beta|\l)  \vee (\nl \wedge ((\alpha|\nl) \vee (\beta|\nl))) \\
& = & (\alpha|\l) \vee (\beta|\l)  \vee (\nl \wedge (\alpha|\nl)) \vee (\nl \wedge (\beta|\nl)) \\
& = & (\exists \l . \alpha) \vee (\exists \l . \beta)
\end{eqnarray*}
%This is also a consequence of Proposition~17 in~\cite{LangLM03}.

\item [(b)] Follows from Part~(a) and Theorem~\ref{theo:lq-duality} (duality).

\item [(c)] Suppose literals \(\l\) and \(\nl\) do not appear in \(\beta\).
By Definition~\ref{def:elq}, \(\exists \l . \beta = \beta\) and also:
\begin{eqnarray*}
\exists \l (\alpha \wedge \beta) 
& = & ((\alpha \wedge \beta)|\l)  \vee (\nl \wedge ((\alpha \wedge \beta)|\nl)) \\
& = & ((\alpha|\l) \wedge \beta)  \vee (\nl \wedge (\alpha|\nl) \wedge \beta)\\
& = & ((\alpha|\l)  \vee (\nl \wedge (\alpha|\nl))) \wedge \beta \\
& = & (\exists \l . \alpha) \wedge (\exists \l . \beta)
\end{eqnarray*}

\item [(d)] Follows from Part~(c) and Theorem~\ref{theo:lq-duality} (duality). \qedhere
\end{itemize}
\end{proof}

\begin{proof}[\bf Proof of Proposition~\ref{prop:q-cnf}]
The first statement follows directly from Propositions~\ref{prop:quantify-base} and~\ref{prop:quantify-compound}(b,d).
By Proposition~\ref{prop:quantify-compound}(b), universal literal quantification distributes over conjuncts (clauses).
By Proposition~\ref{prop:quantify-compound}(d), it also distributes over disjuncts (literals) when they do not share variables (the literals
of a clause are over distinct variables). Hence, we just need to replace each literal \(\l'\) in the CNF with \(\forall \l \cdot \l'\).
By Proposition~\ref{prop:quantify-base}, \(\forall \l \cdot \l'=\bot\) if \(\nl = \l'\) and \(\forall \l \cdot \l'=\l'\).

As to the second statement, for literal $\l$, the clauses of CNF $\Delta$ can be partitioned into:
\begin{eqnarray*}
\Delta_a & = & \{\alpha \mid \alpha \in \Delta \mbox{  and } \l \in \alpha\} \\
\Delta_b & = & \{\alpha \mid \alpha \in \Delta \mbox{ and } \nl \in \alpha\} \\
\Delta_c & = & \{\alpha \mid \alpha \in \Delta \mbox{ and } \l \not \in \alpha, \nl \not \in \alpha\} 
\end{eqnarray*}
By Definition \ref{def:elq},
$\exists \l . \Delta = (\Delta \cd \l) \vee (\nl \wedge (\Delta\cd\nl)) = (\nl \vee (\Delta \cd \l)) \wedge ((\Delta \cd \l) \vee (\Delta \cd \nl))$
so we get CNFs:
\begin{eqnarray*}
\Delta \cd \l       & = & \{\alpha\setminus \{\nl\} \mid \alpha \in \Delta_b\} \cup \Delta_c \\
\Delta \cd \nl      & = & \{\alpha\setminus \{\l\}   \mid \alpha \in \Delta_a\} \cup \Delta_c \\
\Delta_A            & = & \{\nl \vee \alpha \mid \alpha \in \Delta \cd \l\}  \mbox{ represents \(\nl \vee (\Delta \cd \l)\)} \\
\Delta_B            & = & \{\alpha \vee \beta \mid \alpha \in \Delta \cd \l \mbox{ and } \beta \in \Delta \cd \nl\} \mbox{ represents \((\Delta \cd \l)\wedge(\Delta \cd \nl)\)} \\
\exists \l . \Delta & = & \Delta_A \cup \Delta_B 
\end{eqnarray*}
CNF \(\Delta_B\) does not contain the variable of literal $\l$ and can be expressed as follows:
\begin{eqnarray*}
\Delta_B & = & \Delta_c \cup \Delta_d \\
\Delta_d & =  & \{(\alpha\setminus\{\l\}) \cup (\beta\setminus\{\nl\}) \mid \alpha \in \Delta, \beta \in \Delta, \l \in \alpha \mbox{ and } \nl \in \beta\}
\end{eqnarray*}
where \(\Delta_d\) contains the resolvents of \(\Delta\) on the variable of literal \(\l\).
We can express CNF \(\Delta_A\) as: 
\[
\Delta_A  =
\{\nl \vee (\alpha\setminus\{\nl\}) \mid \alpha \in \Delta_b\} \cup
 \{\nl \vee \alpha \mid \alpha \in \Delta_c \}.
\]
The first part is equivalent to \(\Delta_b\) and the second part is subsumed by \(\Delta_c\) so we now have:
\[
\exists \l . \Delta = \Delta_A \cup \Delta_B = \Delta_b \cup \Delta_c \cup \Delta_d
\]
That is, \(\exists \l . \Delta\) consists of all clauses of \(\Delta\) that do not contain literal \(\l\) (\(\Delta_b\) and \(\Delta_c\)) in
addition to all resolvents of \(\Delta\) on the variable \(X\) of literal \(\l\) (\(\Delta_d\)).
When $\Delta$ is closed under resolution on variable \(X\), we have \(\Delta_d \subseteq \Delta_c\). In this case, 
$\exists \l \cdot \Delta$ can be obtained from $\Delta$ by keeping only its clauses \(\Delta_b\) and \(\Delta_c\), 
which is equivalent to removing clauses \(\Delta_a\) (containing literal \(\l\)) as claimed by the theorem.
\end{proof}

\begin{proof}[\bf Proof of Corollary~\ref{coro:q-cnf}]
Follows directly from Proposition~\ref{prop:q-cnf} given that a CNF $\Delta$ in prime implicate form
is such that the resolvent of any two clauses of $\Delta$ is subsumed by a clause of $\Delta$~\cite{Quine55}.
See also Proposition~19 in~\cite{LangLM03}.
\end{proof}

\begin{proof}[\bf Proof of Proposition~\ref{prop:q-dnf}]
Follows directly from Proposition~\ref{prop:q-cnf} and Theorem~\ref{theo:lq-duality} (duality).
\end{proof}

\begin{proof}[\bf Proof of Corollary~\ref{cor:q-dnf}]
Follows directly from Proposition~\ref{prop:q-dnf} given that a DNF $\Delta$ in prime implicant form
is such that the consensus of any two terms of $\Delta$ is subsumed by a term of $\Delta$~\cite{Quine55}.
\end{proof}

\begin{proof}[\bf Proof of Proposition~\ref{prop:elq-D-DNNF}]
By Propositions~\ref{prop:quantify-base} and~\ref{prop:quantify-compound}(a,c), one can 
existentially quantify literals from a Decision-DNNF circuit in linear time since conjuncts in these circuits 
do not share variables. The result is a DNNF circuit as one would only be replacing 
some literals with \(\top,\) therefore preserving the decomposability property.
\end{proof}

\begin{proof}[\bf Proof of Proposition~\ref{prop:ulq-D-DNNF-transform}]
(1) \((\l \vee \beta) \wedge (\nl \vee \alpha)\) is equivalent to \((\l \wedge \alpha) \vee (\nl \wedge \beta) \vee (\alpha\wedge\beta)\),
which is equivalent to \((\l \wedge \alpha) \vee (\nl \wedge \beta)\) since \(\alpha\wedge\beta\) is subsumed by 
 \((\l \wedge \alpha) \vee (\nl \wedge \beta)\). 
(2) \(\Gamma\) can be obtained from \(\Delta\) by flipping some \(\vee\) to \(\wedge\) and vice versa.
(3) Every disjunction in \(\Gamma\) appears in a fragment \((\l \vee \beta) \wedge (\nl \vee \alpha)\). 
By definition of a Decision-DNNF circuit, we know that \(\l\) shares no variables with \(\beta\), and \(\nl\) share no variables with \(\alpha\).
\end{proof}

\begin{proof}[\bf Proof of Proposition~\ref{prop:ulq-D-DNNF}]
Using Proposition~\ref{prop:ulq-D-DNNF-transform}, we can in linear time transform a Decision-DNNF circuit into an NNF 
circuit in which disjuncts do not share variables. Using Propositions~\ref{prop:quantify-base} and~\ref{prop:quantify-compound}(b,d),
we can then universally quantify literals in linear time, leading to an NNF circuit since we only replace some
literals with constants. 
\end{proof}

\begin{proof}[\bf Proof of Proposition~\ref{prop:elq-SDD}]
By Propositions~\ref{prop:quantify-base} and~\ref{prop:quantify-compound}(a,c), one can 
existentially quantify literals from an SDD circuit in linear time since conjuncts in these circuits 
do not share variables. The result is guaranteed to be a DNNF circuit as one would only be replacing 
some literals with \(\top\), therefore preserving the decomposability property.
\end{proof}

\begin{proof}[\bf Proof of Proposition~\ref{prop:sdd-transform}]~
\begin{itemize}
\item[(1)] We first observe that primes \(p_1, \ldots, p_n\) form a partition. Let \(N = \{1,\ldots,n\}\).
Then \(\bigvee_{i \in N} p_i = \top\). Moreover, for \(S \subseteq N\), \(\bigvee_{i \in S} p_i = \bigwedge_{i \in N\setminus S} \neg p_i\). 
We now have:
\begin{eqnarray*}
(\neg p_1 \vee s_1) \wedge \ldots \wedge (\neg p_n \vee s_n)
& = & \bigvee_{S \subseteq N}  \left[ \left[\bigwedge_{i \in N\setminus S} \neg p_i \right] \wedge \left[ \bigwedge_{j \in S} s_j \right] \right] \\
& = & \bigvee_{S \subseteq N} \left[\left[\bigvee_{i \in S} p_i\right] \wedge \left[\bigwedge_{j \in S} s_j \right]\right].
\end{eqnarray*}
We will consider the above disjuncts according to the set \(S\). When \(S = \{\}\), the disjunct is \(\bot\). When \(S = \{i\}\), the
disjunct is \((p_i \wedge s_i)\). When \(S=N\), the disjunct is \((s_1 \wedge \ldots \wedge s_n)\). Otherwise, \(1 < |S| < N\) and the disjunct
is equivalent to \(\bigvee_{i \in S} (p_i \wedge \bigwedge_{j \in S} s_j)\). Each term \((p_i \wedge \bigwedge_{j \in S} s_j)\) 
is subsumed by the disjunct \((p_i \wedge s_i)\) generated by \(S=\{i\}\). Moreover, the term \((s_1 \wedge \ldots \wedge s_n)\)
is subsumed by \((p_1 \wedge s_1) \vee \ldots \vee (p_n \wedge s_n)\) since primes \(p_i\) form a partition.
Hence, \((\neg p_1 \vee s_1) \wedge \ldots \wedge (\neg p_n \vee s_n) = (p_1 \wedge s_1) \vee \ldots \vee (p_n \wedge s_n)\).

\item[(2)] An SDD circuit is an NNF circuit. We can construct a negation for each node in an NNF circuit while at most doubling the size of the 
NNF circuit, a process that can be done in time linear in the NNF circuit size. This can be done by traversing the NNF circuit bottom, while constructing a negation for each encountered node. The process is trivial for constants and literals. For a disjunction \(\alpha_1 \vee \ldots \vee \alpha_n\), the
negation is \(\neg \alpha_1 \wedge \ldots \wedge \neg \alpha_n\) and we already have nodes for all \(\neg \alpha_i\). A similar 
process is applied to conjunctions. We can therefore replace each fragment
\((p_1 \wedge s_1) \vee \ldots \vee (p_n \wedge s_n)\)  by the fragment \((\neg p_1 \vee s_1) \wedge \ldots \wedge (\neg p_n \vee s_n)\)
in time linear in the size of SDD circuit. 

\item[(3)] Every disjunction in \(\Gamma\) appears in a fragment \((\neg p_1 \vee s_1) \wedge \ldots \wedge (\neg p_n \vee s_n)\). 
By definition of SDDs, prime \(p_i\) shares no variables with sub \(s_i\) so \(\neg p_i\) shares no variables with \(s_i\). \qedhere
\end{itemize}
\end{proof}

\begin{proof}[\bf Proof of Proposition~\ref{prop:ulq-SDD}]
Using Proposition~\ref{prop:sdd-transform}, we can in linear time transform an SDD circuit into an NNF 
circuit in which disjuncts do not share variables. Using Propositions~\ref{prop:quantify-base} and~\ref{prop:quantify-compound}(b,d),
we can then universally quantify literals in linear time, leading to an NNF circuit since we only replace some
literals with constants. 
\end{proof}

%\subsection*{Section~\ref{sec:XAI}}

\begin{proof}[\bf Proof of Theorem~\ref{theo:ifeature}]
Given Definition~\ref{def:ifeature}, 
we need to show \(\Delta(\delta) = \Delta(\erase{\delta}{X_1,\ldots,X_n})\) iff \(\delta \models \forall X_1, \ldots, X_n \cdot~\Delta_\delta\).
We next show both directions of the theorem.

(\(\then\)) Suppose \(\Delta(\delta) = \Delta(\erase{\delta}{X_1,\ldots,X_n})\). 
Then \(\erase{\delta}{X_1,\ldots,X_n} \models \Delta_\delta\). 
By Propositions~\ref{prop:ulq-imply} and~\ref{prop:ulq-uvq}, we have 
\(\forall X_1,\ldots,X_n \cdot \erase{\delta}{X_1,\ldots,X_n} \models  \forall X_1, \ldots, X_n \cdot~\Delta_\delta\).
Since variables \(X_1,\ldots,X_n\) are not mentioned in \(\erase{\delta}{X_1,\ldots,X_n}\), we also have
\(\erase{\delta}{X_1,\ldots,X_n} \models  \forall X_1, \ldots, X_n \cdot~\Delta_\delta\).
Since \(\delta \models \erase{\delta}{X_1,\ldots,X_n}\), we finally get
\(\delta \models \forall X_1, \ldots, X_n \cdot~\Delta_\delta\).

(\(\bthen\)) Suppose \(\delta \models \forall X_1, \ldots, X_n \cdot~\Delta_\delta\). 
Then \(\exists X_1, \ldots, X_n \cdot \delta \models \exists X_1, \ldots, X_n (\forall X_1, \ldots, X_n \cdot~\Delta_\delta)\) by Proposition~\ref{prop:elq-imply}.
Moreover, \(\forall X_1, \ldots, X_n \cdot~\Delta_\delta\) is independent of variables \(X_1, \ldots, X_n\) by Proposition~\ref{prop:evq},
so we have \(\exists X_1, \ldots, X_n (\forall X_1, \ldots, X_n \cdot~\Delta_\delta) = \forall X_1, \ldots, X_n \cdot~\Delta_\delta\). 
Since  \(\exists X_1, \ldots, X_n \cdot \delta = \erase{\delta}{X_1,\ldots,X_n}\), we get
\(\erase{\delta}{X_1,\ldots,X_n} \models \forall X_1, \ldots, X_n \cdot~\Delta_\delta\).
Since \(\forall X_1, \ldots, X_n \cdot~\Delta_\delta \models \Delta_\delta\), we have 
\(\erase{\delta}{X_1,\ldots,X_n} \models \Delta_\delta\).
Hence, \(\Delta(\delta) = \Delta(\erase{\delta}{X_1,\ldots,X_n})\).
\end{proof}

\begin{proof}[\bf Proof of Theorem~\ref{theo:icharacteristic}]
Let \(\alpha = \delta \cap \{\nl_1, \ldots, \nl_n\}\).
It suffices to show that decision \(\Delta(\delta)\) is independent of characteristics \(\nl_1, \ldots, \nl_n\)
iff \(\delta\) is an \(\alpha\)-independent model of \(\Delta_\delta\). If this holds,
Theorem~\ref{theo:icharacteristic} will then follow directly from Theorem~\ref{theo:ulq-semantics}.

(\(\then\))~Suppose decision \(\Delta(\delta)\) is independent of characteristics \(\nl_1, \ldots, \nl_n\).
By Definition~\ref{def:icharacteristic}, \(\Delta(\delta) = \Delta(\delta\setminus\{\nl_1, \ldots, \nl_n\})\) and hence
\(\delta\setminus\{\nl_1, \ldots, \nl_n\} \models \Delta_\delta\). By definition of \(\alpha\), we now have
\(\delta\setminus\alpha \models \Delta_\delta\). Since \(\alpha \subseteq \delta\), then
\(\delta\) is an \(\alpha\)-independent model of \(\Delta_\delta\).

(\(\bthen\))~Suppose \(\delta\) is an \(\alpha\)-independent model of \(\Delta_\delta\).
Then \(\delta\setminus\alpha \models \Delta_\delta\) and hence \(\delta\setminus\{\nl_1, \ldots, \nl_n\} \models \Delta_\delta\). 
By definition, \(\delta \models \Delta_\delta\) so \(\Delta(\delta) = \Delta(\delta\setminus\{\nl_1, \ldots, \nl_n\})\). 
By Definition~\ref{def:icharacteristic}, decision \(\Delta(\delta)\) is independent of characteristics \(\nl_1, \ldots, \nl_n\).
\end{proof}

\begin{proof}[\bf Proof of Proposition~\ref{prop:ifc}]
By Proposition~\ref{prop:ulq-uvq}, \(\forall X_1, \ldots, X_n \cdot \Delta = \forall \l_1,\nl_1, \ldots, \l_n,\nl_n \cdot \Delta\).
By Proposition~\ref{prop:ulq-order}, \(\forall \l_1,\nl_1, \ldots, \l_n,\nl_n \cdot \Delta = \forall \nl_1,\ldots,\nl_n (\forall \l_1, \ldots,\l_n \cdot \Delta)\).
By Theorem~\ref{theo:ulq-semantics-b}, \(\forall \nl_1,\ldots,\nl_n (\forall \l_1, \ldots,\l_n \cdot \Delta) \models \forall \l_1, \ldots,\l_n \cdot \Delta\).
Hence, \(\forall X_1, \ldots, X_n \cdot \Delta \models \forall \l_1, \ldots, \l_n \cdot \Delta\).
\end{proof}

\begin{proof}[\bf Proof of Theorem~\ref{theo:reasons}]
Let \(\gamma = \{\l_1, \ldots, \l_m\}\), \(X_{m+1},  \ldots, X_n\) be all classifier features not mentioned in \(\gamma\)
and let \(\Gamma\) be the complete reason \(\forall \l_1, \ldots, \l_m, X_{m+1}, \ldots, X_n \cdot \Delta_\gamma\).
By definition, \(\Delta(\gamma^\star)=\Delta(\gamma)\) is equivalent to \(\gamma^\star \models \Delta_\gamma\).
Moreover, \(\Gamma \models \Delta_\gamma\) by Proposition~\ref{prop:uvq} and Theorem~\ref{theo:ulq-semantics-b}.
We next prove both directions of the theorem.

(\(\then\))~Suppose \(\gamma^\star\) is a sufficient reason for decision \(\Delta(\gamma)\).
By definition, \(\gamma^\star\) is a prime implicant of the complete reason \(\Gamma\):
\(\gamma^\star \models \Gamma\) and no strict subset of \(\gamma^\star\) satisfies this property.
We now have \(\gamma^\star \models \Gamma \models \Delta_\gamma\) and
need to show that no strict subset of \(\gamma^\star\) satisfies \(\gamma^\star \models \Delta_\gamma\).
Suppose to the contrary: \(\alpha \subset \gamma^\star\) and \(\alpha \models \Delta_\gamma\). 
By Propositions~\ref{prop:ulq-imply} and~\ref{prop:ulq-uvq}, 
we have \(\forall \l_1, \ldots, \l_m, X_{m+1}, \ldots, X_n \cdot  \alpha \models \forall \l_1, \ldots, \l_m, X_{m+1}, \ldots, X_n \cdot  \Delta_\gamma\).
Since variables \(X_{m+1}, \ldots, X_n\) do not appear in \(\alpha\) and \(\alpha \subseteq \{\l_1, \ldots, \l_m\}\),
we have \(\forall \l_1, \ldots, \l_m, X_{m+1}, \ldots, X_n \cdot  \alpha = \alpha\) and then
\(\alpha \models \forall \l_1, \ldots, \l_m, X_{m+1}, \ldots, X_n \cdot  \Delta_\gamma\).
Since \(\alpha \subset \gamma^\star\), then \(\gamma^\star\) cannot be a prime implicant of \(\Gamma\) which is a contradiction. 
Hence, \(\gamma^\star\) is a minimal subset of \(\gamma\) that satisfies
\(\gamma^\star \models \Delta_\gamma\) (and \(\Delta(\gamma^\star)=\Delta(\gamma)\)).

(\(\bthen\))~Suppose \(\gamma^\star\) is a minimal subset of \(\gamma\) that satisfies \(\Delta(\gamma^\star)=\Delta(\gamma)\).
Then \(\gamma^\star\) is a prime implicant of \(\Delta_\gamma\).
We need to show that \(\gamma^\star\) is a sufficient reason for decision \(\Delta(\gamma)\) which by definition is equivalent
to \(\gamma^\star\) being a prime implicant of the complete reason \(\Gamma\). Since \(\gamma^\star \models \Delta_\gamma\), we can
use Propositions~\ref{prop:ulq-imply} and~\ref{prop:ulq-uvq} as in the first part to 
get \(\gamma^\star \models \Gamma\) (\(\gamma^\star\) is an implicant of \(\Gamma\)).
Since \(\gamma^\star \models \Gamma \models \Delta_\gamma\), then \(\gamma^\star\) must be a prime
implicant of \(\Gamma\), otherwise it cannot be a prime implicant of \(\Delta_\gamma\).
\end{proof}

\begin{proof}[\bf Proof of Theorem~\ref{theo:bias}]
By Theorem~\ref{theo:ifeature}, \(\forall X_1, \ldots, X_n \cdot \Delta\) characterizes all positive instances with
decisions that are independent of protected features \(X_1, \ldots, X_n\); that is, positive instances with unbiased decisions.
Therefore, \(\Delta \wedge \neg (\forall X_1, \ldots, X_n \cdot \Delta)\) characterizes all positive instances
with biased decisions. One can similarly show the second part of the theorem, which characterizes negative instances with biased decisions.
\end{proof}

%%%%%%%%%%%%%%%%
%% b-rule dynamics
%%%%%%%%%%%%%%%%

\begin{proof}[\bf Proof of Theorem~\ref{theo:brules-del-keep}] The following proof invokes Definition~\ref{def:brule} frequently.
\begin{itemize}
\item [(a)]~Suppose \(\Rule(\alpha,\l_i) \in \RS(\varphi)\). Then \(\alpha,\l_i \models \varphi\) and \(\alpha,\nl_i \not \models \varphi\).
We have \(\alpha,\l_i \models \forall \l_i \varphi\) by Proposition~\ref{prop:ulq-preserve} and 
\(\alpha,\nl_i \not \models \forall \l_i \cdot \varphi\) by Theorem~\ref{theo:ulq-semantics-b}. 
Hence, \(\Rule(\alpha,\l_i) \in \RS(\forall \l_i \cdot \varphi)\).
\item[(b)]~Suppose \(\Rule(\alpha,\nl_i) \in \RS(\forall \l_i \cdot \varphi)\). 
Then \(\alpha,\nl_i \models \forall \l_i \cdot \varphi\) and \(\alpha,\l_i \not \models \forall \l_i \cdot \varphi\). 
Since \(\alpha,\nl_i \models \forall \l_i \cdot \varphi\), we get \(\alpha,\nl_i \models \varphi\) by Theorem~\ref{theo:ulq-semantics-b}. 
Consider two cases: \(\alpha,\l_i \models \varphi\) or \(\alpha,\l_i \not \models \varphi\). The first case is impossible
since it implies \(\alpha,\l_i \models \forall \l_i \cdot \varphi\) by Proposition~\ref{prop:ulq-preserve}
so we must have \(\alpha,\l_i \not \models \varphi\). 
Since \(\alpha,\nl_i \models \varphi\), then \(\alpha,\nl_i\) is an \(\nl_i\)-boundary model for \(\varphi\)
so it cannot be a model of \(\forall \l_i \cdot \varphi\) by Theorem~\ref{theo:ulq-semantics-b}, which is a contradiction.
We must therefore have \(\Rule(\alpha,\nl_i) \not \in \RS(\forall \l_i \cdot \varphi)\). 
\item [(c)]~Similar to Part~(a).
\item [(d)]~Suppose \(\Rule({\alpha,\nl_i},\l_j) \in \RS(\varphi)\). 
Then \(\alpha,\nl_i,\nl_j \not \models \varphi\) and 
hence \(\alpha,\nl_i,\nl_j \not \models \forall \l_i \cdot \varphi\) by Theorem~\ref{theo:ulq-semantics-b}. \\
\noindent (\(\then\)) Suppose \(\Rule({\alpha,\nl_i},\l_j) \in \RS(\forall \l_i \cdot \varphi)\).
Then \(\alpha,\nl_i,\l_j \models \forall \l_i \cdot \varphi\) and \(\alpha,\nl_i,\nl_j \not \models \forall \l_i \cdot \varphi\).
To show \(\Rule({\alpha,\l_j},\nl_i) \not \in \RS(\varphi)\), it suffices to show \(\alpha,\l_i,\l_j \models \varphi\).
Suppose \(\alpha,\l_i,\l_j \not \models \varphi\). Then \(\alpha,\l_i,\l_j \not \models \forall \l_i \cdot \varphi\) by Theorem~\ref{theo:ulq-semantics-b}.
Since \(\alpha,\nl_i,\l_j \models \forall \l_i \cdot \varphi\), 
we have \(\Rule({\alpha,\l_j},\nl_i) \in \RS(\forall \l_i \cdot \varphi)\) which contradicts~(b). 
Hence, \(\alpha,\l_i,\l_j \models \varphi\) and \(\Rule({\alpha,\l_j},\nl_i) \not \in \RS(\varphi)\). \\
\noindent (\(\bthen\)) Suppose \(\Rule({\alpha,\l_j},\nl_i) \not \in \RS(\varphi)\).
To show \(\Rule({\alpha,\nl_i},\l_j) \in \RS(\forall \l_i \cdot \varphi)\),  
we need to show that we have
\(\alpha,\nl_i,\l_j \models \forall \l_i \cdot \varphi\) and \(\alpha,\nl_i,\nl_j \not \models \forall \l_i \cdot \varphi\).
The latter follows from the supposition \(\Rule({\alpha,\nl_i},\l_j) \in \RS(\varphi)\).
The former also holds since supposition \(\Rule({\alpha,\l_j},\nl_i) \not \in \RS(\varphi)\) implies that 
\(\alpha,\nl_i,\l_j\) is not an \(\nl_i\)-boundary model for \(\varphi\) (Definition~\ref{def:brule}) 
so it is not dropped when universally quantifying literal \(\l_i\) (Theorem~\ref{theo:ulq-semantics-b}).  
Hence, \(\Rule({\alpha,\nl_i},\l_j) \in \RS(\forall \l_i \cdot \varphi)\). \qedhere
\end{itemize}
\end{proof}

\begin{proof}[\bf Proof of Theorem~\ref{theo:brules-add1}]
Let b-rule \(r=\Rule(\beta,\l_k)\), world \(\w = \beta,\l_k\) and suppose \(r \not \in \RS(\varphi)\) and that \(r \in \RS(\forall \l_i \cdot \varphi)\).
We next show that \(k \neq i\) and \(\nl_i \in \beta\), which is sufficient to prove the theorem.
By Theorem~\ref{theo:brules-del-keep}(b), \(\l_k \neq \nl_i\).
By \(r \in \RS(\forall \l_i \cdot \varphi)\) and Definition~\ref{def:brule}, 
\(\w \in \MS(\forall \l_i \cdot \varphi)\) and \(\w[\nl_k] \not \in \MS(\forall \l_i \cdot \varphi)\).
By \(\w \in \MS(\forall \l_i \cdot \varphi)\) and Theorem~\ref{theo:ulq-semantics-b}, \(\w \in \MS(\varphi)\).
By \(\w \in \MS(\varphi)\), \(r \not \in \RS(\varphi)\) and Definition~\ref{def:brule}, \(\w[\nl_k] \in \MS(\varphi)\).
By \(\w[\nl_k] \in \MS(\varphi)\), \(\w[\nl_k] \not \in \MS(\forall \l_i \cdot \varphi)\) and Theorem~\ref{theo:ulq-semantics-b},
world \(\w^\star = \w[\nl_k]\) must be an \(\nl_i\)-boundary model of \(\varphi\); that is, \(\nl_i \in \w^\star\)
and \(\w^\star[\l_i] \not \in \MS(\varphi)\).
%In summary, we now have \(\l_k \in \w\), \(\w \in \MS(\varphi)\), \(\w^\star \in \MS(\varphi)\) and \(\w^\star[\l_i] \not \in \MS(\varphi)\). 
To show \(i \neq k\), suppose the contrary \(i = k\).
Then \(\w^\star[\l_i] = (\w[\nl_k])[\l_i] = (\w[\nl_k])[\l_k] = \w[\l_k] = \w\).
This conflicts with \(\w \in \MS(\varphi)\) and \(\w^\star[\l_i] \not \in \MS(\varphi)\) so \(i \neq k\).
Since \(\nl_i \in \w^\star\) we get \(\nl_i \in \beta\).
\end{proof}

\begin{proof}[\bf Proof of Theorem~\ref{theo:brules-add2}]
Let world \(\w = \alpha,\nl_i,\l_j\) and let
\begin{enumerate}
\item[(A)] \(\Rule({\alpha,\l_i},\l_j) \in \RS(\varphi)\)
\item[(B)] \(\Rule({\alpha,\nl_j},\nl_i) \in \RS(\varphi)\)
\item[(C)] \(\Rule({\alpha,\l_j},\l_i) \not \in \RS(\varphi)\) 
\item[(D)] \(\Rule({\alpha,\nl_i},\nl_j) \not \in \RS(\varphi)\)
\end{enumerate}
(\(\then\))~Suppose \(r \not \in \RS(\varphi)\) and \(r \in \RS(\forall \l_i \cdot \varphi)\).
By \(r \in \RS(\forall \l_i \cdot \varphi)\) and Definition~\ref{def:brule}, 
\(\w \in \MS(\forall \l_i \cdot \varphi)\) and \(\w[\nl_j] \not \in \MS(\forall \l_i \cdot \varphi)\).
By \(\w \in \MS(\forall \l_i \cdot \varphi)\) and Theorem~\ref{theo:ulq-semantics-b}, \(\w \in \MS(\varphi)\).
By \(\w \in \MS(\varphi)\), \(r \not \in \RS(\varphi)\) and Definition~\ref{def:brule}, \(\w[\nl_j] \in \MS(\varphi)\).
By \(\w[\nl_j] \in \MS(\varphi)\), \(\w[\nl_j] \not \in \MS(\forall \l_i \cdot \varphi)\) and Theorem~\ref{theo:ulq-semantics-b},
world \(\w^\star = \w[\nl_j]\) must be an \(\nl_i\)-boundary model of \(\varphi\);  
that is, \(\nl_i  \in \w^\star\) and \(\w^\star[\l_i] \not \in \MS(\varphi)\).
By \(\w \in \MS(\varphi)\), \(\w \in \MS(\forall \l_i \cdot \varphi)\) and Theorem~\ref{theo:ulq-semantics-b},
\(\w\) is not an \(\nl_i\)-boundary model of \(\varphi\) and hence \(\w[\l_i] \in \MS(\varphi)\).
We now have
\begin{enumerate}
\item[(1)] \(\alpha,\nl_i,\l_j   =  \w \in \MS(\varphi)\),
\item[(2)] \(\alpha,\l_i,\l_j   = \w[\l_i] \in \MS(\varphi)\),
\item[(3)] \(\alpha,\nl_i,\nl_j = \w[\nl_j] = \w^\star \in \MS(\varphi)\) and
\item[(4)] \(\alpha,\l_i,\nl_j = \w^\star[\l_i] \not \in \MS(\varphi)\).
\end{enumerate}
By Definition~\ref{def:brule}, 
(2) and (4) imply (A);
(3) and (4) imply (B); 
(1) implies (C) and~(D).

(\(\bthen\))~Suppose (A), (B), (C) and (D).
By Definition~\ref{def:brule}, (A) implies (2) and (4); (B) implies (3) and (4); (C) and (2) imply (1); (D) and (3) imply (4).
We have now established (1), (2), (3) and (4)---we only need (A) and  (B) together with either (C) or (D) to establish this.
By (3) and Definition~\ref{def:brule}, we get \(r \not \in \RS(\varphi)\).
By (1) and (2), \(\alpha,\nl_i,\l_j\) is not an \(\nl_i\)-boundary model of \(\varphi\) and 
hence \(\alpha,\nl_i,\l_j \in \MS(\forall \l_i \cdot \varphi)\) by Theorem~\ref{theo:ulq-semantics-b}.
By (3) and (4), \(\alpha,\nl_i,\nl_j\) is an \(\nl_i\)-boundary model of \(\varphi\) and 
hence \(\alpha,\nl_i,\nl_j \not \in \MS(\forall \l_i \cdot \varphi)\) by Theorem~\ref{theo:ulq-semantics-b}.
By \(\w=\alpha,\nl_i,\l_j \in \MS(\forall \l_i \cdot \varphi)\), \(\alpha,\nl_i,\nl_j \not \in \MS(\forall \l_i \cdot \varphi)\) and
Definition~\ref{def:brule}, we get  \(r \in \RS(\forall \l_i \cdot \varphi)\).
\end{proof}

\vskip 0.2in
\bibliographystyle{theapa}
\bibliography{biblio,biblioExt}

\begin{thebibliography}{}

\bibitem[\protect\BCAY{Audemard, Koriche,\ \BBA\ Marquis}{Audemard
  et~al.}{2020}]{kr/AudemardKM20}
Audemard, G., Koriche, F., \BBA\ Marquis, P. \BBOP2020\BBCP.
\newblock \BBOQ On tractable {XAI} queries based on compiled
  representations\BBCQ\
\newblock In {\Bem Proc. of KR'20}, \BPGS\ 838--849.

\bibitem[\protect\BCAY{Beame\ \BBA\ Liew}{Beame\ \BBA\
  Liew}{2015}]{uai/BeameL15}
Beame, P.\BBACOMMA\  \BBA\ Liew, V. \BBOP2015\BBCP.
\newblock \BBOQ New limits for knowledge compilation and applications to exact
  model counting\BBCQ\
\newblock In {\Bem Proc. of {UAI}'15}, \BPGS\ 131--140. {UAI} Press.

\bibitem[\protect\BCAY{Bollig\ \BBA\ Buttkus}{Bollig\ \BBA\
  Buttkus}{2019}]{mst/BolligB19}
Bollig, B.\BBACOMMA\  \BBA\ Buttkus, M. \BBOP2019\BBCP.
\newblock \BBOQ On the relative succinctness of sentential decision
  diagrams\BBCQ\
\newblock {\Bem Theory Comput. Syst.}, {\Bem 63\/}(6), 1250--1277.

\bibitem[\protect\BCAY{Boole}{Boole}{1854}]{Boole1854}
Boole, G. \BBOP1854\BBCP.
\newblock {\Bem An investigation of the laws of thought}.
\newblock Walton and Maberley, London.

\bibitem[\protect\BCAY{Bryant}{Bryant}{1986}]{tc/Bryant86}
Bryant, R.~E. \BBOP1986\BBCP.
\newblock \BBOQ Graph-based algorithms for {B}oolean function
  manipulation\BBCQ\
\newblock {\Bem {IEEE} Trans. Computers}, {\Bem 35\/}(8), 677--691.

\bibitem[\protect\BCAY{Chan\ \BBA\ Darwiche}{Chan\ \BBA\
  Darwiche}{2003}]{ChanD03}
Chan, H.\BBACOMMA\  \BBA\ Darwiche, A. \BBOP2003\BBCP.
\newblock \BBOQ Reasoning about {B}ayesian network classifiers\BBCQ\
\newblock In {\Bem Proc. of {UAI}'03}, \BPGS\ 107--115.

\bibitem[\protect\BCAY{Choi, Shih, Goyanka,\ \BBA\ Darwiche}{Choi
  et~al.}{2020}]{ChoiShihGoyankaDarwiche20}
Choi, A., Shih, A., Goyanka, A., \BBA\ Darwiche, A. \BBOP2020\BBCP.
\newblock \BBOQ On symbolically encoding the behavior of random forests\BBCQ\
\newblock In {\Bem Proc. of FoMLAS'20, 3rd Workshop on Formal Methods for
  ML-Enabled Autonomous Systems}.

\bibitem[\protect\BCAY{Cimatti, Roveri,\ \BBA\ Traverso}{Cimatti
  et~al.}{1998}]{DBLP:conf/aips/CimattiRT98}
Cimatti, A., Roveri, M., \BBA\ Traverso, P. \BBOP1998\BBCP.
\newblock \BBOQ Strong planning in non-deterministic domains via model
  checking\BBCQ\
\newblock In {\Bem Proc. of AIPS'98}, \BPGS\ 36--43.

\bibitem[\protect\BCAY{Coste{-}Marquis, Berre, Letombe,\ \BBA\
  Marquis}{Coste{-}Marquis
  et~al.}{2006}]{DBLP:journals/jsat/Coste-MarquisBLM06}
Coste{-}Marquis, S., Berre, D.~L., Letombe, F., \BBA\ Marquis, P.
  \BBOP2006\BBCP.
\newblock \BBOQ Complexity results for quantified {B}oolean formulae based on
  complete propositional languages\BBCQ\
\newblock {\Bem J. Satisf. Boolean Model. Comput.}, {\Bem 1\/}(1), 61--88.

\bibitem[\protect\BCAY{Darwiche}{Darwiche}{2001}]{jacm/Darwiche01}
Darwiche, A. \BBOP2001\BBCP.
\newblock \BBOQ Decomposable negation normal form\BBCQ\
\newblock {\Bem J. {ACM}}, {\Bem 48\/}(4), 608--647.

\bibitem[\protect\BCAY{Darwiche}{Darwiche}{2011}]{ijcai/Darwiche11}
Darwiche, A. \BBOP2011\BBCP.
\newblock \BBOQ {SDD:} {A} new canonical representation of propositional
  knowledge bases\BBCQ\
\newblock In {\Bem Proc. of {IJCAI}'11}, \BPGS\ 819--826.

\bibitem[\protect\BCAY{Darwiche\ \BBA\ Hirth}{Darwiche\ \BBA\
  Hirth}{2020}]{DarwicheH20}
Darwiche, A.\BBACOMMA\  \BBA\ Hirth, A. \BBOP2020\BBCP.
\newblock \BBOQ On the reasons behind decisions\BBCQ\
\newblock In {\Bem Proc. of {ECAI}'20}, \BPGS\ 712--720.

\bibitem[\protect\BCAY{Darwiche\ \BBA\ Marquis}{Darwiche\ \BBA\
  Marquis}{2002}]{jair/DarwicheM02}
Darwiche, A.\BBACOMMA\  \BBA\ Marquis, P. \BBOP2002\BBCP.
\newblock \BBOQ A knowledge compilation map\BBCQ\
\newblock {\Bem J. Artif. Intell. Res.}, {\Bem 17}, 229--264.

\bibitem[\protect\BCAY{Doherty, {\L}ukasziewicz,\ \BBA\
  Madali\'nska-Bugaj}{Doherty et~al.}{1998}]{Dohertyetal98}
Doherty, P., {\L}ukasziewicz, W., \BBA\ Madali\'nska-Bugaj, E. \BBOP1998\BBCP.
\newblock \BBOQ The {PMA} and relativizing change for action update\BBCQ\
\newblock In {\Bem Proc. of {KR}'98}, \BPGS\ 258--269.

\bibitem[\protect\BCAY{Doherty, {\L}ukasziewicz,\ \BBA\
  Madali\'nska-Bugaj}{Doherty et~al.}{2000}]{Dohertyetal00}
Doherty, P., {\L}ukasziewicz, W., \BBA\ Madali\'nska-Bugaj, E. \BBOP2000\BBCP.
\newblock \BBOQ The {PMA} and relativizing change for action update\BBCQ\
\newblock {\Bem Fundamenta Informaticae}, {\Bem 44\/}(1-2), 95--131.

\bibitem[\protect\BCAY{E{\'{e}}n\ \BBA\ Biere}{E{\'{e}}n\ \BBA\
  Biere}{2005}]{sat/EenB05}
E{\'{e}}n, N.\BBACOMMA\  \BBA\ Biere, A. \BBOP2005\BBCP.
\newblock \BBOQ Effective preprocessing in {SAT} through variable and clause
  elimination\BBCQ\
\newblock In {\Bem {SAT}}, \lowercase{\BVOL}\ 3569 of {\Bem Lecture Notes in
  Computer Science}, \BPGS\ 61--75. Springer.

\bibitem[\protect\BCAY{Eiter, Ianni, Schindlauer, Tompits,\ \BBA\ Wang}{Eiter
  et~al.}{2006}]{Eiteretal06}
Eiter, T., Ianni, G., Schindlauer, R., Tompits, H., \BBA\ Wang, K.
  \BBOP2006\BBCP.
\newblock \BBOQ Forgetting in managing rules and ontologies\BBCQ\
\newblock In {\Bem Proc. of WI'06, 2006 IEEE / WIC / ACM International
  Conference on Web Intelligence}, \BPGS\ 411--419.

\bibitem[\protect\BCAY{Eiter\ \BBA\ Wang}{Eiter\ \BBA\
  Wang}{2006}]{EiterWang06}
Eiter, T.\BBACOMMA\  \BBA\ Wang, K. \BBOP2006\BBCP.
\newblock \BBOQ Forgetting and conflict resolving in disjunctive logic
  programming\BBCQ\
\newblock In {\Bem Proc. of AAAI'06}.

\bibitem[\protect\BCAY{Eiter\ \BBA\ Wang}{Eiter\ \BBA\
  Wang}{2008}]{EiterWang08}
Eiter, T.\BBACOMMA\  \BBA\ Wang, K. \BBOP2008\BBCP.
\newblock \BBOQ Semantic forgetting in answer set programming\BBCQ\
\newblock {\Bem Artif. Intell.}, {\Bem 172\/}(14), 1644--1672.

\bibitem[\protect\BCAY{Eiter\ \BBA\ Kern{-}Isberner}{Eiter\ \BBA\
  Kern{-}Isberner}{2019}]{ki/EiterK19}
Eiter, T.\BBACOMMA\  \BBA\ Kern{-}Isberner, G. \BBOP2019\BBCP.
\newblock \BBOQ A brief survey on forgetting from a knowledge representation
  and reasoning perspective\BBCQ\
\newblock {\Bem K{\"{u}}nstliche Intell.}, {\Bem 33\/}(1), 9--33.

\bibitem[\protect\BCAY{Fargier\ \BBA\ Marquis}{Fargier\ \BBA\
  Marquis}{2014}]{DBLP:journals/ai/FargierM14}
Fargier, H.\BBACOMMA\  \BBA\ Marquis, P. \BBOP2014\BBCP.
\newblock \BBOQ Disjunctive closures for knowledge compilation\BBCQ\
\newblock {\Bem Artif. Intell.}, {\Bem 216}, 129--162.

\bibitem[\protect\BCAY{Fredman\ \BBA\ Khachiyan}{Fredman\ \BBA\
  Khachiyan}{1996}]{DBLP:journals/jal/FredmanK96}
Fredman, M.~L.\BBACOMMA\  \BBA\ Khachiyan, L. \BBOP1996\BBCP.
\newblock \BBOQ On the complexity of dualization of monotone disjunctive normal
  forms\BBCQ\
\newblock {\Bem J. Algorithms}, {\Bem 21\/}(3), 618--628.

\bibitem[\protect\BCAY{Giunchiglia, Marin,\ \BBA\ Narizzano}{Giunchiglia
  et~al.}{2009}]{DBLP:series/faia/GiunchigliaMN09}
Giunchiglia, E., Marin, P., \BBA\ Narizzano, M. \BBOP2009\BBCP.
\newblock \BBOQ Reasoning with quantified {B}oolean formulas\BBCQ\
\newblock In Biere, A., Heule, M., van Maaren, H., \BBA\ Walsh, T.\BEDS, {\Bem
  Handbook of Satisfiability}, \lowercase{\BVOL}\ 185 of {\Bem Frontiers in
  Artificial Intelligence and Applications}, \BPGS\ 761--780. {IOS} Press.

\bibitem[\protect\BCAY{Gurvich\ \BBA\ Khachiyan}{Gurvich\ \BBA\
  Khachiyan}{1999}]{DBLP:journals/dam/GurvichK99}
Gurvich, V.\BBACOMMA\  \BBA\ Khachiyan, L. \BBOP1999\BBCP.
\newblock \BBOQ On generating the irredundant conjunctive and disjunctive
  normal forms of monotone {B}oolean functions\BBCQ\
\newblock {\Bem Discret. Appl. Math.}, {\Bem 96-97}, 363--373.

\bibitem[\protect\BCAY{Hagen, Horatschek,\ \BBA\ Mundhenk}{Hagen
  et~al.}{2009}]{DBLP:conf/alenex/HagenHM09}
Hagen, M., Horatschek, P., \BBA\ Mundhenk, M. \BBOP2009\BBCP.
\newblock \BBOQ Experimental comparison of the two
  {F}redman-{K}hachiyan-algorithms\BBCQ\
\newblock In {\Bem In Proc. of {ALENEX}'09, 11th Workshop on Algorithm
  Engineering and Experiments}, \BPGS\ 154--161.

\bibitem[\protect\BCAY{Herzig}{Herzig}{1996}]{Herzig96}
Herzig, A. \BBOP1996\BBCP.
\newblock \BBOQ The {PMA} revisited\BBCQ\
\newblock In {\Bem Proc. of {KR}'96}, \BPGS\ 40--50.

\bibitem[\protect\BCAY{Herzig, Lang,\ \BBA\ Marquis}{Herzig
  et~al.}{2013}]{Herzigetal13}
Herzig, A., Lang, J., \BBA\ Marquis, P. \BBOP2013\BBCP.
\newblock \BBOQ Propositional update operators based on formula/literal
  dependence\BBCQ\
\newblock {\Bem {ACM} Transactions on Computational Logic}, {\Bem 14\/}(3), 24.

\bibitem[\protect\BCAY{Herzig\ \BBA\ Rifi}{Herzig\ \BBA\
  Rifi}{1998}]{HerzigRifi98}
Herzig, A.\BBACOMMA\  \BBA\ Rifi, O. \BBOP1998\BBCP.
\newblock \BBOQ Update operations: a review\BBCQ\
\newblock In {\Bem Proc. of {ECAI}'98}, \BPGS\ 13--17.

\bibitem[\protect\BCAY{Herzig\ \BBA\ Rifi}{Herzig\ \BBA\
  Rifi}{1999}]{HerzigRifi99}
Herzig, A.\BBACOMMA\  \BBA\ Rifi, O. \BBOP1999\BBCP.
\newblock \BBOQ Propositional belief update and minimal change\BBCQ\
\newblock {\Bem Artif. Intell.}, {\Bem 115}, 107--138.

\bibitem[\protect\BCAY{Huang\ \BBA\ Darwiche}{Huang\ \BBA\
  Darwiche}{2007}]{jair/HuangD07}
Huang, J.\BBACOMMA\  \BBA\ Darwiche, A. \BBOP2007\BBCP.
\newblock \BBOQ The language of search\BBCQ\
\newblock {\Bem J. Artif. Intell. Res.}, {\Bem 29}, 191--219.

\bibitem[\protect\BCAY{Ignatiev, Narodytska,\ \BBA\ Marques{-}Silva}{Ignatiev
  et~al.}{2019a}]{IgnatievNM19a}
Ignatiev, A., Narodytska, N., \BBA\ Marques{-}Silva, J. \BBOP2019a\BBCP.
\newblock \BBOQ Abduction-based explanations for machine learning models\BBCQ\
\newblock In {\Bem Proc. of {AAAI}'19}, \BPGS\ 1511--1519.

\bibitem[\protect\BCAY{Ignatiev, Narodytska,\ \BBA\ Marques{-}Silva}{Ignatiev
  et~al.}{2019b}]{IgnatievNM19b}
Ignatiev, A., Narodytska, N., \BBA\ Marques{-}Silva, J. \BBOP2019b\BBCP.
\newblock \BBOQ On relating explanations and adversarial examples\BBCQ\
\newblock In {\Bem Proc. of NeurIPS'19}, \BPGS\ 15857--15867.

\bibitem[\protect\BCAY{Ignatiev, Narodytska,\ \BBA\ Marques{-}Silva}{Ignatiev
  et~al.}{2019c}]{ignatiev2019validating}
Ignatiev, A., Narodytska, N., \BBA\ Marques{-}Silva, J. \BBOP2019c\BBCP.
\newblock \BBOQ On validating, repairing and refining heuristic {ML}
  explanations\BBCQ\
\newblock {\Bem CoRR}, {\Bem abs/1907.02509}.

\bibitem[\protect\BCAY{Katsuno\ \BBA\ Mendelzon}{Katsuno\ \BBA\
  Mendelzon}{1989}]{DBLP:conf/ijcai/KatsunoM89}
Katsuno, H.\BBACOMMA\  \BBA\ Mendelzon, A.~O. \BBOP1989\BBCP.
\newblock \BBOQ A unified view of propositional knowledge base updates\BBCQ\
\newblock In {\Bem Proc. of {IJCAI}'89}, \BPGS\ 1413--1419.

\bibitem[\protect\BCAY{Khachiyan, Boros, Elbassioni,\ \BBA\ Gurvich}{Khachiyan
  et~al.}{2006}]{DBLP:journals/dam/KhachiyanBEG06}
Khachiyan, L., Boros, E., Elbassioni, K.~M., \BBA\ Gurvich, V. \BBOP2006\BBCP.
\newblock \BBOQ An efficient implementation of a quasi-polynomial algorithm for
  generating hypergraph transversals and its application in joint
  generation\BBCQ\
\newblock {\Bem Discret. Appl. Math.}, {\Bem 154\/}(16), 2350--2372.

\bibitem[\protect\BCAY{Kleine{ }B{\"{u}}ning, Karpinski,\ \BBA\
  Fl{\"{o}}gel}{Kleine{ }B{\"{u}}ning
  et~al.}{1995}]{DBLP:journals/iandc/BuningKF95}
Kleine{ }B{\"{u}}ning, H., Karpinski, M., \BBA\ Fl{\"{o}}gel, A.
  \BBOP1995\BBCP.
\newblock \BBOQ Resolution for quantified {B}oolean formulas\BBCQ\
\newblock {\Bem Inf. Comput.}, {\Bem 117\/}(1), 12--18.

\bibitem[\protect\BCAY{Klieber, Janota, Marques{-}Silva,\ \BBA\ Clarke}{Klieber
  et~al.}{2013}]{DBLP:conf/cp/KlieberJMC13}
Klieber, W., Janota, M., Marques{-}Silva, J., \BBA\ Clarke, E.~M.
  \BBOP2013\BBCP.
\newblock \BBOQ Solving {QBF} with free variables\BBCQ\
\newblock In {\Bem Proc. of CP'13}, \lowercase{\BVOL}\ 8124, \BPGS\ 415--431.

\bibitem[\protect\BCAY{Knaster}{Knaster}{1928}]{Knaster28}
Knaster, B. \BBOP1928\BBCP.
\newblock \BBOQ Un th\'eor\`eme sur les fonctions d'ensembles\BBCQ\
\newblock {\Bem Ann. Soc. Polon. Math}, {\Bem 6}, 133--134.

\bibitem[\protect\BCAY{Lang\ \BBA\ Marquis}{Lang\ \BBA\
  Marquis}{2010}]{LangMarquis10}
Lang, J.\BBACOMMA\  \BBA\ Marquis, P. \BBOP2010\BBCP.
\newblock \BBOQ Reasoning under inconsistency: A forgetting-based
  approach\BBCQ\
\newblock {\Bem Artif. Intell.}, {\Bem 174\/}(12-13), 799--823.

\bibitem[\protect\BCAY{Lang, Liberatore,\ \BBA\ Marquis}{Lang
  et~al.}{2003}]{LangLM03}
Lang, J., Liberatore, P., \BBA\ Marquis, P. \BBOP2003\BBCP.
\newblock \BBOQ Propositional independence: Formula-variable independence and
  forgetting\BBCQ\
\newblock {\Bem J. Artif. Intell. Res.}, {\Bem 18}, 391--443.

\bibitem[\protect\BCAY{Lin\ \BBA\ Reiter}{Lin\ \BBA\
  Reiter}{1994}]{LinReiter94}
Lin, F.\BBACOMMA\  \BBA\ Reiter, R. \BBOP1994\BBCP.
\newblock \BBOQ Forget it!\BBCQ\
\newblock In {\Bem Proc. of AAAI Fall Symposium on Relevance}, \BPGS\ 154--159.

\bibitem[\protect\BCAY{Miller\ \BBA\ Thornton}{Miller\ \BBA\
  Thornton}{2008}]{MVLBook}
Miller, D.~M.\BBACOMMA\  \BBA\ Thornton, M.~A. \BBOP2008\BBCP.
\newblock {\Bem Multiple Valued Logic: Concepts and Representations},
  \lowercase{\BVOL}~12 of {\Bem Synthesis lectures on digital circuits and
  systems}.
\newblock Morgan {\&} Claypool Publishers.

\bibitem[\protect\BCAY{Narodytska, Kasiviswanathan, Ryzhyk, Sagiv,\ \BBA\
  Walsh}{Narodytska et~al.}{2018}]{NarodytskaKRSW18}
Narodytska, N., Kasiviswanathan, S.~P., Ryzhyk, L., Sagiv, M., \BBA\ Walsh, T.
  \BBOP2018\BBCP.
\newblock \BBOQ Verifying properties of binarized deep neural networks\BBCQ\
\newblock In {\Bem Proc. of AAAI'18}, \BPGS\ 6615--6624.

\bibitem[\protect\BCAY{Papadimitriou}{Papadimitriou}{1994}]{Papadimitriou94}
Papadimitriou, C.~H. \BBOP1994\BBCP.
\newblock {\Bem Computational complexity}.
\newblock Addison--Wesley.

\bibitem[\protect\BCAY{Pipatsrisawat\ \BBA\ Darwiche}{Pipatsrisawat\ \BBA\
  Darwiche}{2008}]{aaai/PipatsrisawatD08}
Pipatsrisawat, K.\BBACOMMA\  \BBA\ Darwiche, A. \BBOP2008\BBCP.
\newblock \BBOQ New compilation languages based on structured
  decomposability\BBCQ\
\newblock In {\Bem Proc. of {AAAI}'08}, \BPGS\ 517--522.

\bibitem[\protect\BCAY{Quine}{Quine}{1955}]{Quine55}
Quine, W.~V. \BBOP1955\BBCP.
\newblock \BBOQ A way to simplify truth functions\BBCQ\
\newblock {\Bem American Mathematical Monthly}, {\Bem 62}, 627--631.

\bibitem[\protect\BCAY{Sedaghat, Stephen,\ \BBA\ Chindelevitch}{Sedaghat
  et~al.}{2018}]{DBLP:conf/wea/SedaghatSC18}
Sedaghat, N., Stephen, T., \BBA\ Chindelevitch, L. \BBOP2018\BBCP.
\newblock \BBOQ Speeding up dualization in the {F}redman-{K}hachiyan algorithm
  {B}\BBCQ\
\newblock In {\Bem In Proc of SEA'18, 17th International Symposium on
  Experimental Algorithms}, \lowercase{\BVOL}\ 103 of {\Bem LIPIcs}, \BPGS\
  6:1--6:13. Schloss Dagstuhl - Leibniz-Zentrum f{\"{u}}r Informatik.

\bibitem[\protect\BCAY{Shi, Shih, Darwiche,\ \BBA\ Choi}{Shi
  et~al.}{2020}]{shi2020tractable}
Shi, W., Shih, A., Darwiche, A., \BBA\ Choi, A. \BBOP2020\BBCP.
\newblock \BBOQ On tractable representations of binary neural networks\BBCQ\
\newblock In {\Bem Proc. of KR'20}, \BPGS\ 882--892.

\bibitem[\protect\BCAY{Shih, Choi,\ \BBA\ Darwiche}{Shih
  et~al.}{2018}]{ShihCD18}
Shih, A., Choi, A., \BBA\ Darwiche, A. \BBOP2018\BBCP.
\newblock \BBOQ A symbolic approach to explaining {B}ayesian network
  classifiers\BBCQ\
\newblock In {\Bem Proc. of {IJCAI}'18}, \BPGS\ 5103--5111.

\bibitem[\protect\BCAY{Shih, Choi,\ \BBA\ Darwiche}{Shih
  et~al.}{2019}]{ShihCD19}
Shih, A., Choi, A., \BBA\ Darwiche, A. \BBOP2019\BBCP.
\newblock \BBOQ Compiling {B}ayesian network classifiers into decision
  graphs\BBCQ\
\newblock In {\Bem Proc. of {AAAI}'19}, \BPGS\ 7966--7974.

\bibitem[\protect\BCAY{Shukla, Biere, Pulina,\ \BBA\ Seidl}{Shukla
  et~al.}{2019}]{DBLP:conf/ictai/ShuklaBPS19}
Shukla, A., Biere, A., Pulina, L., \BBA\ Seidl, M. \BBOP2019\BBCP.
\newblock \BBOQ A survey on applications of quantified {B}oolean formulas\BBCQ\
\newblock In {\Bem Proc. of ICTAI'19}, \BPGS\ 78--84.

\bibitem[\protect\BCAY{Stockmeyer}{Stockmeyer}{1977}]{Stockmeyer77}
Stockmeyer, L.~J. \BBOP1977\BBCP.
\newblock \BBOQ The polynomial-time hierarchy\BBCQ\
\newblock {\Bem Theoretical Computer Science}, {\Bem 3}, 1--22.

\bibitem[\protect\BCAY{Subbarayan\ \BBA\ Pradhan}{Subbarayan\ \BBA\
  Pradhan}{2005}]{NiVER}
Subbarayan, S.\BBACOMMA\  \BBA\ Pradhan, D.~K. \BBOP2005\BBCP.
\newblock \BBOQ {NiVER:} non-increasing variable elimination resolution for
  preprocessing sat instances\BBCQ\
\newblock In Hoos, H.~H.\BBACOMMA\  \BBA\ Mitchell, D.~G.\BEDS, {\Bem Theory
  and Applications of Satisfiability Testing}, \BPGS\ 276--291, Berlin,
  Heidelberg. Springer Berlin Heidelberg.

\bibitem[\protect\BCAY{Tarski}{Tarski}{1955}]{Tarski55}
Tarski, A. \BBOP1955\BBCP.
\newblock \BBOQ A lattice-theoretical fixpoint theorem and its
  applications\BBCQ\
\newblock {\Bem Pacific Journal of Mathematics}, {\Bem 5\/}(2), 285--309.

\bibitem[\protect\BCAY{Wang, Sattar,\ \BBA\ Su}{Wang
  et~al.}{2005}]{Wang-etal05}
Wang, K., Sattar, A., \BBA\ Su, K. \BBOP2005\BBCP.
\newblock \BBOQ A theory of forgetting in logic programming\BBCQ\
\newblock In {\Bem Proc. of AAAI'05}, \BPGS\ 682--688.

\bibitem[\protect\BCAY{Weber}{Weber}{1986}]{DBLP:conf/eds/Weber86}
Weber, A. \BBOP1986\BBCP.
\newblock \BBOQ Updating propositional formulas\BBCQ\
\newblock In {\Bem Proc. of EDS'86, 1st Int. Conf. on Expert Database Systems},
  \BPGS\ 487--500.

\bibitem[\protect\BCAY{Winslett}{Winslett}{1990}]{Winslett90}
Winslett, M. \BBOP1990\BBCP.
\newblock {\Bem Updating Logical Databases}.
\newblock Cambridge University Press, Cambridge, England.

\bibitem[\protect\BCAY{Zhang\ \BBA\ Foo}{Zhang\ \BBA\ Foo}{2006}]{ZhangFoo06}
Zhang, Y.\BBACOMMA\  \BBA\ Foo, N.~Y. \BBOP2006\BBCP.
\newblock \BBOQ Solving logic program conflict through strong and weak
  forgettings\BBCQ\
\newblock {\Bem Artif. Intell.}, {\Bem 170\/}(8-9), 739--778.

\bibitem[\protect\BCAY{Zhang, Foo,\ \BBA\ Wang}{Zhang
  et~al.}{2005}]{Zhang-etal05}
Zhang, Y., Foo, N.~Y., \BBA\ Wang, K. \BBOP2005\BBCP.
\newblock \BBOQ Solving logic program conflict through strong and weak
  forgettings\BBCQ\
\newblock In {\Bem Proc. of IJCAI'05}, \BPGS\ 627--634.

\end{thebibliography}

\end{document}